\documentclass[letterpaper,twocolumn,10pt]{article}
\usepackage{usenix2019_v3}
\usepackage[available,functional,reproduced]{usenixbadges}
\usepackage{cite}
\usepackage{amsmath,amssymb,amsfonts}
\usepackage{threeparttable}
\usepackage{algorithm,algorithmicx,algpseudocode}
\usepackage{xspace}
\usepackage{tablefootnote}
\usepackage{amsthm}
\usepackage{graphicx}
\usepackage{textcomp}
\usepackage{xcolor}
\usepackage{booktabs}
\usepackage{multirow}
\usepackage{hyperref}
\newtheorem{theorem}{Theorem}
\newtheorem{definition}{Definition}
\newtheorem{lemma}{Lemma}
\newcommand{\framework}{PatchCleanser\xspace}
\algnewcommand{\LeftCommenta}[1]{\Statex \hspace{1.3em} \(\triangleright\) #1}
\algnewcommand{\LeftCommentb}[1]{\Statex \hspace{2.7em} \(\triangleright\) #1}

\def\BibTeX{{\rm B\kern-.05em{\sc i\kern-.025em b}\kern-.08em
    T\kern-.1667em\lower.7ex\hbox{E}\kern-.125emX}}

\newcommand{\MohNote}[1]{}
\newcommand{\SanNote}[1]{}

\newcommand{\SomNote}[1]{}
\newcommand{\SaeedNote}[1]{}
\newcommand{\AbhrNote}[1]{}
\newcommand{\SamNote}[1]{}

\newcommand{\fhat}[2]{\ifthenelse{\equal{#2}{}}{\hat{f}(#1)}{\ifthenelse{\equal{#2}{0}}{\hat{f}(\emptyset)}{\hat{f}(#1_{\leq #2})}}}
\newcommand{\ftild}[2]{\ifthenelse{\equal{#2}{}}{\tilde{f}(#1)}{\ifthenelse{\equal{#2}{0}}{\tilde{f}(\emptyset)}{\tilde{f}(#1_{\leq #2})}}}
\newcommand{\ftildstar}[2]{\ifthenelse{\equal{#2}{}}{\tilde{f^*}(#1)}{\ifthenelse{\equal{#2}{0}}{\tilde{f^*}(\emptyset)}{\tilde{f^*}(#1_{\leq #2})}}}
\newcommand{\ghat}[2]{\ifthenelse{\equal{#2}{}}{\hat{g}(#1)}{\ifthenelse{\equal{#2}{0}}{\hat{g}(\emptyset)}{\hat{g}(#1_{\leq #2})}}}
\newcommand{\mfix}[2]{\ifthenelse{\equal{#2}{}}{m(#1)}{\ifthenelse{\equal{#2}{0}}{m(\emptyset)}{m(#1_{\leq #2})}}}

\newcommand{\fapp}[2]{\ifthenelse{\equal{#2}{}}{\tilde{f}(#1)}{\ifthenelse{\equal{#2}{0}}{\tilde{f}(\emptyset)}{\tilde{f}(#1_{\leq #2})}}}

\newcommand{\parag}[1]{\noindent{\bf #1}}
\renewcommand{\paragraph}[1]{\parag{#1}}

\newcommand{\aSF}{\mathsf{a}}
\newcommand{\gSF}{\mathsf{g}}
\newcommand{\hSF}{\mathsf{h}}

\newcommand{\avr}[2]{\ifthenelse{\equal{#2}{}}{\aSF({#1})}{\ifthenelse{\equal{#2}{0}}{\aSF(\emptyset)}{\aSF({#1}_{\leq #2})}}}

\newcommand{\avrMax}[2]{\ifthenelse{\equal{#2}{}}{\aSF^*({#1})}{\ifthenelse{\equal{#2}{0}}{\aSF^*(\emptyset)}{\aSF^*({#1}_{\leq #2})}}}

\newcommand{\avrApp}[2]{\ifthenelse{\equal{#2}{}}{\tilde{\aSF}({#1})}{\ifthenelse{\equal{#2}{0}}{\tilde{\aSF}(\emptyset)}{\tilde{\aSF}({#1}_{\leq #2})}}}

\newcommand{\avrAppMax}[2]{\ifthenelse{\equal{#2}{}}{\tilde{\aSF}^*({#1})}{\ifthenelse{\equal{#2}{0}}{\tilde{\aSF}^*(\emptyset)}{\tilde{\aSF}^*({#1}_{\leq #2})}}}

\newcommand{\ArgMax}[2]{\ifthenelse{\equal{#2}{}}{\hSF({#1})}{\ifthenelse{\equal{#2}{0}}{\hSF(\emptyset)}{\hSF({#1}_{\leq #2})}}}

\newcommand{\AppArgMax}[2]{\ifthenelse{\equal{#2}{}}{\tilde{\hSF}({#1})}{\ifthenelse{\equal{#2}{0}}{\tilde{\hSF}(\emptyset)}{\tilde{\hSF}({#1}_{\leq #2})}}}

\newcommand{\gain}[2]{\ifthenelse{\equal{#2}{}}{\gSF(#1)}{\gSF(#1_{\leq #2})}}
\newcommand{\gainMax}[2]{\ifthenelse{\equal{#2}{}}{\gSF^*(#1)}{\gSF^*(#1_{\leq #2})}}
\newcommand{\gainApp}[2]{\ifthenelse{\equal{#2}{}}{\tilde{\gSF}(#1)}{\tilde{\gSF}(#1_{\leq #2})}}
\newcommand{\gainAppMax}[2]{\ifthenelse{\equal{#2}{}}{\tilde{\gSF}^*(#1)}{\tilde{\gSF}^*(#1_{\leq #2})}}

\newcommand{\pr}[2][]{\Pr_{\ifthenelse{\isempty{#1}}{}{{#1}}}\left[{#2}\right]}

\newcommand{\remove}[1]{}

\newcommand{\st}{\mathrm{~~s.t.~~}}

\newcommand{\cA}{{\mathcal A}}

\newcommand{\cL}{{\mathcal L}}
\newcommand{\cM}{{\mathcal M}}

\newcommand{\cP}{{\mathcal P}}

\newcommand{\cR}{{\mathcal R}}
\newcommand{\cS}{{\mathcal S}}

\newtheorem*{claim}{Claim}

\newcommand{\sdotfill}{\textcolor[rgb]{0.8,0.8,0.8}{\dotfill}} 

\makeatletter
\def\th@protocol{%
    \normalfont 
    \setbeamercolor{block title example}{bg=orange,fg=white}
    \setbeamercolor{block body example}{bg=orange!20,fg=black}
    \def\inserttheoremblockenv{exampleblock}
  }
\makeatother

\newtheorem{proto}[theorem]{Protocol}
\newtheorem{protoc}[theorem]{Protocol}

\newcommand{\namedref}[2]{#1~\ref{#2}}

\newcommand{\torestate}[3]{%
\expandafter \def \csname BBRESTATE #2 \endcsname{#3}
\theoremstyle{plain}
\newtheorem{BBRESTATETHMNUM#2}[theorem]{#1}
\begin{BBRESTATETHMNUM#2}\label{#2}\csname BBRESTATE #2 \endcsname   \end{BBRESTATETHMNUM#2}
\newtheorem*{BBRESTATETHMNONNUM#2}{\namedref{#1}{#2}}
}

\newcommand{\restate}[1]{\begin{BBRESTATETHMNONNUM#1}[Restated] \csname BBRESTATE #1 \endcsname
\end{BBRESTATETHMNONNUM#1}}

\begin{document}

\title{\framework: Certifiably Robust Defense against Adversarial Patches \\for Any Image Classifier}

\author{{\rm Chong Xiang}\\Princeton University\and{\rm Saeed Mahloujifar}\\Princeton University\and{\rm Prateek Mittal}\\Princeton University}
\maketitle

\begin{abstract}
The adversarial patch attack against image classification models aims to inject adversarially crafted pixels within a restricted image region (i.e., a patch) for inducing model misclassification. This attack can be realized in the physical world by printing and attaching the patch to the victim object; thus, it imposes a real-world threat to computer vision systems. To counter this threat, we design \framework as a certifiably robust defense against adversarial patches. In \framework, we perform two rounds of \textit{pixel masking} on the input image to neutralize the effect of the adversarial patch. This image-space operation makes \framework compatible with any state-of-the-art image classifier for achieving high accuracy. Furthermore, we can prove that \framework will always predict the correct class labels on certain images against any adaptive white-box attacker within our threat model, achieving certified robustness. We extensively evaluate \framework on the ImageNet, ImageNette, CIFAR-10, CIFAR-100, SVHN, and Flowers-102 datasets and demonstrate that our defense achieves similar clean accuracy as state-of-the-art classification models and also significantly improves certified robustness from prior works. Remarkably, \framework achieves 83.9\% top-1 clean accuracy and 62.1\% top-1 certified robust accuracy against a 2\%-pixel square patch anywhere on the image for the 1000-class ImageNet dataset.\footnote{Our source code is available at \url{https://github.com/inspire-group/PatchCleanser}.}
\end{abstract}

\section{Introduction}\label{sec-introduction}
The adversarial patch attack~\cite{brown2017adversarial,karmon2018lavan,yang2020patchattack} against image classification models aims to induce test-time misclassification. A patch attacker injects adversarially crafted pixels within a localized and restricted region (i.e., a patch) and can realize a physical-world attack by printing and attaching the patch to the victim object. The physically realizable nature of patch attacks imposes a significant threat to real-world computer vision systems. 

To secure the deployment of critical computer vision systems, there has been an active research thread on certifiably robust defenses against adversarial patches~\cite{chiang2020certified,zhang2020clipped,levine2020randomized,xiang2021patchguard,metzen2021efficient,cropping}. These defenses aim to provide a certifiable guarantee on making correct predictions on certain images, even in the presence of an adaptive white-box attacker. This strong robustness property provides a pathway toward ending the arms race between attackers and defenders.

\begin{figure*}
    \centering
    \includegraphics[width=\linewidth]{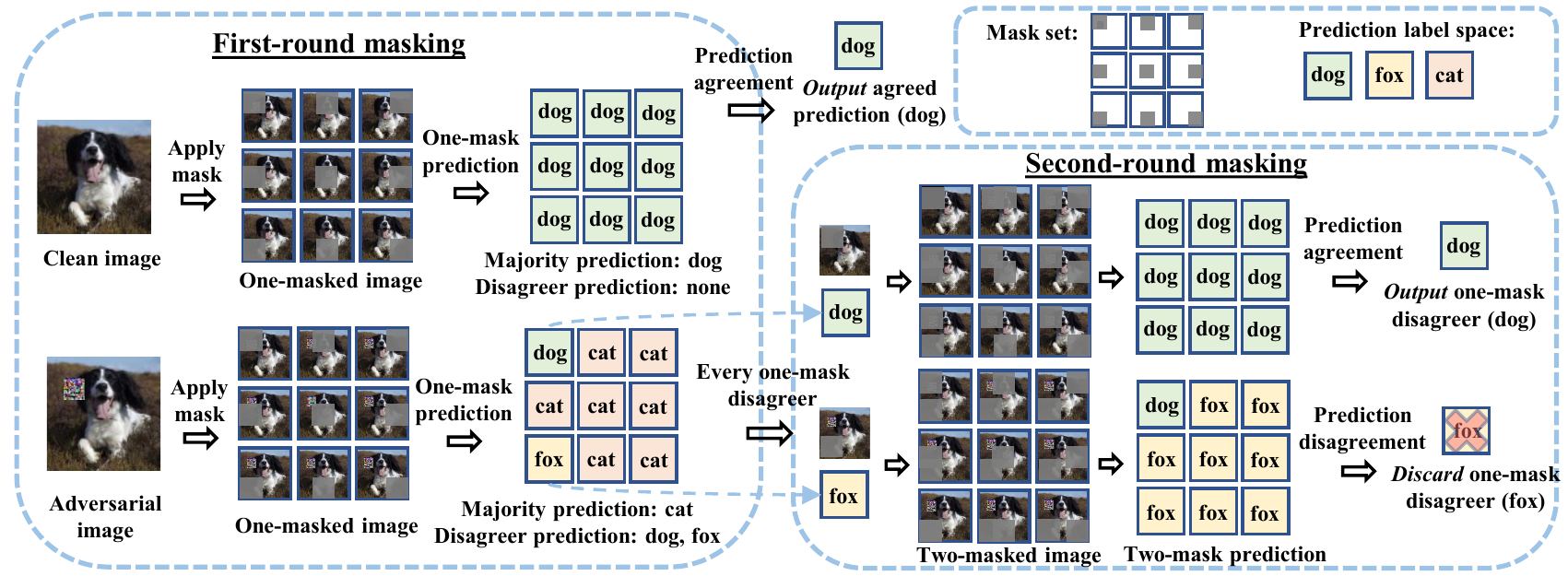}
    \vspace{-1.5em}
    \caption{\textbf{Overview of double-masking defense.} The defense applies masks to the input image and evaluates model prediction on every masked image. \textit{Clean image:} all one-mask predictions typically agree on the correct label (``dog"); our defense outputs the agreed prediction. \textit{Adversarial image:} one-mask predictions have a disagreement; we aim to recover the benign prediction. We first categorize all one-mask predictions into the \textit{majority prediction} (the one with the highest prediction label occurrence; the label ``cat" in this example) and \textit{disagreer predictions} (the ones that disagree with the majority; the labels ``dog" and ``fox"). For every mask that leads to a disagreer prediction, we add a set of second masks and evaluate two-mask predictions. If all two-mask predictions agree with this one-mask disagreer, we \textit{output} its prediction label (the label ``dog"; illustrated in the upper row of the second-round masking); otherwise, we \textit{discard} it (the label ``fox"; in the lower row of the second-round masking).}
    \label{fig-overview}
\end{figure*}

\textbf{Limitation of prior works: the dependence on specific model architectures.} While prior works have made significant contributions to certifiable robustness, their defense performance is hindered by their dependence on specific model architectures. The most common architecture constraint of state-of-the-art certifiably robust defenses against patch attacks~\cite{zhang2020clipped,levine2020randomized,xiang2021patchguard,metzen2021efficient,cropping} is the dependence on small receptive fields (receptive field is the region of the input image that an extracted feature is looking at, or affected by). The small receptive field bounds the number of features that can be corrupted by the adversarial patch but also limits the information received by each feature. As a result, defenses with small receptive fields are limited in their classification accuracy: for example, the best top-1 clean accuracy on ImageNet~\cite{deng2009imagenet} achieved by prior certifiably robust defenses is around 55\%~\cite{xiang2021patchguard,cropping} while state-of-the-art undefended classification models can attain an accuracy of 80\%-90\%~\cite{bit,wightman2021resnet,vit,resmlp}. The poor clean accuracy discourages the real-world deployment of proposed defenses and also limits the achievable robustness (since the robust accuracy can be no higher than the clean accuracy).\footnote{Chiang et al.~\cite{chiang2020certified} proposed the first certifiably robust defense against adversarial patches via Interval Bound Propagation~\cite{gowal2018effectiveness,mirman2018differentiable}. This defense does not rely on small receptive fields but requires extremely expensive model training. As a result, it is only applicable to small classification models (with limited performance) and low-resolution images.}

\textbf{Limitation of prior works: abstention from predictions.} Minority Reports (MR)~\cite{mccoyd2020minority} is the only certifiably robust defense with no assumption on the model architecture; however, it suffers from a weaker security guarantee of being able to only \textit{detect} a patch attack (i.e., alert when an attack is detected). As a result, an attacker can force the model to always alert and abstain from making a prediction. The inability of a model to make a prediction can compromise functionality in applications where human fallback is unavailable (e.g., level-5 autonomous vehicles without human drivers).

\textbf{\framework: architecture-agnostic certifiably robust image classification (without abstention).} In order to overcome the limitations of prior works, we propose \framework as a certifiably robust image classification (without any abstention) defense that is compatible with any image classifier. The high-level idea of \framework is to robustly remove/mask all adversarial pixels on the input image so that we can obtain accurate predictions (on the masked images) from \textit{any} state-of-the-art image classifier. 

However, the key question is: \textit{How can we mask out the patch if the patch location is unknown?} An intuitive idea is to place a mask at all possible image locations and evaluate model predictions on every masked image. If the mask is large enough, then at least one masked image is benign (i.e., no adversarial pixels) and is likely to give a correct prediction (a similar intuition is used in MR for attack detection~\cite{mccoyd2020minority}). Unfortunately, despite the existence of one benign (and usually correct) masked prediction, it is challenging to robustly distinguish this benign prediction from other masked predictions that can be adversarially manipulated by an adaptive attacker. To solve this challenge, we propose a \textit{double-masking} algorithm that achieves certifiable robustness.

We provide a defense overview in Figure~\ref{fig-overview}. The double-masking algorithm involves two rounds of pixel masking. In the first round of masking (left of the figure), we apply every mask from a \textit{mask set} to the input image and evaluate model predictions on \textit{one-masked} images. The mask set is constructed in a way that at least one mask can remove the entire patch (regardless of the patch location) and give a benign (and usually correct) masked prediction. When the algorithm operates on a clean image, all one-mask predictions usually reach a unanimous agreement, and \framework will output the agreed label (top of Figure~\ref{fig-overview}). On the other hand, for an adversarial image, since at least one mask can remove the patch and recover the benign prediction, we will see a disagreement between the benign prediction and malicious predictions (left bottom of Figure~\ref{fig-overview}). To robustly identify the benign one-mask prediction, we perform a second round of masking: we apply a set of second masks to every \textit{one-masked} image and use inconsistencies in model predictions on a set of \textit{two-masked} images to filter out all malicious one-mask predictions (right of the figure). We will present the details of our double-masking defense in Section~\ref{sec-two-round} and demonstrate that it provides certifiable robustness for certain images against \textit{any patch attacker within our threat model} in Section~\ref{sec-provable}.

\textbf{Evaluation: state-of-the-art clean accuracy and certified robust accuracy.} We instantiate \framework with three representative state-of-the-art architectures for image classification: ResNet~\cite{resnet}, Vision Transformer (ViT)~\cite{vit}, and ResMLP~\cite{resmlp}. We evaluate our defense performance on six image datasets: ImageNet~\cite{deng2009imagenet}, ImageNette~\cite{imagenette}, CIFAR-10~\cite{cifar}, CIFAR-100~\cite{cifar}, SVHN~\cite{svhn}, and Flowers-102~\cite{flowers}.
We demonstrate that \framework achieves state-of-the-art (clean) classification accuracy and also greatly improves the certified robust accuracy from prior works~\cite{chiang2020certified,zhang2020clipped,levine2020randomized,xiang2021patchguard,metzen2021efficient}. In Figure~\ref{fig-main-comparison}, we plot the clean accuracy and certified robust accuracy of different defenses on the ImageNet dataset~\cite{deng2009imagenet} to visualize our significant performance improvements.
Our contributions can be summarized as follows:
\begin{itemize}
    \setlength\itemsep{0em}
    \item We present \framework's double-masking defense that is compatible with any image classifier to mitigate the threat of adversarial patch attacks.
    \item We formally prove the certifiable robustness of \framework for certain images against any adaptive white-box attacker within our threat model.
    \item We evaluate \framework on three state-of-the-art classification models and six benchmark datasets and demonstrate the significant improvements in clean accuracy and certified robust accuracy (e.g., Figure~\ref{fig-main-comparison}).
\end{itemize}

\begin{figure}[t]
    \centering
    \includegraphics[width=0.7\linewidth]{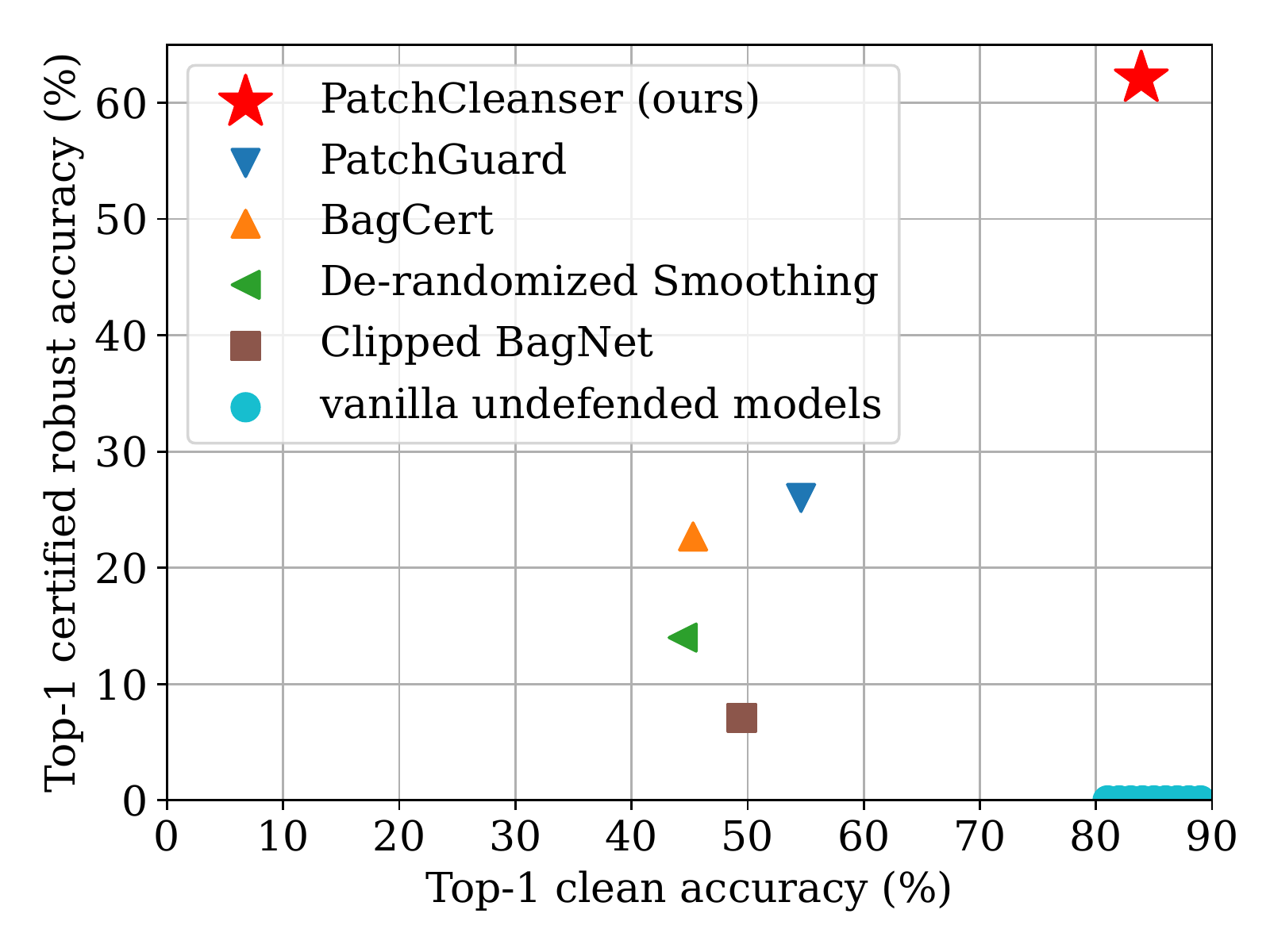}
    \vspace{-1.5em}
    \caption{ImageNet~\cite{deng2009imagenet} clean accuracy  and certified robust accuracy of \framework and PatchGuard~\cite{xiang2021patchguard}, BagCert~\cite{metzen2021efficient}, De-randomized Smoothing~\cite{levine2020randomized}, and Clipped BagNet~\cite{zhang2020clipped}; 
    the certified robust accuracy is evaluated against a 2\%-pixel square patch anywhere over the image.}
    \label{fig-main-comparison}
  
\end{figure}

\section{Problem Formulation}\label{sec-formulation}
In this section, we formulate image classification models, attack threat models, and our defense objectives.

\subsection{Image Classification Model}
In this paper, we focus on the image classification problem. We use $\mathcal{X}\subset [0,1]^{W\times H\times C}$ to denote the image space, where each image has width $W$, height $H$, number of channels $C$, and the pixels are re-scaled to $[0,1]$. We further denote the label space as $\mathcal{Y}$. An image classification model is denoted as $\mathbb{F}:\mathcal{X}\rightarrow \mathcal{Y}$, which takes an image $\mathbf{x}\in \mathcal{X}$ as input and predicts the class label $y\in\mathcal{Y}$. 

We do not make any assumption on the architecture of the image classification model $\mathbb{F}$. Our defense is compatible with any popular model such as ResNet~\cite{resnet}, Vision Transformer~\cite{vit}, and ResMLP~\cite{resmlp}. 

\subsection{Threat Model}
\textbf{Attack objective.} We focus on test-time evasion attacks against image classification models. Given a model $\mathbb{F}$, an image $\mathbf{x}$, and its true class label $y$, the attacker aims to find an image $\mathbf{x}^\prime \in \mathcal{A}(\mathbf{x}) \subset \mathcal{X}$ satisfying a constraint $\mathcal{A}$ such that $\mathbb{F}(\mathbf{x}^\prime) \neq y$. The constraint $\mathcal{A}$ is defined by the attacker's threat model, which we discuss next. 

\textbf{Attacker capability.} The patch attacker has \textit{arbitrary} control over the image pixels in a \textit{restricted} region, and this region can be \textit{anywhere} on the image.
Formally, we use a binary tensor $\mathbf{r} \in \{0,1\}^{W\times H}$ to represent the restricted region, where the \textit{pixels within the region are set to $0$} and others are set to $1$. We further use $\mathcal{R}$ to denote a set of regions $\mathbf{r}$ (i.e., a set of patches at different locations). Then, we can express the patch attacker's constraint set $\mathcal{A}_{\mathcal{R}}(\mathbf{x})$ as $\{\mathbf{r}\odot \mathbf{x} + (\mathbf{1}-\mathbf{r}) \odot \mathbf{x}^{\prime} \ |\  \mathbf{x},\mathbf{x}^\prime \in \mathcal{X}, \mathbf{r} \in \mathcal{R}\}$, where $\odot$ refers to the element-wise multiplication operator. When clear from the context, we drop $\mathcal{R}$ and use $\mathcal{A}$ instead of $\mathcal{A}_{\mathcal{R}}$.

\textbf{An open research question: one single square patch at any image location.}
In this paper, we primarily focus on a popular open research question where $\mathbf{r}$ represents \textit{one square region that can be anywhere on the image} and the defender has a \textit{conservative estimation of the patch size}.\footnote{We note that similar assumptions on defender's knowledge are also commonly used in defenses against conventional global $L_p$ perturbations. For example, verifiably robust network training~\cite{gowal2018effectiveness,mirman2018differentiable} and empirical adversarial training~\cite{goodfellow2014explaining,madry2017towards} need to know the norm and magnitude of the perturbations.} This enables a performance comparison with prior works that also focus on this setting~\cite{chiang2020certified,zhang2020clipped,levine2020randomized,xiang2021patchguard,metzen2021efficient} (Section~\ref{sec-evaluation}). Moreover, we note that designing high-performance certifiably robust defenses under this setting is extremely challenging due to attacker's \textit{arbitrary control} over the patch location and patch content as well as \textit{full knowledge} of the defense setup.

\textbf{Flexibility of \framework.}
Nevertheless, our defense design is general and can be easily adapted for even stronger attackers. In addition to evaluating our defense under the setting of \textit{one single square patch} in Section~\ref{sec-evaluation}, we also quantitatively analyze our defense against attackers who can use \textit{a set of different patch shapes} (e.g., all possible rectangle shapes covering a certain area at any image location) and who can apply \textit{multiple patches} (e.g., two patches at any image location) in Section~\ref{sec-discussion-multiple-shape-patch}.

\subsection{Defense Objective}\label{sec-formulation-defense}
We design \framework with four major objectives.

\textbf{Robust classification.} We aim to build a defended model $\mathbb{D}$ for \textit{robust classification}. That is, we want to have $\mathbb{D}(\mathbf{x}^\prime)=\mathbb{D}(\mathbf{x})=y$ for a clean data point $(\mathbf{x},y)\in\mathcal{X}\times\mathcal{Y}$ and any adversarial example $\mathbf{x}^\prime \in \mathcal{A}(\mathbf{x})$. Note that we aim to \emph{recover the correct prediction without any abstention}, which is harder than merely detecting an attack (e.g., Minority Reports~\cite{mccoyd2020minority}).

\textbf{Certifiable robustness.} We aim to design defenses with certifiable robustness~\cite{chiang2020certified,zhang2020clipped,levine2020randomized,xiang2021patchguard,metzen2021efficient}: given a clean data point $(\mathbf{x},y)$, the defended model can always make a correct prediction for any adversarial example within the threat model, i.e., $\mathbb{D}(\mathbf{x}^\prime)=\mathbb{D}(\mathbf{x})=y,\ \forall\ \mathbf{x}^\prime\in\mathcal{A}(\mathbf{x})$. 
We will design a robustness certification procedure, which takes a clean data point $(\mathbf{x},y)$ and threat model $\mathcal{A}$ as inputs, to check if the robustness can be certified. The certification procedure should account for all possible attackers within the threat model $\mathcal{A}$, who could have full knowledge of our defense and full access to our model parameters. We note that the certification provides a \textit{provable lower bound} for model robustness against adaptive attacks. This is a significant improvement over traditional empirical defenses~\cite{hayes2018visible,naseer2019local,wu2019defending,rao2020adversarial,Mu2021defending,cosgrove2020robustness}, whose robustness could be undermined by an adaptive attacker. 

We note that we only discuss robustness certification for \textit{labeled} images because the certification procedure needs ground-truth labels to check the correctness of model predictions. In our evaluation, we apply our certification procedure to labeled test sets and calculate the fraction of certified images, termed  as \textit{certified robust accuracy}, as our robustness metric. This accuracy indicates the estimated robustness (against the strongest adaptive attacks) when we deploy the defense; we do not aim to guarantee robustness/correctness for individual images in the wild.

\textbf{Compatibility with any model architecture.} As discussed in Section~\ref{sec-introduction}, prior works~\cite{chiang2020certified,zhang2020clipped,levine2020randomized,xiang2021patchguard,metzen2021efficient} on certifiably robust image classification suffer from their dependence on the model architecture (e.g., small receptive fields). Such dependence limits the model performance and hinders the practical deployment of the defense. In \framework, we aim to design a defense that is compatible with any state-of-the-art model architecture to achieve high defense performance (recall Figure~\ref{fig-main-comparison}) and benefit from any advancement in image classification research.

\textbf{Scalability to high-resolution images.} Finally, we aim to design our defense to be efficient enough to scale to high-resolution images. Since state-of-the-art models usually use high-resolution images for better classification performance~\cite{resnet,bit,vit,resmlp}, defenses that only work for low-resolution images~\cite{chiang2020certified} have limited applicability.

\section{\framework Design}\label{sec-defense}

In this section, we introduce our \framework defense, which is agnostic to model architectures and achieves certifiable robustness. \framework performs two rounds of pixel masking (i.e., \textit{double-masking)} on the input image to neutralize the effect of the adversarial patch (without knowing the location and content of the patch). We present our formulation of pixel masks in Section~\ref{sec-pixel-mask} and then discuss the details of double-masking algorithm in Section~\ref{sec-two-round}. We prove the robustness of our double-masking defense for certain images in Section~\ref{sec-provable}. Finally, we discuss implementation details and present an end-to-end \framework defense pipeline in Section~\ref{sec-defense-implementation}. We provide a summary of important notation in Table~\ref{tab-notation}.

\subsection{Pixel Mask Set}\label{sec-pixel-mask}
\framework aims to mask out the entire patch on the image and obtain accurate predictions from any state-of-the-art classification model. In this subsection, we introduce the concept of a mask set used in our masking operations.

\textbf{Mask set formulation.} We represent each mask as a binary tensor $\mathbf{m} \in \{0,1\}^{W\times H}$ in the same shape as the $W \times H$ images; \textit{the elements within the mask take values of $0$}, and others are $1$. We further denote a set of masks as $\cM$ (these are similar to the definitions of $\mathbf{r}$ and $\mathcal{R}$). We require the mask set $\cM$ to have the $\cR$-covering property as defined below.

\begin{definition}[$\mathcal{R}$-covering]A mask set $\cM$ is $\mathcal{R}$-covering if, for any patch in the patch region set $\cR$, at least one mask from the mask set $\cM$ can cover the entire patch, i.e., 
$$\forall\ \mathbf{r}\in \mathcal{R},\ \exists\ \mathbf{m} \in \mathcal{M} \st \mathbf{m}[i,j]\leq \mathbf{r}[i,j], \  \forall (i,j)$$
\end{definition}

\noindent For a particular patch region set $\cR$, there are multiple valid $\cR$-covering mask sets $\cM$ with a variable number of masks and different mask sizes/shapes. We will discuss a general approach for $\cR$-covering mask set generation in Section~\ref{sec-defense-implementation}. In the next subsection, we introduce how to perform our double-masking defense with a $\cR$-covering mask set.

\subsection{Double-masking for Robust Prediction}\label{sec-two-round}

The double-masking algorithm is the core module of \framework; it performs two rounds of masking with an $\cR$-covering mask set to robustly recover the correct prediction label. Our defense is based on the intuition that model predictions on images without adversarial pixels are generally correct and invariant to the masking operation: in Figure~\ref{fig-overview}, we can visually recognize the dog even with one or two masks on the image.\footnote{A similar intuition is used in existing works~\cite{hayes2018visible,mccoyd2020minority,chou2020sentinet}, but we are the first to design a certifiably robust image classification defense without abstention.} In this subsection, we first introduce the high-level defense design and then explain the algorithm details.

\begin{table}[t]
    \centering
    \caption{Summary of important notation}
 \resizebox{\linewidth}{!}
  { \begin{tabular}{l|l|l|l}
    \toprule
    \textbf{Notation} & \textbf{Description} & \textbf{Notation} & \textbf{Description} 
    \\
    \midrule
    $\mathbb{F}$ &Undefended model& $\bar{p},\bar{p}_0\times\bar{p}_1$ & Estimated patch size\\
 $\mathbf{x}\in\mathcal{X}$ & Input image &  $p,p_0\times p_1$ &Actual patch size \\
 $y,\bar{y}\in\mathcal{Y}$ &  Class label &  $k,k_0\times k_1$ &Budget of  \#masks   \\
     $\mathbf{m}\in\mathcal{M}$ & Pixel mask & $m,m_0\times m_1$ & Mask size\\
     $\mathbf{r}\in\mathcal{R}$ & Patch region & $s,s_0\times s_1$ & Mask stride\\
  $\mathcal{P}\subset\mathcal{M}\times\mathcal{Y}$ &Masked prediction set &$n,n_0\times n_1$&Image size\\
      \bottomrule
    \end{tabular}}
    \label{tab-notation}
\end{table}

\textbf{First-round masking: detecting a prediction disagreement.} Recall that Figure~\ref{fig-overview} gives an overview of our double-masking algorithm. In the first round of masking, we apply every mask $\mathbf{m}$ from the $\cR$-covering mask set $\mathcal{M}$ to the input image and evaluate all \textit{one-mask predictions} (left of Figure~\ref{fig-overview}). In the clean setting, all one-mask predictions are likely to reach a unanimous agreement on the correct label, and we will output the agreed prediction (top of Figure~\ref{fig-overview}). In the adversarial setting, at least one mask will remove all adversarial pixels; thus, at least one one-mask prediction is benign and likely to be correct (bottom left of Figure~\ref{fig-overview}). In this case, we will detect a disagreement in one-mask predictions (benign versus malicious); we will then perform a second round of masking to settle this disagreement.

\textbf{Second-round masking: settling the prediction disagreement.} We first divide all one-mask prediction labels into two groups: the \textit{majority prediction} (the prediction label with the highest occurrence) and the \textit{disagreer predictions} (other labels that disagree with the majority). We need to decide which prediction label to trust (i.e., the majority or one of the disagreers). To solve this problem, we iterate over every \textit{disagreer} prediction, get its corresponding first-round mask, and add a second mask from our mask set $\cM$ to compute a set of \textit{two-mask predictions} (right of Figure~\ref{fig-overview}). If the first-round disagreer mask removes the patch, every second-round mask is applied to a ``clean" image, and thus all two-mask predictions (evaluated with one first-round mask and different second-round masks) are likely to have a unanimous agreement. We can trust and return this agreed prediction. On the other hand, if the first-round disagreer mask does not remove the patch, the one-masked image is still ``adversarial", and the second-round mask will cause a disagreement in two-mask predictions (when one of the second masks covers the patch). 
In this case, we discard this one-mask disagreer. Finally, if we try all one-mask disagreer predictions and no prediction label is returned, we trust and return the one-mask majority prediction as the default exit case.

\textbf{Algorithm details.} We provide the defense pseudocode in Algorithm~\ref{alg-prediction}. The defense takes an image $\mathbf{x}$, an undefended model $\mathbb{F}$, and an $\cR$-covering mask set $\mathcal{M}$ as inputs and outputs a robust prediction $\bar{y}$. Line~\ref{ln-single-masking}-\ref{ln-first-mask-e} illustrates the first-round masking; Line~\ref{ln-double-masking-s}-\ref{ln-double-masking-e} demonstrates the second-round masking. 

\textit{Details of first-round masking.} In Algorithm~\ref{alg-prediction}, we first call the masking sub-procedure $\textsc{MaskPred}(\cdot)$ using the mask set $\mathcal{M}$ (Line~\ref{ln-single-masking}). The mask set $\mathcal{M}$ needs to ensure that at least one mask can remove the entire patch (i.e., $\cR$-covering); we will discuss the mask set generation approach in Section~\ref{sec-defense-implementation}.

In $\textsc{MaskPred}(\cdot)$, we aim to collect all masked predictions and determine the majority prediction label (i.e., the label with the highest occurrence) as well as disagreer predictions (i.e., other predictions). We first generate a set $\mathcal{P}$ for holding all mask-prediction pairs. Next, for each mask $\mathbf{m}$ in the mask set $\mathcal{M}$, we evaluate the masked prediction via $\bar{y}\gets\mathbb{F}(\mathbf{x}\odot\mathbf{m})$; here $\odot$ is the element-wise multiplication operator. We then add the mask-prediction pair ($\mathbf{m},\bar{y}$) to the set $\mathcal{P}$. After gathering all masked predictions, we identify the label with the highest prediction occurrence as majority prediction $\bar{y}_{\text{maj}}$ (Line~\ref{ln-get-maj-pred}). Furthermore, we construct a disagreer prediction set $\mathcal{P}_{\text{dis}}$, whose elements are disagreer mask-prediction pairs (Line~\ref{ln-get-dis-set}). Finally, we return the majority prediction label $\bar{y}_{\text{maj}}$ and the disagreer prediction set $\mathcal{P}_{\text{dis}}$.

After the first call of $\textsc{MaskPred}(\cdot)$, we check if one-mask predictions reach a unanimous agreement (i.e., the disagreer prediction set $\mathcal{P}_{\text{dis}}$ is empty; Line~\ref{ln-one-mask-agreement}). If $\mathcal{P}_{\text{dis}}$ is empty, we consider the input image likely as a clean image and return the agreed/majority prediction (\textit{Case I: agreed prediction}; Line~\ref{ln-case-1}). On the other hand, a non-empty disagreer set implies a first-round prediction disagreement, and the algorithm proceeds to the second-round masking to settle the disagreement.

\begin{algorithm}[t]
    \centering
    \caption{Double-masking defense of \framework}\label{alg-prediction}
    \begin{algorithmic}[1]
    \renewcommand{\algorithmicrequire}{\textbf{Input:}}
    \renewcommand{\algorithmicensure}{\textbf{Output:}}
    \Require Image $\mathbf{x}$, vanilla prediction model $\mathbb{F}$, mask set $\mathcal{M}$
    \Ensure  Robust prediction $\bar{y}$ 
    \Procedure{DoubleMasking}{$\mathbf{x},\mathbb{F},\mathcal{M}$}
    \State $\bar{y}_{\text{maj}},\mathcal{P}_{\text{dis}}\gets\textsc{MaskPred}(\mathbf{x},\mathbb{F},\mathcal{M})$\Comment{First-rnd. mask}\label{ln-single-masking}
    \If{$\mathcal{P}_{\text{dis}}=\varnothing$}\label{ln-first-mask-empty-condition} \label{ln-one-mask-agreement}
    \State \Return $\bar{y}_{\text{maj}}$ \Comment{\textit{Case I: agreed prediction}}\label{ln-case-1}
    \EndIf  \label{ln-first-mask-e}
    
    \For{each $(\mathbf{m}_{\text{dis}},\bar{y}_{\text{dis}})\in\mathcal{P}_{\text{dis}}$}\label{ln-double-masking-s}\Comment{Second-rnd. mask}
       \State $\bar{y}^\prime,\mathcal{P}^\prime\gets\textsc{MaskPred}(\mathbf{x}\odot\mathbf{m}_{\text{dis}},\mathbb{F},\mathcal{M})$ \label{ln-double-masking}
       \If{$\mathcal{P}^\prime = \varnothing$}\label{ln-two-condition-s} 
       \State \Return $\bar{y}_{\text{dis}}$ \Comment{\textit{Case II: disagreer prediction}}\label{ln-case-2}
       \EndIf\label{ln-two-condition-e}
    \EndFor\label{ln-double-masking-e}
    \State\Return $\bar{y}_{\text{maj}}$\Comment{\textit{Case III: majority prediction}}\label{ln-case-3}
    \EndProcedure
\item[]
    \Procedure{MaskPred}{$\mathbf{x},\mathbb{F},\mathcal{M}$}
    \State$\mathcal{P}\gets \varnothing$\Comment{A set for mask-prediction pairs}
    \For{$\mathbf{m}\in\mathcal{M}$} \Comment{Enumerate every mask $\mathbf{m}$}
    \State $\bar{y}\gets \mathbb{F}(\mathbf{x}\odot\mathbf{m})$ \Comment{Evaluate masked prediction}
    \State $\mathcal{P}\gets\mathcal{P}\ \bigcup \ \{(\mathbf{m},\bar{y})\}$ \Comment{Update set $\mathcal{P}$}
    \EndFor
    \State $\bar{y}_{\text{maj}}\gets\arg\max_{{y}^*}|\{(\mathbf{m},\bar{y})\in\mathcal{P}\ |\ \bar{y}=y^*\}|$\label{ln-get-maj-pred}\Comment{Majority}
    \State $\mathcal{P}_{\text{dis}}\gets\{(\mathbf{m},\bar{y})\in\mathcal{P} \ | \ \bar{y}\neq\bar{y}_{\text{maj}}\}$\label{ln-get-dis-set}\Comment{Disagreers}
    \State\Return $\bar{y}_{\text{maj}},\mathcal{P}_{\text{dis}}$
    \EndProcedure
    \end{algorithmic}

\end{algorithm}

\textit{Details of second-round masking.} The pseudocode of the second-round masking is in Line~\ref{ln-double-masking-s}-\ref{ln-double-masking-e}. We will look into \textit{every} one-mask disagreer prediction in  $\mathcal{P}_{\text{dis}}$ (Line~\ref{ln-double-masking-s}). For each $(\mathbf{m}_{\text{dis}},\bar{y}_{\text{dis}})$, we apply the disagreer mask $\mathbf{m}_{\text{dis}}$ to the image and feed the masked image $\mathbf{x}\odot\mathbf{m}_{\text{dis}}$ to the masking sub-procedure $\textsc{MaskPred}(\cdot)$ for the second-round masking (Line~\ref{ln-double-masking}). If all two-mask predictions reach a unanimous agreement (i.e., $\mathcal{P}^\prime=\varnothing$), we consider that the first-round mask $\mathbf{m}_{\text{dis}}$ has already removed the adversarial perturbations. Our algorithm returns this one-mask disagreer prediction (\textit{Case II: disagreer prediction}; Line~\ref{ln-case-2}). 
On the other hand, if two-mask predictions disagree (i.e., $\mathcal{P}^\prime\neq\varnothing$), we consider that the disagreer mask $\mathbf{m}_{\text{dis}}$ has not removed the patch. In this case, we discard this one-mask disagreer prediction and move to the next one. In the end, if we return no prediction in the second-round masking, we trust and return the one-mask majority prediction $\bar{y}_{\text{maj}}$ (\textit{Case III: majority prediction}; Line~\ref{ln-case-3}).

\textbf{Remark: defense complexity.} When the number of disagreer predictions is bounded by a small constant $C$ (which is the usual case in the clean setting), the defense complexity is $O(|\cM|+C\cdot|\cM|)$. However, its worst-case complexity is $O(|\cM|^2)$ (doing all two-mask predictions). In Appendix~\ref{apx-new-inference}, we discuss and evaluate another defense algorithm that has the same robustness guarantees and a better worst-case inference complexity $O(|\cM|)$, at the cost of a drop in clean accuracy.

\subsection{Robustness Certification for Double-Masking Defense}\label{sec-provable}
In this subsection, we discuss how to certify the robustness of our double-masking algorithm for a given image. Recall that we say our defense is certifiably robust for a given image if our model prediction is always correct against any adaptive white-box attacker within our threat model $\cA_\cR$. The certification only applies to labeled images since we need ground-truth labels to check the prediction correctness.

First, we define a concept of \textit{two-mask correctness}, which we claim is a sufficient condition for certified robustness. 
\begin{definition}[two-mask correctness]
A model $\mathbb{F}$ has two-mask correctness for a mask set $\cM$ and a clean image data point $(\mathbf{x},y)$, if model predictions on all possible two-masked images are correct, i.e., 
$$\mathbb{F}(\mathbf{x}\odot\mathbf{m}_0\odot\mathbf{m}_1)=y\ ,\forall \ \mathbf{m}_0\in\mathcal{M},\forall\ \mathbf{m}_1\in\mathcal{M}$$
\end{definition}

\noindent Next, we present our theorem stating that two-mask correctness for a clean image (and an $\cR$-covering mask set) implies the certifiable robustness of our defense to adversarial patches (constrained by $\cA_\cP$) on that image.

\begin{theorem}\label{thm-pred}
Given a clean data point $(\mathbf{x},y)$, a classification model $\mathbb{F}$, a mask set $\mathcal{M}$, and the threat model $\mathcal{A}_\cR$, if $\cM$ is $\cR$-covering and $\mathbb{F}$ has two-mask correctness for $\cM$ and $(\mathbf{x},y)$, then our double-masking defense (Algorithm~\ref{alg-prediction}) will always return a correct label, i.e.,  $\textsc{DoubleMasking}(\mathbf{x}^\prime,\mathbb{F},\mathcal{M})=y ,\ \forall\  \mathbf{x}^\prime\in\mathcal{A}_\cR(\mathbf{x})$.
\end{theorem}

\renewcommand{\proofname}{Proof}

\begin{proof}
First, we present three useful claims for our proof.

\begin{claim}
Given the same conditions of Theorem~\ref{thm-pred} ($\cR$-covering and two-mask correctness), we have: 
\begin{enumerate}\setlength\itemsep{0em}
    \item There is at least one correct one-mask prediction in the first-round masking (Line~\ref{ln-single-masking} of Algorithm~\ref{alg-prediction}).
    \item There is at least one correct two-mask prediction in every iteration of the second-round masking (Line~\ref{ln-double-masking} of Algorithm~\ref{alg-prediction}).
    \item If a first-round mask removes the patch, then all its second-round two-mask predictions (Line~\ref{ln-double-masking} of Algorithm~\ref{alg-prediction}) are correct.
\end{enumerate}

\end{claim}

\begin{proof}
The proof of three claims follows from the definitions of $\cR$-covering and two-mask correctness.
\begin{enumerate}\setlength\itemsep{0em}
    \item The first claim holds since at least one first-round mask removes the patch (due to $\cR$-covering) and recovers the correct prediction (due to two-mask correctness; note that two-mask correctness reduces to ``one-mask" correctness when two masks are at the same locations).
    \item The second claim holds since at least one second-round mask removes the patch (due to $\cR$-covering) and recovers the correct prediction (due to two-mask correctness). 
    \item The third claim holds since all the two-mask predictions are benign and correct when the patch is removed by the first-round mask (due to two-mask correctness).
\end{enumerate}
\vspace{-2em}
\end{proof}

Next, we use these three claims to prove three lemmas, which together show that our double-masking algorithm (Algorithm~\ref{alg-prediction}) will never return an incorrect prediction (given $\cR$-covering and two-mask correctness). 

\begin{lemma}\label{lemma-case-i}
Given the same conditions of Theorem~\ref{thm-pred}, 
Algorithm~\ref{alg-prediction} will never return an incorrect label via \textit{Case I} (Line~\ref{ln-case-1} of Algorithm~\ref{alg-prediction}).
\end{lemma}
\begin{proof}
If Algorithm~\ref{alg-prediction} returns an incorrect label ($\bar{y}_{\text{maj}}\neq y$) via \textit{Case I}, it means that $\cP_{\text{dis}}=\varnothing$ and all one-mask predictions in the first-round masking are incorrect as $\bar{y}_{\text{maj}}$. This leads to a contradiction because our first claim indicates that at least one one-mask prediction is correct.
\end{proof}

\begin{lemma}\label{lemma-case-ii}
Given the same conditions of Theorem~\ref{thm-pred}, 
Algorithm~\ref{alg-prediction} will never return an incorrect label via \textit{Case II} (Line~\ref{ln-case-2} of Algorithm~\ref{alg-prediction}).
\end{lemma}
\begin{proof}

If Algorithm~\ref{alg-prediction} returns an incorrect label ($\bar{y}_{\text{dis}}\neq y$) via \textit{Case II}, it means that $\cP^\prime=\varnothing$ and all two-mask predictions for the first-round disagreer $(\mathbf{m}_{\text{dis}},\bar{y}_{\text{dis}})$ in the second-round masking are incorrect as $\bar{y}_{\text{dis}}$. This leads to a contradiction because our second claim indicates that at least one two-mask prediction is correct in any iteration of the second-round masking.
\end{proof}

\begin{lemma}\label{lemma-case-iii}
Given the same conditions of Theorem~\ref{thm-pred}, 
Algorithm~\ref{alg-prediction} will never return an incorrect label via \textit{Case III} (Line~\ref{ln-case-3} of Algorithm~\ref{alg-prediction}).
\end{lemma}
\begin{proof}

If Algorithm~\ref{alg-prediction} returns  an incorrect label via \textit{Case III}, we have $\bar{y}_{\text{maj}}\neq y$. This implies that the correct label $y$ is a disagreer (recall that at least one first-round mask removes the patch and gives the correct one-mask prediction label $y$). From our third claim, we know that for this one-mask disagreer (whose first-round mask removes the patch), all its two-mask predictions are correct due to two-mask correctness. Therefore, we have $\cP^\prime=\varnothing$, and Algorithm~\ref{alg-prediction} will return this disagreer $\bar{y}_{\text{dis}}=y$ via \textit{Case II}. This contradicts with Algorithm~\ref{alg-prediction} returning a label via \textit{Case III}. 
\end{proof}%

Putting things together, we prove in the above three lemmas that our double-masking algorithm (Algorithm~\ref{alg-prediction}) will never return an incorrect label. Since Algorithm~\ref{alg-prediction} will always return a prediction label, we have proved that our defense will always return a correct label under the conditions of Theorem~\ref{thm-pred}.
\end{proof}

\begin{algorithm}[t]
    \centering
    \caption{Robustness certification for \framework}\label{alg-prediction-provable}
    \begin{algorithmic}[1]
    \renewcommand{\algorithmicrequire}{\textbf{Input:}}
    \renewcommand{\algorithmicensure}{\textbf{Output:}}
    \Require Image $\mathbf{x}$, ground-truth label $y$, vanilla prediction model $\mathbb{F}$, mask set $\cM$, threat model $\cA_\cR$
    \Ensure  Whether $\mathbf{x}$ has certified robustness
    \Procedure{Certification}{$\mathbf{x},y,\mathbb{F},\cM,\cA_\cR$}
    \If{$\cM$ is not $\cR$-covering}\Comment{Insecure mask set} \label{ln-check-lemma}
    \State\Return \texttt{False}
    \EndIf
    \For{every $(\mathbf{m}_0,\mathbf{m}_1) \in \mathcal{M}\times\mathcal{M}$} \label{ln-for-s}
    \State $\bar{y}^\prime \gets \mathbb{F}(\mathbf{x}\odot\mathbf{m}_0\odot\mathbf{m}_1)$\Comment{Two-mask pred.}
    \If{$\bar{y}^\prime \neq y$}
    \State\Return \texttt{False} \Comment{Possibly vulnerable}
    \EndIf
    \EndFor\label{ln-for-e}
    \State\Return \texttt{True}\Comment{Certified robustness!}
    \EndProcedure
    \end{algorithmic}
\end{algorithm}

\textbf{Robustness certification procedure.} From Theorem~\ref{thm-pred}, we can certify the robustness of our defense on a clean/test image by checking if our model has two-mask correctness on that image. The pseudocode for our robust certification is presented in Algorithm~\ref{alg-prediction-provable}. It takes the labeled clean data point $(\mathbf{x},y)$, vanilla prediction model $\mathbb{F}$, mask set $\cM$, and threat model $\cA_\cR$ as inputs, and aims to determine if our defense (Algorithm~\ref{alg-prediction}) can always predict correctly as $y$.
First, the certification procedure checks if mask set $\cM$ is $\cR$-covering (Line~\ref{ln-check-lemma}). Next, it evaluates all possible two-mask predictions (Line~\ref{ln-for-s}-\ref{ln-for-e}). If any of the two-mask predictions is incorrect, the algorithm returns \texttt{False} (possibly vulnerable). On the other hand, if all two-mask predictions match the ground-truth $y$, we have certified robustness for this image, and the algorithm returns \texttt{True}.

In our evaluation (Section~\ref{sec-evaluation}), we will apply Algorithm~\ref{alg-prediction-provable} to labeled datasets and report \textit{certified robust accuracy}, the fraction of labeled test images for which Algorithm~\ref{alg-prediction-provable} returns \texttt{True}. We note that Theorem~\ref{thm-pred} ensures that this certified accuracy is the \textit{lower bound} of model accuracy against any adaptive attacker within the threat model $\cA_\cR$. For example, a certified robust accuracy of 62.1\% on the ImageNet~\cite{deng2009imagenet} dataset implies that \framework can correctly classify 62.1\% of the ImageNet test images, no matter how an adaptive attacker (within the threat model) generates and places the patch.

\subsection{Implementation of the End-to-End \framework Defense}\label{sec-defense-implementation}

In this subsection, we first present an adaptive $\cR$-covering mask set generation technique and then provide a complete view of our end-to-end \framework defense.

\textbf{Adaptive mask set generation.} In practice, a defense needs to operate within the constraints of available computational resources. This imposes a bound on the number of masked model predictions that we can evaluate.\footnote{For example, there are 40k possible locations for a 24$\times$24 mask on a 224$\times$224 image, and it is computationally expensive to evaluate all 40k masked predictions.} Therefore, we need to carefully generate a mask set that meets the computation budget (i.e., the number of masks) while maintaining the security guarantee (i.e., $\cR$-covering). 
We first present our approach for 1-D ``images" (Figure~\ref{fig-mask} provides two visual examples) and then generalize it to 2-D. %

\begin{figure}
    \centering
    \includegraphics[width=0.9\linewidth]{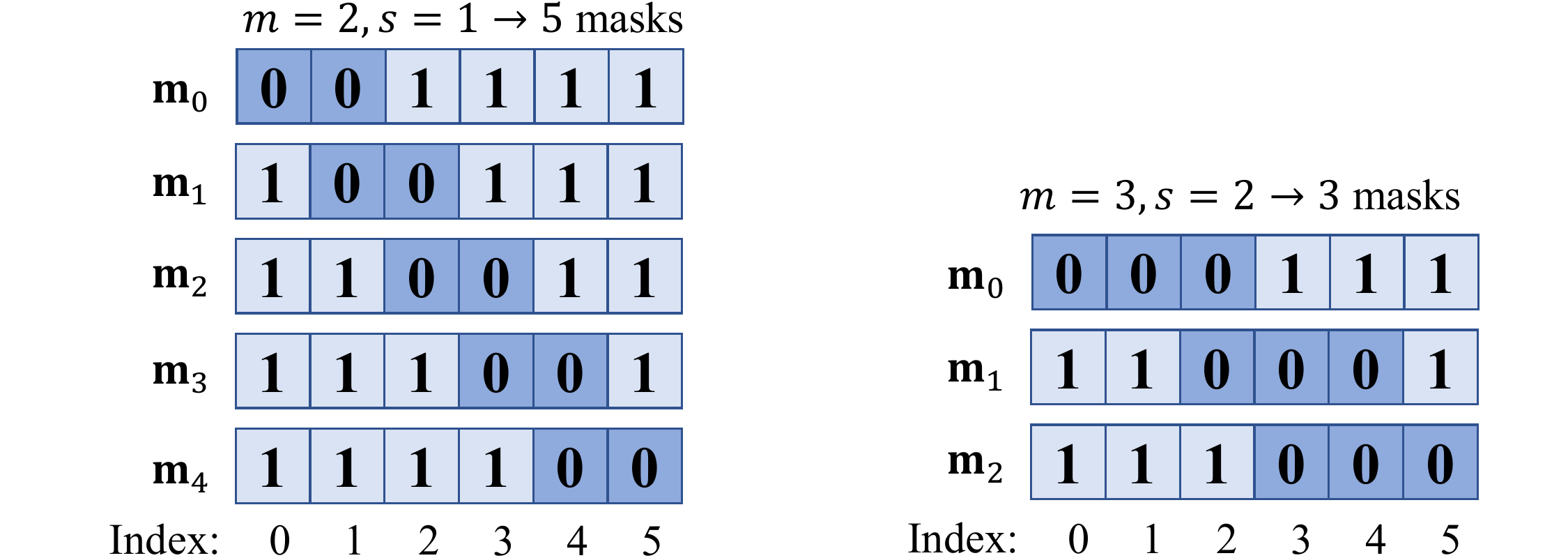}
    \caption{Visual examples for 1-D mask set generation. Both mask sets are $\cR$-covering for a patch of estimated size $\bar{p}=2$ on an image of size $n=6$. \textit{Left example:} smaller mask size $m=2$, smaller mask stride $s=1$, and larger mask set $\mathcal{I}=\{0,1,2,3,4\}$. \textit{Right example:} larger mask size $m=3$, larger mask stride $s=2$, and smaller mask set $\mathcal{I}=\{0,2,3\}$.}
    \label{fig-mask}
\end{figure}

\textit{Adjusting the mask set size.} We generate a mask set via moving a mask over the input image. Consider a mask of width $m$ over an image of size $n$, we first place the mask at the coordinate $0$ (so that the mask covers the indices from $0$ to $m-1$). Next, we move the mask with a stride of $s$ across the image and gather a set of mask locations $\{0,s,2s,\cdots,\lfloor \frac{n-m}{s} \rfloor s\}$. Finally, we place the last mask at the index of $n-m$ in case the mask at $\lfloor \frac{n-m}{s} \rfloor s$ cannot cover the last $m$ pixels. We can define a mask set $\cM_{m,s,n}$ as:
\begin{align}
    \cM_{m,s,n}=\{\mathbf{m}\in\{0,1\}^n \ | \ \mathbf{m}[u]=0, u \in [i,i+m);& \nonumber \\        
     \mathbf{m}[u]=1, u \not\in [i,i+m);&\ i\in \mathcal{I}\}\nonumber\\
     \mathcal{I}=\{0,s,2s,\cdots,\lfloor \frac{n-m}{s} \rfloor s\}\ \bigcup \ \{n-m\}\label{eqn-mask-set-generation}
\end{align}
Furthermore, we can compute the mask set size as:
\begin{equation}\label{eqn-mask-size}
  |\cM_{m,s,n}| = |\mathcal{I}| = \lceil\frac{n-m}{s}\rceil+1  
\end{equation}

This equation shows that we can adjust the mask set size via the mask stride $s$. In the example of Figure~\ref{fig-mask}, we can reduce the mask set size from $5$ to $3$ by increasing the mask stride from $1$ to $2$. Next, we discuss the security property (i.e., $\cR$-covering) of the mask set $\cM$.

\textit{Ensuring the security guarantee.} Using a large mask stride might leave ``gaps" between two adjacent masks; therefore, we need to choose a proper mask size to cover these gaps to ensure that the mask set $\cM$ is $\cR$-covering.

\begin{lemma}\label{lemma-maskset}
The mask set $\cM_{m,s,n}$ is $\cR$-covering for a patch that is no larger than $p^*=m-s+1$.
\end{lemma}

\begin{proof}
Without loss of generality, we consider the first two adjacent masks in the 1-D scenario, whose mask pixel index ranges are $[0,m-1]$ and $[s,s+m-1]$, respectively. Now let us consider an adversarial patch of size $p^*$. In order to avoid being completely masked by the first mask, the smallest index of the patch has to be no smaller than $j^*=(m-1)-(p^*-1)+1=s$. However, we find that the second mask starts from the index $s$, so the patch that evades the first mask will be captured by the second mask.
\end{proof}
\noindent Lemma~\ref{lemma-maskset} indicates that the mask size needs to be no smaller than $m^*=p+s-1$ to ensure $\cR$-covering.

\textit{Mask set generation.} Armed with the ability to adjust the mask set size and to ensure the security guarantee (as discussed above), we now present our complete mask set generation approach. The procedure takes as inputs \textit{the computation budget} $k$ (i.e., the number of masks), the \textit{estimated} patch size $\bar{p}$ (i.e., the security parameter), and the image size $n$, and aims to generate a $\cR$-covering set that satisfies the computational budget $k$. First, based on the the inputs $k,\bar{p},n$, we derive mask stride $s$ and mask size $m$ using Lemma~\ref{lemma-maskset} ($\bar{p}=m-s+1$) and Equation~\ref{eqn-mask-size} ($k=\lceil\frac{n-m}{s}\rceil+1$) as follows: 
\begin{align}
    &s = \lceil\frac{n-\bar{p}+1}{k}\rceil \nonumber\\
    &m = \bar{p}+s-1\label{eqn-setup}
\end{align}
Next, we generate the set $\cM_{m,s,n}$ via Equation~\ref{eqn-mask-set-generation} accordingly. We note that when we have a different estimation for the patch size $\bar{p}$, we only need to adjust the mask stride $s$ and mask size $m$ according to Equation~\ref{eqn-setup} while keeping the number of masks $k$ unchanged.

\textit{Generalizing to 2-D images.} We can easily generalize the 1-D mask set to 2-D by separately applying Equation~\ref{eqn-mask-set-generation} and Equation~\ref{eqn-setup} to each of the two axes of the image. For $n_0\times n_1$ images, $\bar{p}_0\times \bar{p}_1$ patches, $k_0\times k_1$ number of masks, we can calculate $s_0,s_1,m_0,m_1$ with Equation~\ref{eqn-setup} and obtain $\mathcal{I}_0,\mathcal{I}_1$ with Equation~\ref{eqn-mask-set-generation}. The mask set generation becomes $\cM_{(m_0,m_1),(s_0,s_1),(n_0,n_1)}=\{\mathbf{m}\in\{0,1\}^{n_0\times n_1}\ | \ \mathbf{m}[u,v]=0, u\in[i,i+m_0),v\in[j,j+m_1),(i,j)\in\mathcal{I}_0\times\mathcal{I}_1 ;\ \mathbf{m}[u,v]=1, \text{otherwise}\}$.

\textit{Remark: trade-off between efficiency and accuracy.} As shown in Equation~\ref{eqn-setup}, if we want to improve the efficiency (by having a smaller $k$), we will have to use a larger stride $s$ and larger mask size $m$. Intuitively, the model prediction is less accurate for a larger mask; thus, the improvement in efficiency can be at the cost of model accuracy. Our mask set generation approach allows us to balance this trade-off between efficiency and accuracy in the real-world deployment. We will study this trade-off in Section~\ref{sec-detailed-eval}.

\textbf{End-to-end \framework pipeline.} With the mask set generation technique, we can summarize the end-to-end \framework pipeline as follows: 
\begin{enumerate}
\setlength\itemsep{0em}
    \item  First, we perform \textit{adaptive mask set generation} to obtain a secure $\cR$-covering mask set $\cM$ that satisfies a certain computation budget (number of masks $k_0\times k_1$).
    \item  Second, we perform  \textit{double-masking} (Algorithm~\ref{alg-prediction}) with the model $\mathbb{F}$ and the mask set $\cM$ for robust classification.
    \item Third, we can use our certification procedure (Algorithm~\ref{alg-prediction-provable}) to certify the robustness of \framework on given labeled images against any adaptive white-box attacker within the threat model $\cA_\cR$. We will evaluate the fraction of labeled test images that can be certified across multiple datasets in the next section.
\end{enumerate}

\section{Evaluation}\label{sec-evaluation}

We instantiate \framework with three different classification models, and extensively evaluate the defense using three different datasets (we have results for another three datasets in Appendix~\ref{apx-eval}). 
We will demonstrate state-of-the-art clean accuracy and certified robust accuracy of \framework compared with prior works~\cite{chiang2020certified,zhang2020clipped,levine2020randomized,xiang2021patchguard,metzen2021efficient} and provide detailed analysis of our defense under different settings.

In this section, we primarily focus on 
a single square patch that can have arbitrary content and that can be anywhere on the image. This setting is currently an open research question in the field, and also allows for a fair comparison with prior works~\cite{chiang2020certified,zhang2020clipped,levine2020randomized,xiang2021patchguard,metzen2021efficient}. We will show the flexibility of \framework by demonstrating its generalization to a set of different patch shapes and multiple patches in Section~\ref{sec-discussion-multiple-shape-patch}.

\subsection{Setup}
In this subsection, we briefly introduce our evaluation setup. We provide additional details in Appendix~\ref{apx-setup}. Our source code is available at \url{https://github.com/inspire-group/PatchCleanser}.

\textbf{Datasets.} We choose three popular image classification benchmark datasets for evaluation: ImageNet~\cite{deng2009imagenet}, ImageNette~\cite{imagenette}, CIFAR-10~\cite{cifar}. In Appendix~\ref{apx-eval}, we also include results for three additional datasets (Flowers-102~\cite{flowers}, CIFAR-100~\cite{cifar}, and SVHN~\cite{svhn}).

\textit{ImageNet and ImageNette.} ImageNet~\cite{deng2009imagenet} is a challenging image classification dataset which has 1.3M training images and 50k validation images from 1000 classes. ImageNette~\cite{imagenette} is a 10-class subset of ImageNet with 9469 training images and 3925 validation images. ImageNet/ImageNette images have a high resolution, and we resize and crop them to 224$\times$224 before feeding them to different models.

\textit{CIFAR-10.} CIFAR-10~\cite{cifar} is a benchmark dataset for low-resolution image classification. CIFAR-10 has 50k training images and 10k test images from 10 classes. Each image is in the resolution of 32$\times$32. We resize them to 224$\times$224 via bicubic interpolation for a better classification performance.

\textbf{Models.} We choose three representative image classification models to build \framework. We provide model training details in Appendix~\ref{apx-setup}.

\textit{ResNet.} ResNet~\cite{resnet} is a classic Convolutional Neural Network (CNN) model. It uses layers of convolution filters and residual blocks to extract features for image classification. We use ResNetV2-50x1 and its publicly available weights trained for ImageNet~\cite{deng2009imagenet}. We finetune the model for other different datasets used in our evaluation.

\textit{Vision Transformer (ViT).} ViT~\cite{vit} is adapted from NLP Transformer~\cite{vaswani2017attention} for the image classification task. It divides an image into disjoint pixel blocks, and uses self-attention architecture to extract features across different pixel blocks for classification. We use ViT-B16-224~\cite{vit} trained for ImageNet and finetune it on other datasets. 

\textit{Multi-layer Perceptron (MLP).} There have been recent advances in leveraging MLP-only architectures for image classification (e.g., MLP-mixer~\cite{mlpmixer}, ResMLP~\cite{resmlp}). These architectures take pixel blocks as input and ``mix" features/pixels across locations and channels for predictions. We choose ResMLP-S24-224~\cite{resmlp} in our evaluation. We take the pre-trained model for ImageNet, and finetune it for other datasets.

\begin{table*}[t]
    \centering
        \caption{Clean accuracy and certified robust accuracy for different defenses and datasets$^\dag$}
    \resizebox{\linewidth}{!} {\scriptsize
    \begin{threeparttable}
     \begin{tabular}{c|c|c|c|c|c|c|c|c|c|c|c|c|c|c|c|c}
    \toprule
    Dataset & \multicolumn{6}{c|}{ImageNette~\cite{imagenette}} & \multicolumn{6}{c|}{ImageNet~\cite{deng2009imagenet}} & \multicolumn{4}{c}{CIFAR-10~\cite{cifar}}\\
    \midrule
    Patch size  &  \multicolumn{2}{c|}{1\% pixels}   &  \multicolumn{2}{c|}{2\% pixels} & \multicolumn{2}{c|}{3\% pixels} & \multicolumn{2}{c|}{1\% pixels} &         \multicolumn{2}{c|}{2\% pixels} &          \multicolumn{2}{c|}{3\% pixels} & 
    \multicolumn{2}{c|}{0.4\% pixels}& 
    \multicolumn{2}{c}{2.4\% pixels}  \\
         \midrule
       Accuracy (\%)   & clean & robust &clean & robust & clean & robust & clean & robust & clean & robust & clean &  robust&clean& robust &clean& robust \\

    \midrule
    PC-ResNet&\textbf{99.6}&96.4&\textbf{99.6}&94.4&\textbf{99.5}&93.5&81.7&58.4&81.6&53.0&81.4&50.0&98.0&88.5&97.8&78.8\\ 
    PC-ViT&\textbf{99.6}&\textbf{97.5}&\textbf{99.6}&\textbf{96.4}&\textbf{99.5}&\textbf{95.3}&\textbf{84.1}&\textbf{66.4}&\textbf{83.9}&\textbf{62.1}&\textbf{83.8}&\textbf{59.0}&\textbf{99.0}&\textbf{94.3}&\textbf{98.7}&\textbf{89.1}\\
    PC-MLP& 99.4&96.8&99.3&95.3&99.4&94.6&79.6&58.4&79.4&53.8&79.3&50.7&97.4&86.1&97.0&78.0\\
    \midrule
    
    IBP~\cite{chiang2020certified} & \multicolumn{12}{c|}{computationally infeasible}  & 65.8& 51.9& 47.8 & 30.8\\
CBN~\cite{zhang2020clipped}& {94.9} &74.6&{94.9} & 60.9 & \textbf{94.9}& 45.9 &49.5& 13.4&49.5 &7.1  &49.5 &3.1 &{84.2}& 44.2& {84.2}&9.3 \\
    DS~\cite{levine2020randomized}& 92.1& 82.3&92.1&79.1& 92.1& 75.7 &  44.4 & 17.7& 44.4&14.0& 44.4 &11.2 & 83.9&68.9& 83.9&56.2 \\
    PG-BN~\cite{xiang2021patchguard} & \textbf{95.2} &\textbf{89.0}& \textbf{95.0}& \textbf{86.7} & 94.8&\textbf{83.0} & \textbf{55.1} & \textbf{32.3} & \textbf{54.6}  & \textbf{26.0}& \textbf{54.1}& \textbf{19.7} &   84.5 &63.8&83.9& {47.3} \\
    PG-DS~\cite{xiang2021patchguard}&  92.3 & 83.1 &  92.1 & 79.9 & 92.1 & 76.8& 44.1&{19.7}&43.6&{15.7}&43.0&{12.5}  & {84.7} & {69.2} &  {84.6}& {57.7}  \\
    BagCert\tnote{$\ddag$}~\cite{metzen2021efficient}&--&--&--&--&--&--&45.3&27.8&45.3&22.7&45.3&18.0&\textbf{86.0}&\textbf{72.9}&\textbf{86.0}&\textbf{60.0}\\
    \bottomrule
    \end{tabular}
      \begin{tablenotes}
      \item[$\dag$] We mark the best result for \framework models and the best result for prior works in bold.
    \item[$\ddag$]The BagCert numbers are provided by the authors~\cite{metzen2021efficient} through personal communication since the source code is unavailable; results for ImageNette are not provided.
  \end{tablenotes}
    \end{threeparttable}}
    
    \label{tab-huge-provable}
\end{table*}

\textbf{Adversarial patches.} Following prior works~\cite{chiang2020certified,levine2020randomized,xiang2021patchguard,metzen2021efficient}, we report defense performance against a square patch that takes 1\%, 2\%, and 3\% of input image pixels for ImageNet/ImageNette and a square patch with 0.4\% and 2.4\% pixels for CIFAR-10 images. We allow these patches to have \textit{arbitrary} content and be \textit{anywhere} on the image. In Section~\ref{sec-detailed-eval}, we also report results for larger patch sizes (ranging from 2\% to 62\% image pixels) to understand the limit of our defense. In Section~\ref{sec-discussion-multiple-shape-patch}, we quantitatively discuss the implications of using a set of different rectangle patch shapes as well as multiple patches.

\textbf{Defenses.} We build three defense instances PC-ResNet, PC-ViT, PC-MLP using three vanilla models of ResNet, ViT, and MLP. In Section~\ref{sec-detailed-eval}, we will analyze the effect of different defense setups (i.e., the number of masks). In our default setup, we set the number of masks $k_0\times k_1 = k^2= 6\times6$, which has high certified robustness and moderate computational overhead. We then generate the $\cR$-covering mask set $\cM$ as discussed in Section~\ref{sec-defense-implementation}. 

We also report defense performance of prior works Interval Bound Propagation based defense (IBP)~\cite{chiang2020certified}, Clipped BagNet (CBN)~\cite{zhang2020clipped}, De-randomized Smoothing (DS)~\cite{levine2020randomized}, PatchGuard (PG)~\cite{xiang2021patchguard}, and BagCert~\cite{metzen2021efficient} for comparison. We use the optimal defense settings stated in their respective papers.

\textbf{Evaluation Metrics.} We report clean accuracy and certified robust accuracy as our evaluation metrics. 
The \textit{clean accuracy} is defined as the fraction of clean test images that can be correctly classified by our defended model. 
The \textit{certified robust accuracy} is the fraction of test images for which Algorithm~\ref{alg-prediction-provable} returns \texttt{True} (certifies the robustness for this image), i.e., no adaptive white-box attacker can bypass our defense. 
In addition to accuracy metrics, we also use \textit{per-example inference time} to analyze the computational overhead.

\textbf{Remark: no need to implement adaptive attacks.} As discussed in Section~\ref{sec-provable}, certified robust accuracy is the lower bound of model accuracy against any adaptive attacker within the threat model. Therefore, it is not necessary to empirically evaluate robustness using any concrete adaptive attack strategy: empirical robust accuracy is always higher than certified accuracy.

\subsection{State-of-the-art Clean Accuracy and Certified Robust Accuracy across All Datasets}\label{sec-main-eval}

We report our main evaluation results for \framework in Table~\ref{tab-huge-provable} and compare defense performance with prior works~\cite{chiang2020certified,zhang2020clipped,levine2020randomized,xiang2021patchguard,metzen2021efficient}.

\textbf{State-of-the-art clean accuracy.} As shown in Table~\ref{tab-huge-provable}, \framework achieves high clean accuracy. Take PC-ViT as an example, \framework achieves 99.5+\% clean accuracy for 10-class ImageNette, 83.8+\% for 1000-class ImageNet, and 98.7+\% for CIFAR-10. We further report the accuracy of state-of-the-art vanilla classification models in Table~\ref{tab-vanilla}. From these two tables, we can see that the clean accuracy of our defense is very close to the state-of-the-art undefended models (the difference is smaller than 1\%). The high clean accuracy can foster real-world deployment of our defense.\footnote{Similar to prior works like PatchGuard~\cite{xiang2021patchguard}, the clean accuracy for different patch sizes varies slightly due to the use of different mask sets.} 

\begin{table}[t]
    \centering
        \caption{Clean accuracy of vanilla models}
   \resizebox{0.85\linewidth}{!}  
   {\scriptsize
     \begin{tabular}{c|c|c|c}
    \toprule
     & {ImageNette} & {ImageNet} & {CIFAR-10}\\
         \midrule
         ResNet~\cite{resnet} &  99.8\% & 82.3\% &  98.3\%\\
         ViT~\cite{vit}&99.8\%&84.8\%&99.0\%\\
         MLP~\cite{resmlp} & 99.5\%&80.2\%&97.8\%\\
    \bottomrule
    \end{tabular}}
    \label{tab-vanilla}
\end{table}

\textbf{High certified robustness.} In addition to state-of-the-art clean accuracy achieved by our defense, we can see from Table~\ref{tab-huge-provable} that \framework also has very high certified robust accuracy. For ImageNette, our PC-ViT has a certified robust accuracy of 97.5\% against a 1\% square patch. \textit{That is, for 97.5\% of the test images, no strong adaptive white-box attacker who uses a 1\%-pixel square patch anywhere on the image can induce misclassification of our defended model.} Furthermore, we can also see high certified robust accuracy for ImageNet and CIFAR-10, e.g., 66.4\% certified robust accuracy for a 1\%-pixel patch on ImageNet and  94.3\% certified robust accuracy for a 0.4\%-pixel patch on CIFAR-10.

\textbf{Significant improvements in clean accuracy and certified robust accuracy from prior works.} We compare our defense performance with all prior certifiably robust defenses. From Table~\ref{tab-huge-provable}, we can see that all our defense instances (i.e., PC-ResNet, PC-ViT, and PC-MLP) significantly outperform all prior works in terms of both clean accuracy and certified robust accuracy.
Notably, for a 2\%-pixel patch on ImageNet, PC-ViT improves the clean accuracy from 54.6\% to 83.9\% (29.3\% gain in top-1 accuracy) and boosts the certified robust accuracy from 26.0\% to 62.1\% (the accuracy gain is 36.1\%; the improvement is more than 2 times). Moreover, we can see that \textit{the certified robust accuracy of PC-ViT is even higher than the clean accuracy of all prior works}. These significant improvements are due to \framework's compatibility with state-of-the-art classification models, while previous works are fundamentally incompatible with them (recall that PG~\cite{xiang2021patchguard}, DS~\cite{levine2020randomized}, BagCert~\cite{metzen2021efficient}, CBN~\cite{zhang2020clipped} are all limited to models with a small receptive field). 

We can also see large improvements across datasets including ImageNette and CIFAR-10. For a 2\%-pixel patch on ImageNette, PC-ViT improves clean accuracy from 95.0\% to 99.6\% and certified robust accuracy from 86.7\% to 96.4\%.  For a 2.4\%-pixel patch on CIFAR-10, PC-ViT improves clean accuracy from 86.0\% to 98.7\% (12.7\% gain) and certified robust accuracy from 60.0\% to 89.1\% (29.1\% gain). 

\textbf{Takeaways.} In this subsection, we demonstrate that \framework has similar clean accuracy as vanilla state-of-the-art models, as well as high certified robust accuracy. In comparison with prior certifiably robust defenses, we demonstrate significant improvements in both clean accuracy and certified robust accuracy. These improvements showcase the strength of defenses that are compatible with any state-of-the-art model.

\begin{figure*}
\centering
\begin{minipage}[b]{0.34\linewidth}
\includegraphics[width=\linewidth]{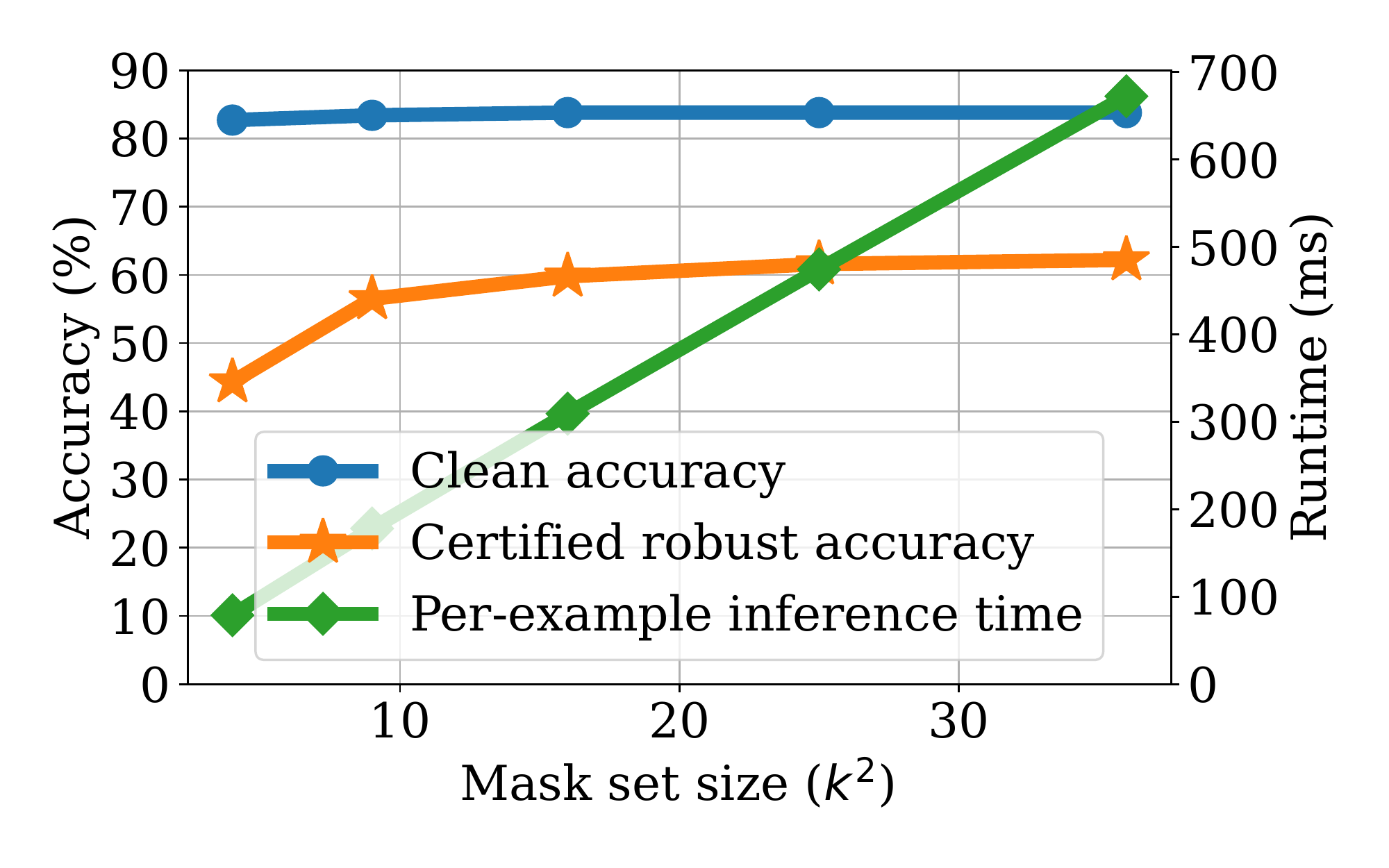}
\vspace{-3em}
\caption{The effect of mask set size on defense performance (ImageNet)}\label{fig-stride-net}
\end{minipage}%
\quad
\quad
\begin{minipage}[b]{0.55\linewidth}
\begin{minipage}[b]{0.5\linewidth}
\includegraphics[width=\linewidth]{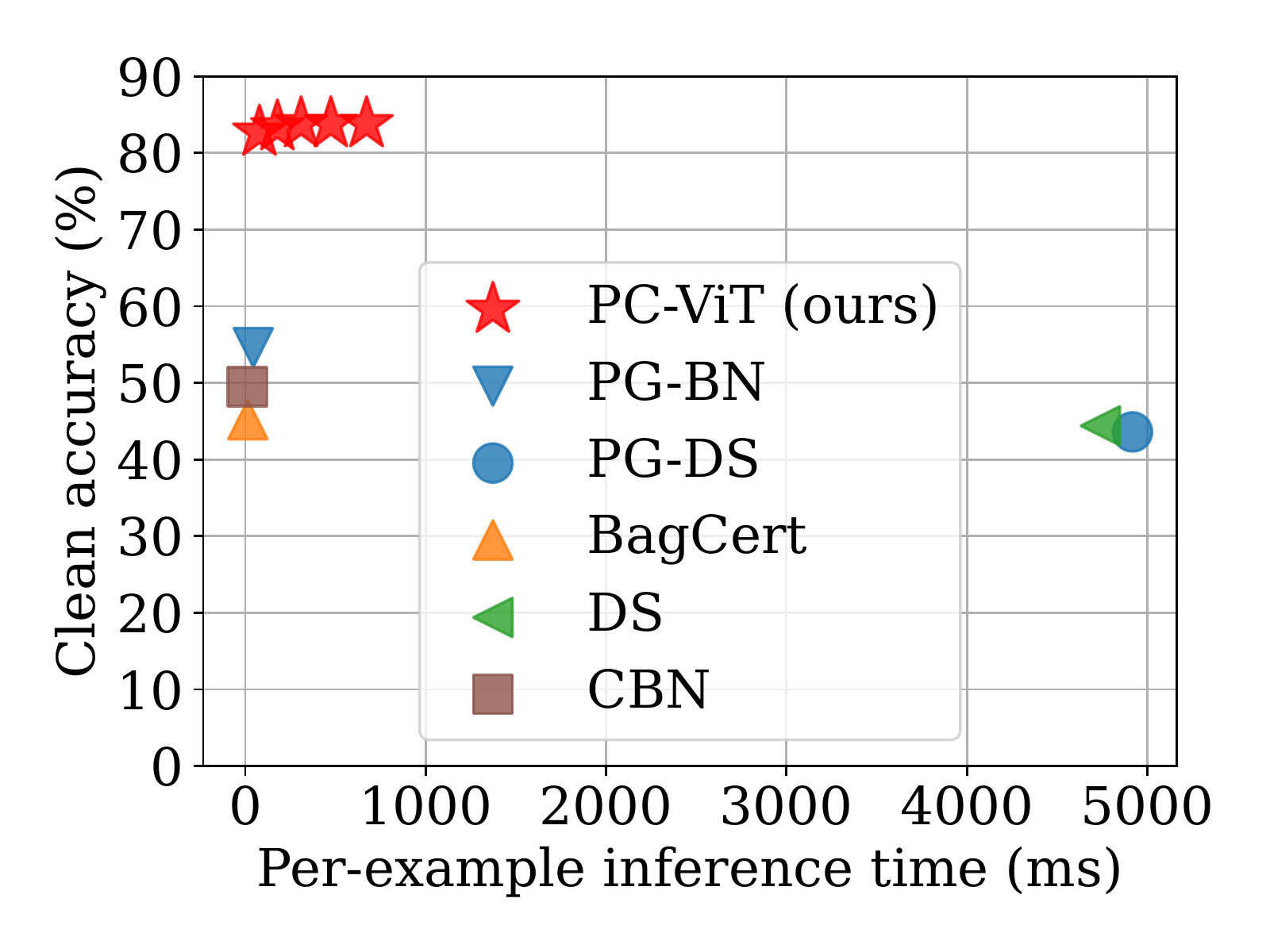}\\
\end{minipage}%
\begin{minipage}[b]{0.5\linewidth}
\includegraphics[width=\linewidth]{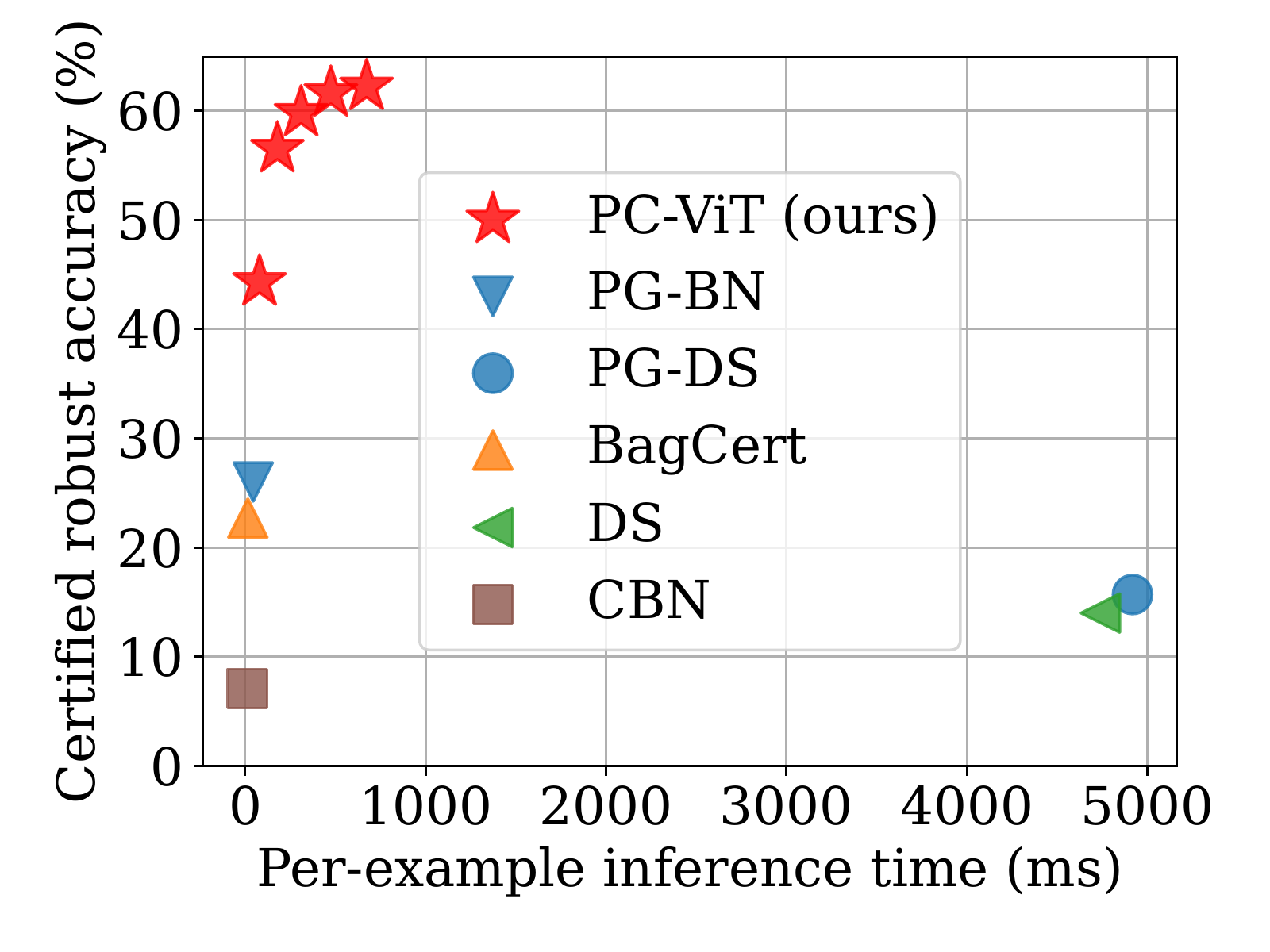}\\
\end{minipage}%
\vspace{-4em}
\caption{Trade-off between overhead and accuracy on ImageNet (left: clean accuracy; right: certified robust accuracy) }\label{fig-efficiency}
\end{minipage}%

\end{figure*}

\begin{figure*}
\centering
\begin{minipage}[b]{0.3\linewidth}
\includegraphics[width=\linewidth]{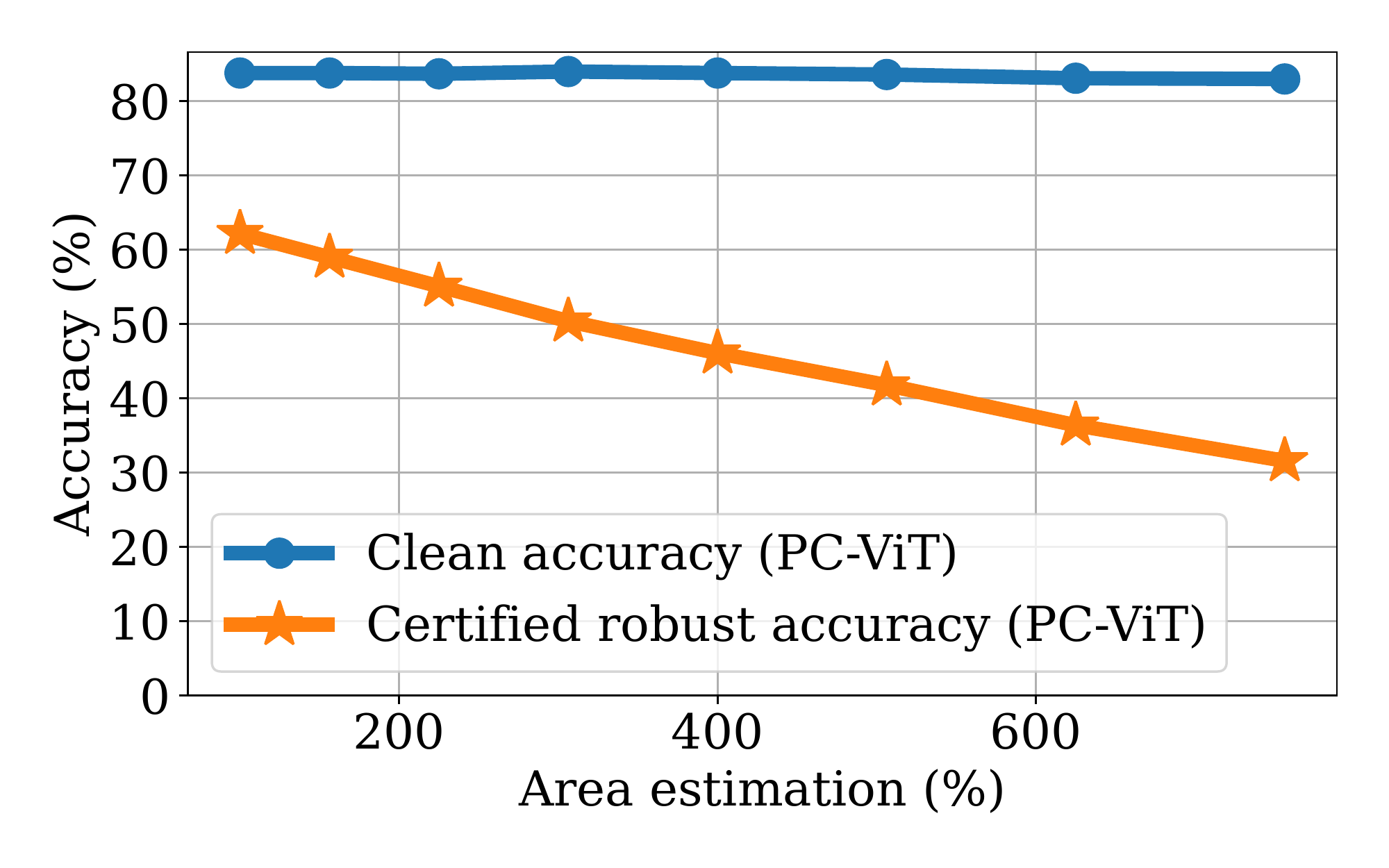}
\vspace{-3em}
\caption{Effect of over-estimated patch size for a 32$\times$32 patch (ImageNet)}\label{fig-over-mask-net}
\end{minipage}%
\quad
\begin{minipage}[b]{0.3\linewidth}
\includegraphics[width=\linewidth]{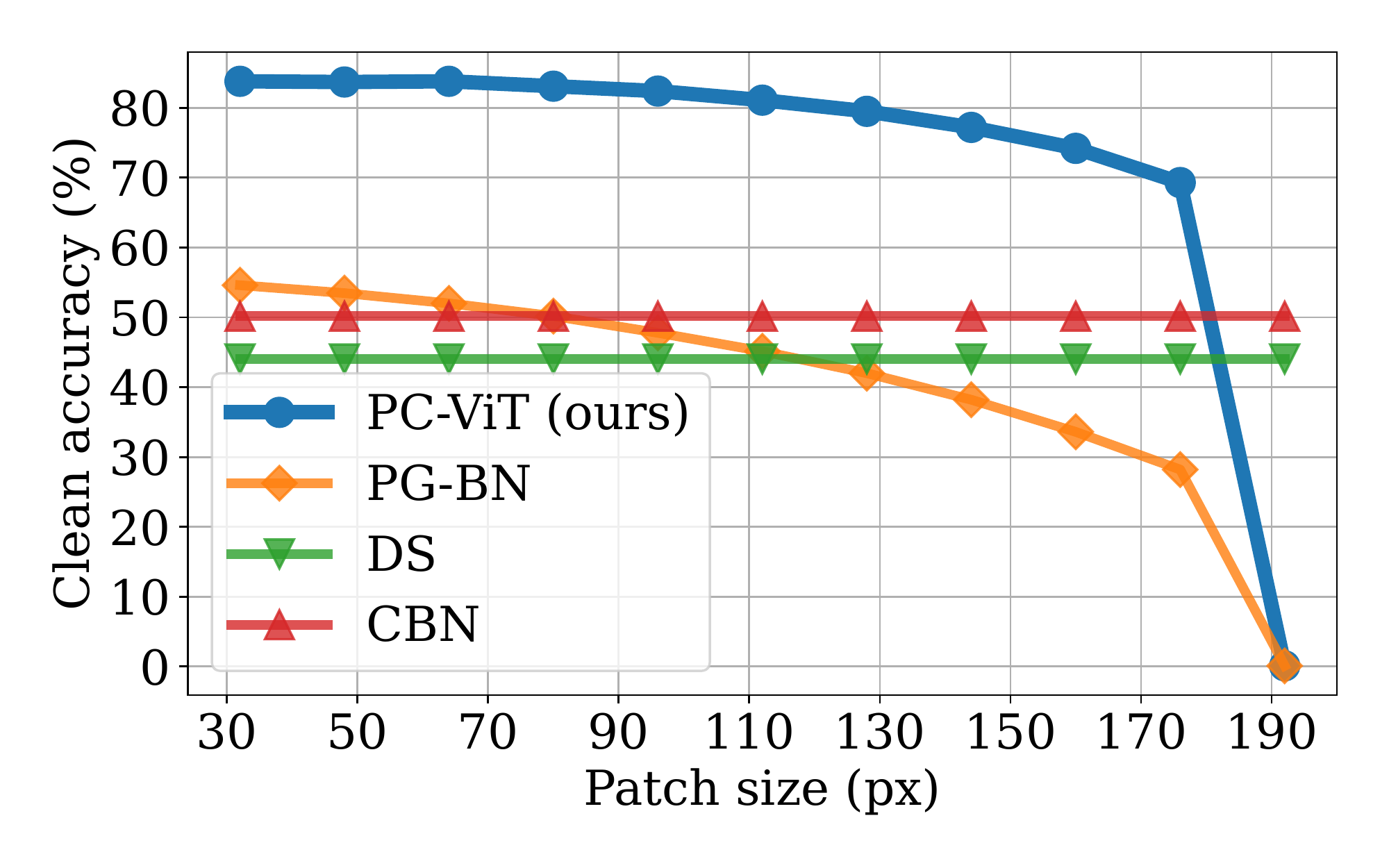}
\vspace{-3em}
\caption{Clean accuracy for defense setups against different patch sizes (ImageNet)} \label{fig-patch-size-clean-net}
\end{minipage}%
\quad
\begin{minipage}[b]{0.3\linewidth}
\includegraphics[width=\linewidth]{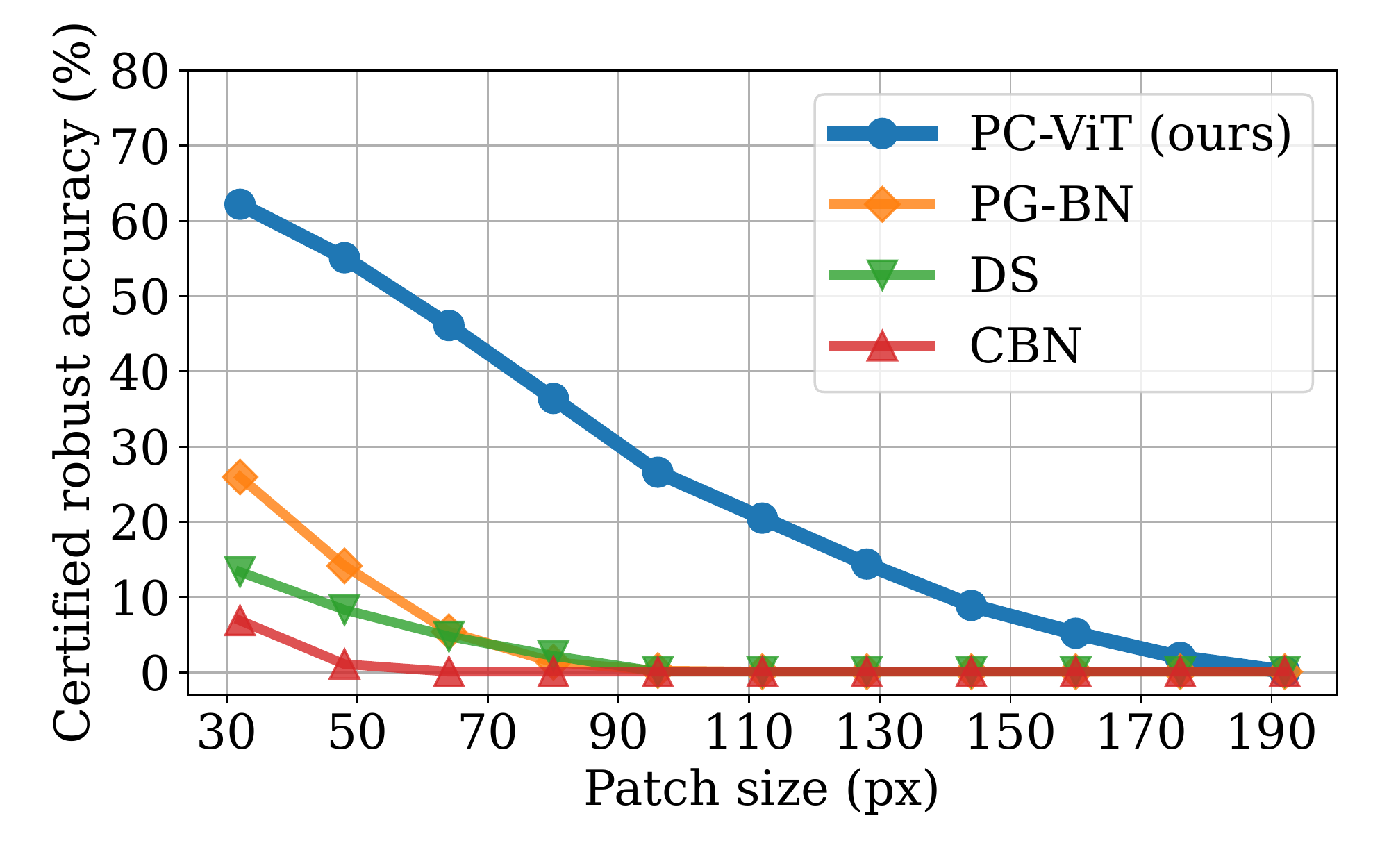}
\vspace{-3em}
\caption{Certified robust accuracy against different patch sizes (ImageNet)} \label{fig-patch-size-robust-net}
\end{minipage}%
\end{figure*}

\subsection{Detailed Analysis of \framework}\label{sec-detailed-eval}

In this subsection, we provide a detailed analysis of \framework models. We will discuss the trade-off between defense performance and defense overhead, study the implications of over-estimated patch sizes, and finally explore the limit of \framework against larger patches.  In Appendix~\ref{apx-dissect}, we analyze the defense performance on images with different object sizes and classes.

\textbf{There is a trade-off between defense performance and defense overhead (balanced by the number of masks).} In this analysis, we use PC-ViT against a 2\%-pixel patch on 5000 randomly selected ImageNet test images to study the trade-off between defense performance and defense overhead (similar analyses for ImageNette and CIFAR-10 are available in Appendix~\ref{apx-eval}). 
In Figure~\ref{fig-stride-net}, we report the clean accuracy, certified robust accuracy, and per-example inference time (evaluated using a batch size of one) for PC-ViT configured with different computation budgets (i.e., number of masks $k^2$). As shown in the figure, as we increase the number of masks, the certified robust accuracy first significantly improves and then gradually saturates. This is because a larger $k^2$ gives a smaller mask stride and leads to a smaller mask size, which enhances the robustness certification. However, we also observe that the per-example inference time greatly increases as we are using a larger number of masks. Therefore, we need to carefully choose a proper mask set size to balance the trade-off between defense performance and defense overhead. In our default setting, we prioritize the defense performance and use a mask set size of $6^2=36$.

We further visualize the defense overhead and defense performance (in terms of clean accuracy and certified robust accuracy) for different defenses in Figure~\ref{fig-efficiency}. As shown in the figure, CBN~\cite{zhang2020clipped} (12.0ms), PG-BN~\cite{xiang2021patchguard} (44.2ms), and BagCert~\cite{metzen2021efficient} (14.0ms) have a very small runtime since they only requires one-time model feed-forward inference. For PC-ViT, we report the performance trade-off under different mask set sizes ranging from 4 to 36 (we omit PC-ResNet and PC-MLP for simplicity; additional results for them are available in Appendix~\ref{apx-eval}). 
We can see that when PC-ViT is optimized for classification accuracy, we have 83.8\% clean accuracy and 62.2\% certified robust accuracy with a moderate defense overhead (672.4ms). On the other hand, when PC-ViT is optimized for defense efficiency, we achieve a small per-example inference time (78.8ms) while still significantly outperforming prior works in terms of clean accuracy (82.7\%) and certified robust accuracy (44.3\%). Furthermore, we note that prior works such as DS~\cite{levine2020randomized} and PG-DS~\cite{xiang2021patchguard} have a much larger defense overhead on ImageNet (4740.0 ms and 4918.0ms, respectively).

From this analysis, we demonstrate that there is a trade-off between defense strength and defense efficiency. In \framework, we can tune mask set size to balance this trade-off. In contrast, while prior works like PG-BN~\cite{xiang2021patchguard} and BagCert~\cite{metzen2021efficient} have a smaller inference time, they cannot further improve their defense performance regardless of additionally available computation resources. Finally, we argue that our defense can be applied to time-sensitive applications like video analysis by performing the defense on a subset of frames. We also note that we can significantly reduce the empirical inference time by running the masked prediction evaluation (i.e., $\textsc{MaskPred}(\cdot)$ in Algorithm~\ref{alg-prediction}) in parallel when multiple GPUs are available. With the improvement in computation resources and the development of high-performance lightweight models, we expect the computational cost to be mitigated in the future.

\textbf{Over-estimation of patch sizes has a small impact on the defense performance.} \framework requires a conservative estimation of the patch size  (as a security parameter) to generate a proper mask set (the dependence on patch size estimation is similar to several prior works~\cite{mccoyd2020minority,xiang2021patchguard,xiang2021patchguardpp}). In this analysis, we aim to study the defense performance when we over-estimate the patch size. In Figure~\ref{fig-over-mask-net}, we plot the defense performance as a function of an estimated patch area (i.e., the number of pixels) on the ImageNet dataset (the actual patch has 32$\times$32 pixels on the 224$\times$224 image; results for other datasets are in Appendix~\ref{apx-eval}).
The x-axis denotes the ratio of the estimated patch area to the actual patch area; 100\% implies no over-estimation. As shown in the figure, as the over-estimation becomes greater, the clean accuracy of PC-ViT is barely affected while the certified robust accuracy gradually drops. We note that even when the estimation of the patch area is conservatively set to 4 times the actual area of the patch, PC-ViT still significantly outperforms all prior works in terms of clean accuracy and certified robust accuracy.

\textbf{Understanding the limit of our defense with larger patch sizes.} In Figure \ref{fig-patch-size-clean-net} and \ref{fig-patch-size-robust-net}, we report the defense performance of PC-ViT against different patch sizes on the 224$\times$224 ImageNet test images (results for other datasets are in Appendix~\ref{apx-eval}).
This analysis helps us to understand the limit of \framework when facing extremely large adversarial patches. Figure~\ref{fig-patch-size-clean-net} shows that, as we increase the patch size, the clean accuracy of PC-ViT slowly decreases. For example, even when the patch size is 112$\times$112 (on the 224$\times$224 image), the clean accuracy is still above 80\%. The clean accuracy finally deteriorates to 0.1\% (random guess) when the patch is extremely large as 192$\times$192. Figure~\ref{fig-patch-size-robust-net} shows that the certified robust accuracy also decreases when a larger patch is used. When a large patch of 64$\times$64 is used, we have 46.1\% certified robust accuracy; when the patch is as large as 112$\times$112 (half of the image size), we still have a non-trivial top-1 certified robust accuracy of 20.5\% for 1000-class classification. We note that we use a fixed number of masks ($k^2=36$) for this analysis
; this shows that \framework performs well across different patch sizes when having a fixed computational budget.

We further plot the clean accuracy and certified accuracy of prior defenses~\cite{xiang2021patchguard,levine2020randomized,zhang2020clipped} in Figure~\ref{fig-patch-size-clean-net} and \ref{fig-patch-size-robust-net}. We can see that the certified robust accuracy of prior works drops quickly to zero when we consider a larger patch, while \framework achieves a much higher robust accuracy across all patch sizes. We note that the clean accuracy of CBN~\cite{zhang2020clipped} and DS~\cite{levine2020randomized} does not change due to their fixed defense parameters. When the certified robust accuracy of DS and CBN reduces to zero (at a patch size of 96px), \framework still has a much higher clean accuracy and certified robust accuracy.

\section{Discussion}\label{sec-discussion}
In this section, we quantitatively discuss the implications of
multiple patch shapes and multiple patches, the Minority Reports defense~\cite{mccoyd2020minority},  limitations and future work directions of \framework.

\subsection{\framework against Multiple Patch Shapes and Multiple Patches}\label{sec-discussion-multiple-shape-patch}

In Section~\ref{sec-evaluation}, we primarily focus on the scenario of one \textit{square} patch. In this subsection, we further demonstrate the compatibility of our defense with (1) a set of different patch shapes as well as (2) multiple patches.

\textbf{Intuition.} The key requirement of \framework is using an $\cR$-covering mask set $\cM$. Therefore, to counter an attacker who can use a patch from a set of different patch shapes or who can use multiple patches, we only need to consider a mask set that includes masks of different shapes or multiple masks to ensure $\cR$-covering and plug the mask set into our double-masking algorithm.

\textbf{Different patch shapes.} First, we consider a scenario where an attacker can use \textit{any} rectangle shape that covers at most 1\% pixels of the 224$\times$224 image (502 pixels), which includes thousands of shapes ranging from 1$\times$224 to 22$\times$22. To counter this strong attacker, we consider a shape set $\cS=\{5\times224, 12\times83, 23\times38, 39\times20, 84\times12, 224\times5\}$; we claim that 6 shapes in $\cS$ together can \textit{cover any 1\%-pixel rectangle shape}.\footnote{There are other valid shape sets $\cS$. Here, we only provide one example as proof of concept.} To prove this covering property, we let $a$ and $b$ be the height and width of the rectangle patch, respectively. We know that $a\cdot b < 502$ (1\% image pixels). If $a\leq 5$, then the patch is covered by the $5\times 224$ rectangle. If $ 5< a \leq  12$, then $b< 502/6 < 84$, and the patch is covered by the $12\times 83$ rectangle. If $12<a\leq23$, then $b< 502/13 < 39$, and the patch is covered by the $23\times 38$ rectangle. If $23<a\leq39$, then $b< 502/24 < 21$, and the patch is covered by the $39\times 20$ rectangle. If $39<a\leq84$, then $b<502/40<13$, and the patch is covered by the $84\times 12$ rectangle. Finally, if $84<a\leq224$, then $b<502/85<6$, and the patch is covered by $224\times 5$ rectangle. Now we have considered all 1\%-pixel rectangles and proved the covering property. We then generate masks for 6 different rectangles and use these masks in the double-masking algorithm.

We implement our strategy and report the defense performance for 500 randomly selected test images in the upper half of Table~\ref{tab-multiple-shape-patch}. We can see that our defense has high clean performance and certified robust accuracy. We note that the reported certified robust accuracy accounts for a much stronger attacker that can use any rectangle shape covering at most 1\% pixels, explaining a drop in certified robustness from the baseline. Nevertheless, our defense performance (while considering a  stronger adversary) is still much better than those of prior works  against a 1\%-pixel square patch (recall Table~\ref{tab-huge-provable}).\footnote{In some cases, we can see a higher clean accuracy for \framework against all rectangle shapes, compared to \framework against the square patch. This is because we are using a larger number of masks, and \framework could become less likely to output an incorrect disagreer label in the clean setting.}

\textbf{Multiple patches.} To handle multiple ($K$) patches, we can generate a mask set with all possible $K$-mask combinations, at least one of which can remove all patches. In order to certify the robustness of a given image, we need to check if the image predictions are correct for all $2K$-mask combinations. 

In the lower part of Table~\ref{tab-multiple-shape-patch}, we report defense performance for 500 random test images against two 1\%-pixel patches. Our defense achieves good defense performance against two patches (e.g., 98.8\% clean accuracy and 89.2\% certified robust accuracy for two 1\% square patches anywhere on the ImageNette images). Moreover, we compare the defense performance against a 2\%-pixel square patch (which has the same number of adversarial pixels). We can see that our defense performance for two patches is reduced compared to that for one-patch; however, the numbers are still much higher than prior works against a 2\%-pixel square patch in Table~\ref{tab-huge-provable}.

\begin{table}[t]
    \centering
        \caption{Defense against different patch shapes (\textit{all possible rectangle shapes that consist of 1\% image pixels}) and multiple patches (\textit{two 1\%-pixel square patches)}. }\label{tab-multiple-shape-patch}
    \resizebox{\linewidth}{!}
    {
   {
    \begin{tabular}{c|c|c|c|c|c|c}
    \toprule
          Dataset &  \multicolumn{2}{c|}{ImageNette} &\multicolumn{2}{c|}{ImageNet} & \multicolumn{2}{c}{CIFAR-10}  \\
         \midrule
     Accuracy (\%)   & clean &robust& clean &robust& clean &robust\\
         \midrule
    Any 1\% rectangle &99.2&91.8&85.4&49.8&99.4&82.6\\
    1\% square (baseline) &99.6&96.6&84.2&68.2&99.0&92.6\\
    \midrule
         Two 1\% square patches& 98.8&89.2&83.8&45.8&98.6&76.6 \\
     One 2\% square patch & 99.2&95.6 &83.8&63.2&99.0&91.2\\
     \bottomrule
    \end{tabular}   
         }}
 
\end{table}

\subsection{\framework and Minority Reports}\label{sec-discussion-mr}

In this paper, we propose \framework for certifiably robust prediction without abstention. In contrast, another pixel-masking defense, Minority Reports (MR)~\cite{mccoyd2020minority}, only achieves a weaker robustness notion for attack \emph{detection}: an attacker can force MR to always abstain from prediction. Though the certified robust accuracy for these two defenses with different robustness notions are not directly comparable, we implement MR and \framework using the same mask set and report their defense performance against a 2\%-pixel patch on ImageNet in Figure~\ref{fig-mr-net}.

First, we observe that MR can balance the trade-off between clean accuracy and certified robust accuracy. Second, MR achieves the highest certified robust accuracy (74.3\%) due to the easier certification for a weaker robustness notion. Third, \framework achieves a similar value of certified robust accuracy ($\sim$60\%) to MR when their clean accuracy is around 84\%. This is remarkable given that \framework  eliminates the issue of abstentions/alerts.

\begin{figure}[t]
    \centering
    \includegraphics[width=0.6\linewidth]{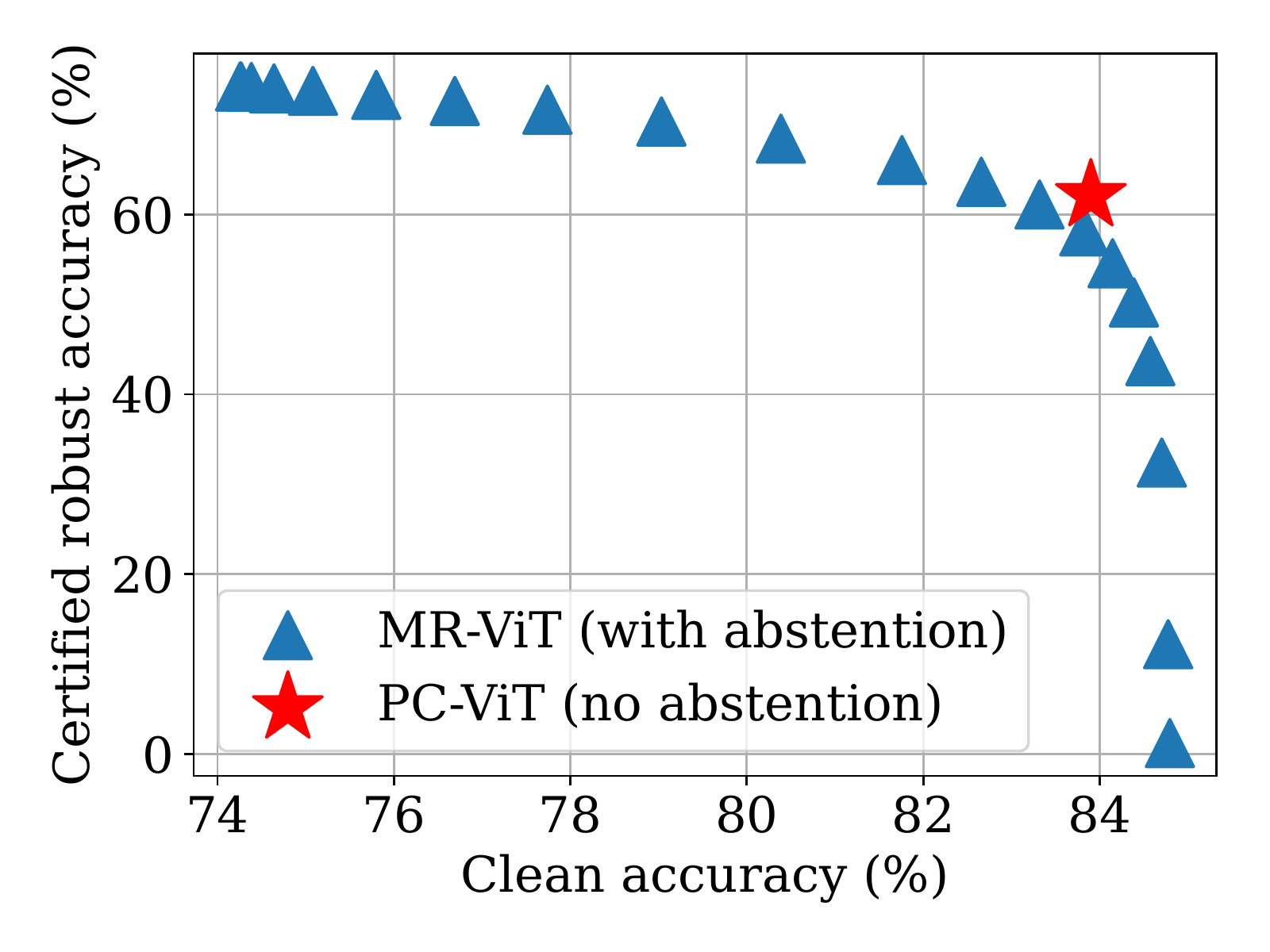}
    \vspace{-1.5em}
    \caption{Defense performance of PC-ViT and MR-ViT on ImageNet; note that robustness notions are different (robust prediction for \framework vs. attack detection for MR).}
    \label{fig-mr-net}
\end{figure}

\subsection{Limitation and Future Work}\label{sec-discussion-future}
In this subsection, we discuss the limitations and future work directions of our defense.

\textbf{Improving defense efficiency.} Compared to some prior works~\cite{zhang2020clipped,xiang2021patchguard,metzen2021efficient}, \framework achieves better performance at the cost of efficiency (recall Figure~\ref{fig-efficiency}). In Section~\ref{sec-detailed-eval} (Figure~\ref{fig-stride-net}), we also see a trade-off between defense performance and efficiency. How to improve the efficiency of the underlying model (e.g., EfficientNet~\cite{efficientnet}) and the algorithm (e.g., our alternative inference algorithm in Appendix~\ref{apx-new-inference}) is interesting to study. We note that \framework's runtime can be improved by evaluating masked predictions in parallel with multiple GPUs.

\textbf{Relaxing the prior estimation of the patch shape and patch size.} \framework requires a conservative estimation of the patch shape/size as the security parameters to generate the mask set. This dependence on the prior knowledge is similar to that of verifiably robust network training~\cite{gowal2018effectiveness,mirman2018differentiable} and empirical adversarial training~\cite{goodfellow2014explaining,madry2017towards} against global perturbations~\cite{szegedy2013intriguing,goodfellow2014explaining,carlini2017towards}, which need to know the norm and magnitude of the perturbations. This limitation is also shared by masking-based defenses~\cite{mccoyd2020minority,xiang2021patchguard,xiang2021patchguardpp}; an underestimated patch size/shape will undermine the robustness. How to relax the dependence on this prior knowledge is important to study. In Section~\ref{sec-discussion-multiple-shape-patch}, we demonstrate how to mitigate the dependence on prior knowledge of patch shape by considering all possible patch shapes and using a union of different mask shapes.

\textbf{Improving classification models' prediction invariance to pixel masking.} \framework is compatible with any classification models. Therefore, any improvement in image classification can benefit \framework. Notably, improving model prediction invariance to masking is crucial to model's robustness against adversarial patches. Relevant research questions include: (1) How to train a model to be robust to pixel masking (e.g., Cutout~\cite{cutout}, CutMix~\cite{yun2019cutmix})? (2) How to design model architecture that is inherently robust to pixel masking (e.g., ViT~\cite{vit}, CompNet~\cite{kortylewski2020compositional})? 

\textbf{Handling potential semantic changes caused by masks.} \framework uses masks to achieve substantial robustness against adversarial patches, and we have demonstrated its effectiveness on common image datasets. However, the masking operation might lead to semantic changes for special classification tasks (e.g., a classifier trained to recognize masks). In these special cases, we could use colored masks for \framework and further train the classifier to distinguish between vanilla masks and \framework masks. We leave further explorations for future work.


\section{Related Work}\label{sec-related-work}

\subsection{Adversarial Patch Attacks}
The adversarial patch attack was first introduced by Brown et
al. \cite{brown2017adversarial}; this attack focused on generating universal adversarial patches to induce model misclassification. Brown et al.~\cite{brown2017adversarial} demonstrated that the patch attacker can realize a physical-world attack by printing and attaching the patch to the victim objects. A concurrent paper on the Localized and Visible Adversarial Noise (LaVAN) attack~\cite{karmon2018lavan} aimed at inducing misclassification in the digital domain. Both of these papers operated in the white-box threat model, with access to the internals of the classifier under attack. PatchAttack~\cite{yang2020patchattack}, on the other hand, proposed a reinforcement learning based attack for generating adversarial patches in the black-box setting.

There have been adversarial patch attacks proposed in other domains such as object detection~\cite{liu2018dpatch}, semantic segmentation~\cite{sehwag2018not}, and network traffic analysis~\cite{shan2021real}. In this paper, we focus on test-time attacks against image classification models.

\subsection{Adversarial Patch Defenses}\label{sec-related-work-defense}

To counter the threat of adversarial patches,  heuristic-based empirical defenses, Digital Watermark (DW)~\cite{hayes2018visible} and Local Gradient Smoothing (LGS)~\cite{naseer2019local}, were first proposed. However, Chiang et al.~\cite{chiang2020certified} had shown that these defenses were ineffective against an adaptive attacker with the knowledge of the defense algorithm and model parameters~\cite{chiang2020certified}.

The ineffectiveness of empirical defenses has inspired many certifiably robust defenses. Chiang et al.~\cite{chiang2020certified} proposed the first certifiably robust defense against adversarial patches via Interval Bound Propagation (IBP)~\cite{gowal2018effectiveness,mirman2018differentiable}, which conservatively bounded the activation values of neurons to derive a robustness certificate. This defense requires expensive training and does not scale to large models and high-resolution images. Zhang et al.~\cite{zhang2020clipped} proposed Clipped BagNet (CBN) to clip features of BagNet (a classification model with small receptive fields) for certified robustness. Levine et al.~\cite{levine2020randomized} proposed De-randomized Smoothing, which fed small image regions to a classification model and performed majority voting for the final prediction. Xiang et al.~\cite{xiang2021patchguard} proposed PatchGuard as a general defense framework with two key ideas: the use of small receptive fields and secure feature aggregation. 
Metzen et al.~\cite{metzen2021efficient} proposed BagCert, a variant of BagNet with majority voting, for certified robustness. 

A key takeaway from our paper is that the dependence of prior works on specific model architectures (e.g., small receptive fields~\cite{zhang2020clipped,levine2020randomized,xiang2021patchguard,metzen2021efficient,cropping}) greatly limits the defense performance; in contrast, the compatibility of \framework with any model architecture leads to state-of-the-art clean accuracy and certified robust accuracy.

Another line of certifiably robust research focuses on attack detection. Minority Reports (MR)~\cite{mccoyd2020minority} places a mask at all image locations and uses the inconsistency in masked prediction voting grids as an attack indicator. PatchGuard++~\cite{xiang2021patchguardpp} performed a similar defense in the feature space. 
We note that the first-round masking of \framework is is similar to the masking operation of MR; we provided a detailed comparison in Section~\ref{sec-discussion-mr}. A concurrent work ScaleCert~\cite{han2021scalecert} uses superficial important neurons to detect a patch attack; we omit its detailed discussion due to different defense objectives.  

Some other recent defenses focus on adversarial training and robust model architecture~\cite{wu2019defending,rao2020adversarial,Mu2021defending,cosgrove2020robustness}, but they lack certifiable robustness guarantees. In other domains like object detection, empirical defenses~\cite{saha2020role,ji2021adversarial,liang2021we} and certifiably robust defenses~\cite{xiang2021detectorguard} have also been proposed. We omit a detailed discussion since \framework focuses on certifiably robust image classification. 

\subsection{Other Adversarial Example Attacks}

In addition to adversarial patch attacks and defenses, there is a significant body of work on adversarial examples. 
Conventional adversarial attacks~\cite{barreno2010security,szegedy2013intriguing,biggio2013evasion,goodfellow2014explaining,papernot2016limitations,carlini2017towards,madry2017towards} aim to introduce a small global $L_p$ perturbation to the image for model misclassification. Empirical defenses~\cite{papernot2016distillation,xu2017feature,meng2017magnet,metzen2017detecting} were first proposed to mitigate the threat of adversarial examples, but were later found vulnerable to a strong adaptive attacker with the knowledge of the defense setup~\cite{carlini2017adversarial,athalye2018obfuscated,tramer2020adaptive}. The fragility of these heuristic-based defenses inspired a new research thread on developing certifiably robust defenses~\cite{raghunathan2018certified,wong2017provable,lecuyer2019certified,cohen2019certified,salman2019provably,gowal2018effectiveness,mirman2018differentiable}. In contrast, we focus on adversarial patch attacks, whose perturbations are localized and thus are realizable in the physical world.
\section{Conclusion}
In this paper, we propose \framework for certifiably robust image classification against adversarial patch attacks. Notably, \framework is compatible with any state-of-the-art classification model (including ones with large receptive fields). \framework uses a double-masking algorithm to remove all adversarial pixels and recover the correct prediction without any abstention. Our evaluation shows that \framework outperforms all prior works by a large margin: it is the first certifiably robust defense that achieves clean accuracy comparable to state-of-the-art vanilla models while simultaneously achieving high certified robust accuracy. \framework thus represents a promising new direction in our quest for secure computer vision systems.

\section*{Acknowledgements}%
We are grateful to David Wagner for shepherding the paper and anonymous reviewers at USENIX Security for their valuable feedback. We are also grateful to Ashwinee Panda, Sihui Dai, Alexander Valtchanov, Xiangyu Qi, and Tong Wu for their insightful comments on the paper draft.
This work was supported in part by the National Science Foundation under grants CNS-1553437 and CNS-1704105, the ARL’s Army Artificial Intelligence Innovation Institute (A2I2), the Office of Naval Research Young Investigator Award, the Army Research Office Young Investigator Prize, Schmidt DataX award, and Princeton E-ffiliates Award.


\bibliographystyle{plain}
\bibliography{reference-short.bib}

\appendix

\section{Details of Experiment Setup}\label{apx-setup}
In this section, we provide additional details of our experiment setup. We release our source code at \url{https://github.com/inspire-group/PatchCleanser}.

\textbf{Additional Datasets.} In Appendix~\ref{apx-eval}, we will include evaluation results for three additional datasets (Flowers-102~\cite{flowers}, CIFAR-100~\cite{cifar}, and SVHN~\cite{svhn}) to demonstrate the general applicability of \framework.

\textit{Flowers-102.} The Flowers-102~\cite{flowers} dataset has labels for 102 different flower categories. Images are in high-resolution; we resize and crop them into 224$\times$224 pixels.

\textit{CIFAR-100.} CIFAR-100~\cite{cifar} is a low resolution image dataset with 100 classes from common objects. The original resolution of CIFAR-100 images is 32$\times$32; we resize them to 224$\times$224 for better classification performance. 

\textit{SVHN.} The Street View House Numbers (SVHN)~\cite{svhn} dataset is a low-resolution image dataset for digit recognition (10 different digits). Each image has a resolution of 32$\times$32  and is centered around a single digit. We resize all images to 224$\times$224.

\textbf{Details of Model training.} For all high-resolution images (i.e., ImageNet~\cite{deng2009imagenet}, ImageNette~\cite{imagenette}, Flowers-102~\cite{flowers}), we resize and crop them into 224$\times$224. For low-resolution images (i.e., CIFAR-10~\cite{cifar}, CIFAR-100~\cite{cifar}, SVHN~\cite{svhn}), we resize them to 224$\times$224 (via bicubic interpolation) without cropping. We use \texttt{timm} library~\cite{rw2019timm} to  build all vanilla models and load weights trained for ImageNet~\cite{deng2009imagenet}. In order to train a model for other datasets, we set the batch size to 64 and use an SGD optimizer with a momentum of 0.9 and an initial learning rate of 0.001 (for ViT and ResMLP) or 0.01 (for ResNet). We train each model for 10 epochs and reduce the learning rate by a factor of 10 every 5 epochs. 

In our default setting, we use Cutout data augmentation~\cite{cutout} for the model training. Specifically, we apply 2 masks of size 128$\times$128 at random locations to the $224\times224$ training images; this training-time data augmentation can improve model prediction invariance to pixel masking. We note that the Cutout training is only an optional step in \framework pipeline. In Appendix~\ref{apx-eval}, we will report the defense performance with and without Cutout training to analyze its effect. Finally, we reiterate that the Cutout training is a data augmentation technique; therefore, it is agnostic to model architectures and only incurs a small training overhead.

\textbf{Details of software and hardware.} Our implementation is based on PyTorch~\cite{pytorch}, and we use \texttt{timm}~\cite{rw2019timm} to build models and load pretrained weights. The runtime evaluation is done on a workstation with 48 Intel Xeon Silver 4214 CPU cores, 384 GB RAM, and 8 GeForce RTX 2080 Ti GPUs. We use a batch size of 1 on one GPU to obtain the average per-example runtime in Section~\ref{sec-detailed-eval} and Appendix~\ref{apx-eval}.

\section{Challenger Masking: Improving Inference Complexity}\label{apx-new-inference}

Our double-masking defense (Algorithm~\ref{alg-prediction} in Section~\ref{sec-two-round}) has inference complexity of $O(|\cM|^2)$ in the worst case (doing all two-mask predictions). In this subsection, we introduce a new inference algorithm named challenger masking, which has better worst-case complexity of $O(|\cM|)$, the same certified robust accuracy, but slightly lower clean accuracy. Similar to our double-masking algorithm (Algorithm~\ref{alg-prediction} in Section~\ref{sec-two-round}), the challenger masking involves two rounds of masking: if the first-round masking reaches a unanimous agreement on masked predictions, we return the agreed prediction label; otherwise, we play a challenger game (in the second-round masking) to settle the disagreement.  

\textbf{Challenger game.} The high-level idea of the challenger game is to let different masked predictions challenge each other and output the game-winner as the final prediction. For two masks $\mathbf{m}_0,\mathbf{m}_1$ that give different masked predictions (i.e., $\bar{y}_0\neq \bar{y}_1, \bar{y}_0 = \mathbb{F}(\mathbf{x}\odot\mathbf{m}_0),\bar{y}_1=\mathbb{F}(\mathbf{x}\odot\mathbf{m}_1)$), we apply both two masks to the image and evaluate the two-mask prediction as $\hat{y}=\mathbb{F}(\mathbf{x}\odot\mathbf{m}_0\odot\mathbf{m}_1)$. If the two-mask prediction $\hat{y}$ agrees with any of the one-mask prediction $\bar{y}_0$ or $\bar{y}_1$, we consider the agreed prediction as the winner of this challenger game. Our algorithm will discard a mask once it loses any challenger game and continue to play the game until there is only one label left (i.e., no challenger exists). Finally, we output the winner label as the robust prediction. Intuitively, if the first-round mask removes the patch, then adding a second mask is unlikely to give a different prediction (since the second mask is applied to a benign image). Therefore, the mask that removes the patch has a great chance to win this challenger game.

\begin{algorithm}[t]
    \centering
    \caption{Challenger masking algorithm}\label{alg-prediction-new}
    \begin{algorithmic}[1]
    \renewcommand{\algorithmicrequire}{\textbf{Input:}}
    \renewcommand{\algorithmicensure}{\textbf{Output:}}
    \Require Image $\mathbf{x}$, vanilla prediction model $\mathbb{F}$, mask set $\mathcal{M}$
    \Ensure  Robust prediction $\bar{y}$ 
    \Procedure{ChallngerMasking}{$\mathbf{x},\mathbb{F},\mathcal{M}$}

    \State $\mathcal{P}\gets\varnothing$
    \For{$\mathbf{m}\in\mathcal{M}$} \Comment{Enumerate every mask $\mathbf{m}$}
    \State $\bar{y}\gets \mathbb{F}(\mathbf{x}\odot\mathbf{m})$ \Comment{Do masked prediction}
    \State $\mathcal{P}\gets\mathcal{P}\ \bigcup \ \{(\mathbf{m},\bar{y})\}$ \Comment{Update set $\mathcal{P}$}
    \EndFor
    \LeftCommenta{Sample a \textit{winner candidate} $(\mathbf{m}^*,\bar{y}^*)$}
    \State $(\mathbf{m}^*,\bar{y}^*)\gets\text{a random sample from }\mathcal{P}$
    \State $\mathcal{L}\gets \{\bar{y} \in \mathcal{Y}\ | \ \exists\  \mathbf{m}\in\mathcal{M},(\mathbf{m},\bar{y})\in\mathcal{P} \}$ \Comment{Label set $\mathcal{L}$}
    \While{$|\mathcal{L}| > 1$}\Comment{Disagreement in predictions}
    \LeftCommentb{Sample a \textit{challenger} $(\mathbf{m},\bar{y})$}
    \State $(\mathbf{m},\bar{y})\gets\text{a random sample from }\mathcal{P}$ s.t. $\bar{y}\neq \bar{y}^*$
    \State $\hat{y}\gets \mathbb{F}(x\odot\mathbf{m}\odot \mathbf{m^*})$\Comment{Two-mask prediction}
    \If{$\hat{y}=\bar{y}$}\Comment{The challenger wins}
    \State $\mathcal{P}\gets\mathcal{P}-\{(\mathbf{m}^*,\bar{y}^*)\}$ \Comment{Remove the candidate}
    \State $(\mathbf{m^*},\bar{y}^*)\gets (\mathbf{m},\bar{y})$ \Comment{Update the candidate}
    \Else\Comment{The challenger loses}
    \State $\mathcal{P}\gets\mathcal{P}-\{(\mathbf{m},\bar{y})\}$ \Comment{Remove the challenger}
    \EndIf
    \State $\mathcal{L}\gets \{\bar{y} \in \mathcal{Y}\ | \ \exists\  \mathbf{m}\in\mathcal{M},(\mathbf{m},\bar{y})\in\mathcal{P} \}$ \Comment{New $\mathcal{L}$}
    \EndWhile
 
    \State\Return $\bar{y}^*$\Comment{Return the candidate prediction}
    \EndProcedure
    \end{algorithmic}
\end{algorithm}

\textbf{Algorithm details.} The pseudocode of our challenger masking algorithm is presented in Algorithm~\ref{alg-prediction-new}. To start with, we do first-round masking to gather the prediction set $\cP$. Next, we randomly sample a mask-prediction pair as our (temporary) winner candidate (i.e., $\mathbf{m}^*,\bar{y}^*$). Furthermore, we generate a label set $\mathcal{L}$ that consists of all prediction labels in the prediction set $\cP$. If the size of the label set $\cL$ equals one, it means that the first-round masking has a unanimous agreement, and we will skip the challenge game (i.e., the while loop) and return the agreed prediction $\bar{y}^*$. On the other hand, when $|\cL|>1$, i.e., the first-round masking has a disagreement, we go into the while loop for our challenger game.

\begin{figure}[t]
    \centering
    \includegraphics[width=0.9\linewidth]{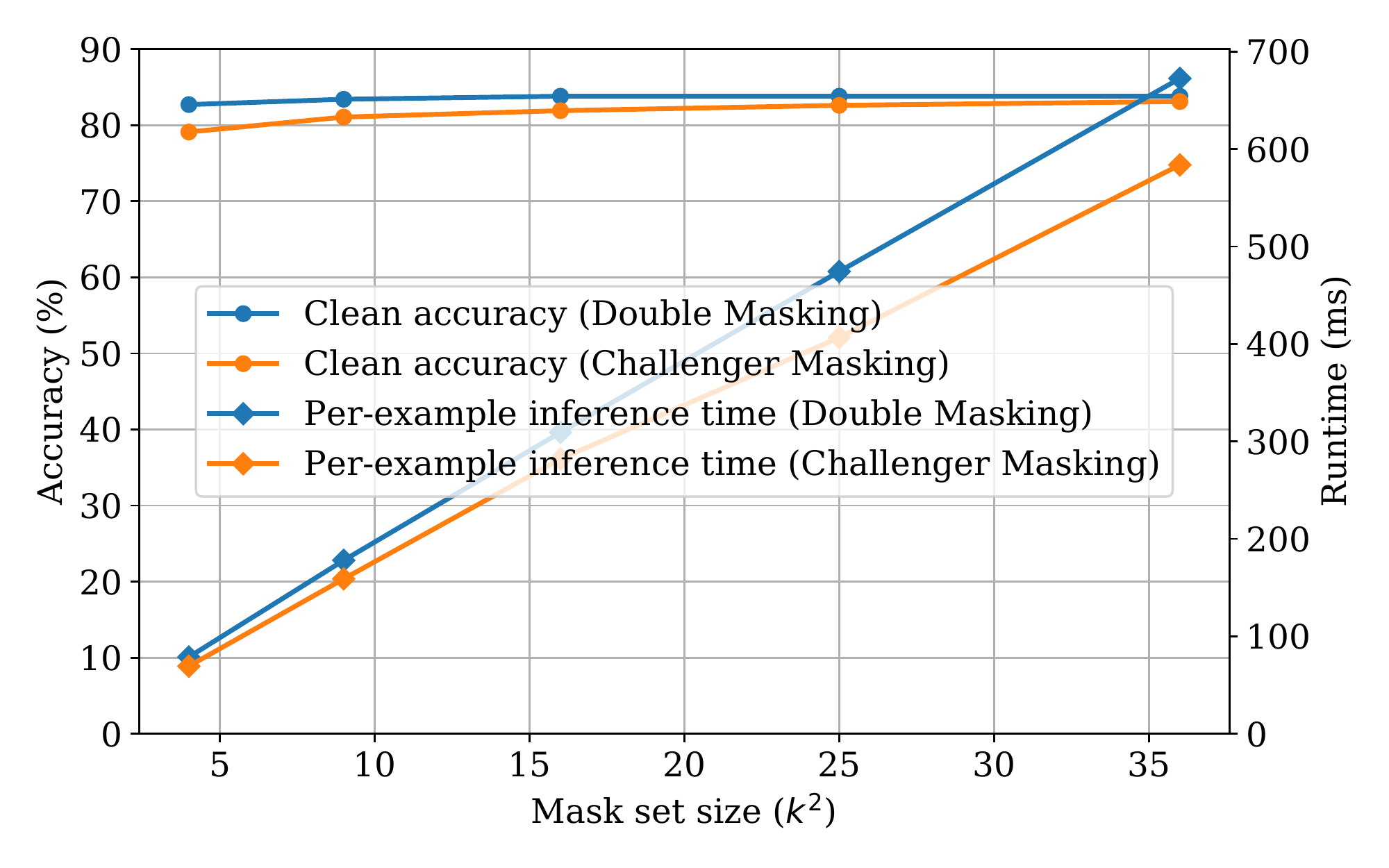}
       \vspace{-1em}
    \caption{Comparison of clean accuracy and per-example inference time between two inference algorithms}
    \label{fig-pc2}
\end{figure}

In the challenger game, we first sample a random challenger $\mathbf{m},\bar{y}$ whose prediction disagrees with the winner candidate (i.e., $\bar{y}\neq\bar{y}^*$). Then, we evaluate two-mask prediction $\hat{y}=\mathbb{F}(\mathbf{x}\odot\mathbf{m}\odot\mathbf{m}^*)$. If $\hat{y}=\bar{y}$, the challenger wins the game, we remove the loser $\mathbf{m}^*,y^*$ from the prediction set $\cP$ and update the winner candidate as $\mathbf{m},\bar{y}$. If $\hat{y}\neq\bar{y}$, we consider the challenge is unsuccessful, we remove the challenger $\mathbf{m},\bar{y}$ from the set $\cP$. After the removal of the loser, we update the label set $\cL$. If there is still more than one prediction labels, we move to the next iteration of the challenger game. Finally, we will have a winner and we return $\bar{y}^*$ as the final prediction.

\begin{figure*}
\centering
\begin{minipage}[b]{0.32\linewidth}
\includegraphics[width=\linewidth]{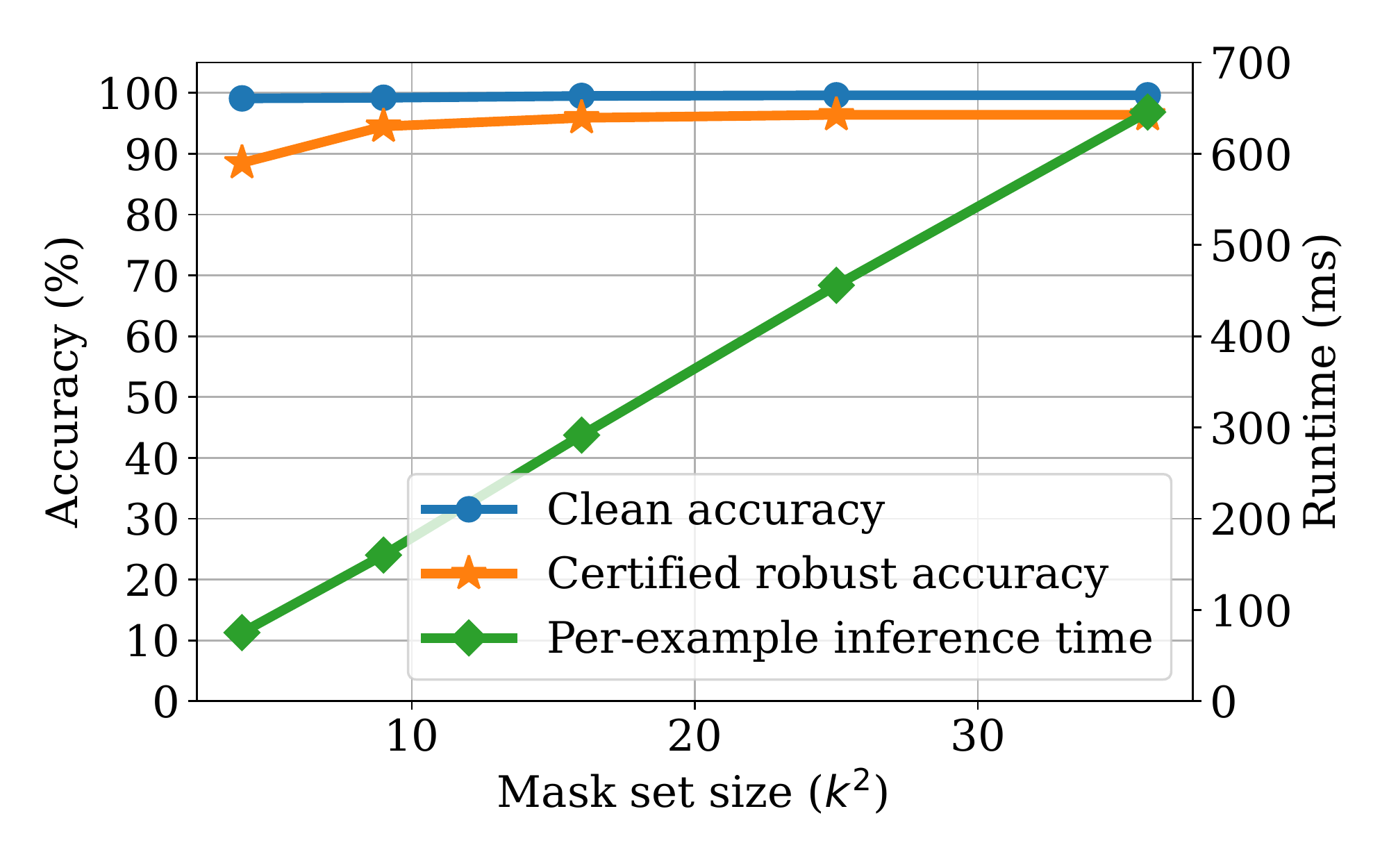}
\vspace{-2em}
\caption{The effect of mask set size on defense performance (ImageNette)}\label{fig-stride-nette}
\end{minipage}%
\quad
\begin{minipage}[b]{0.32\linewidth}
\includegraphics[width=\linewidth]{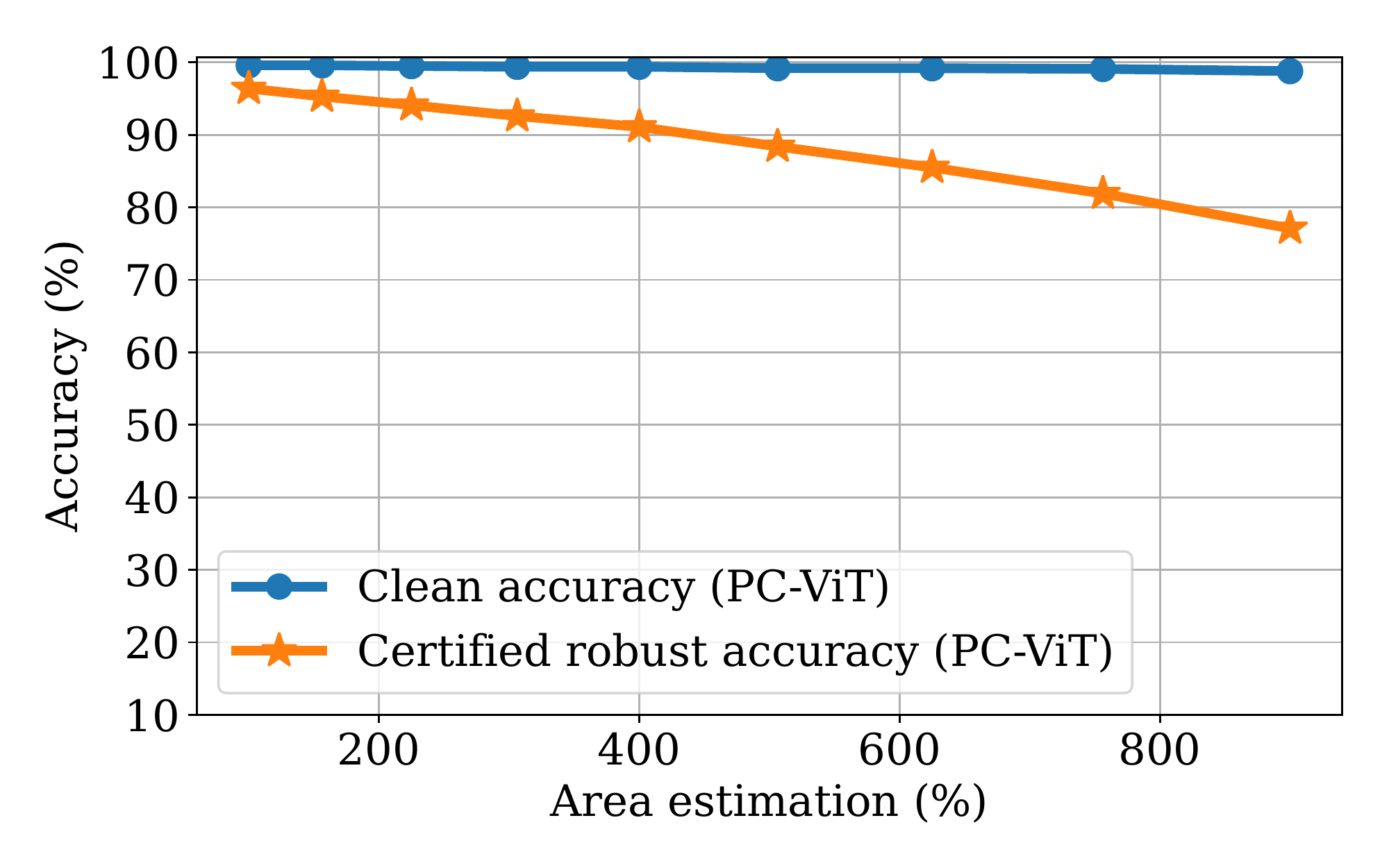}
\vspace{-2em}
\caption{Effect of over-estimated patch size for a 32$\times$32 patch (ImageNette)}\label{fig-over-mask-nette}
\end{minipage}%
\quad
\begin{minipage}[b]{0.32\linewidth}
\includegraphics[width=\linewidth]{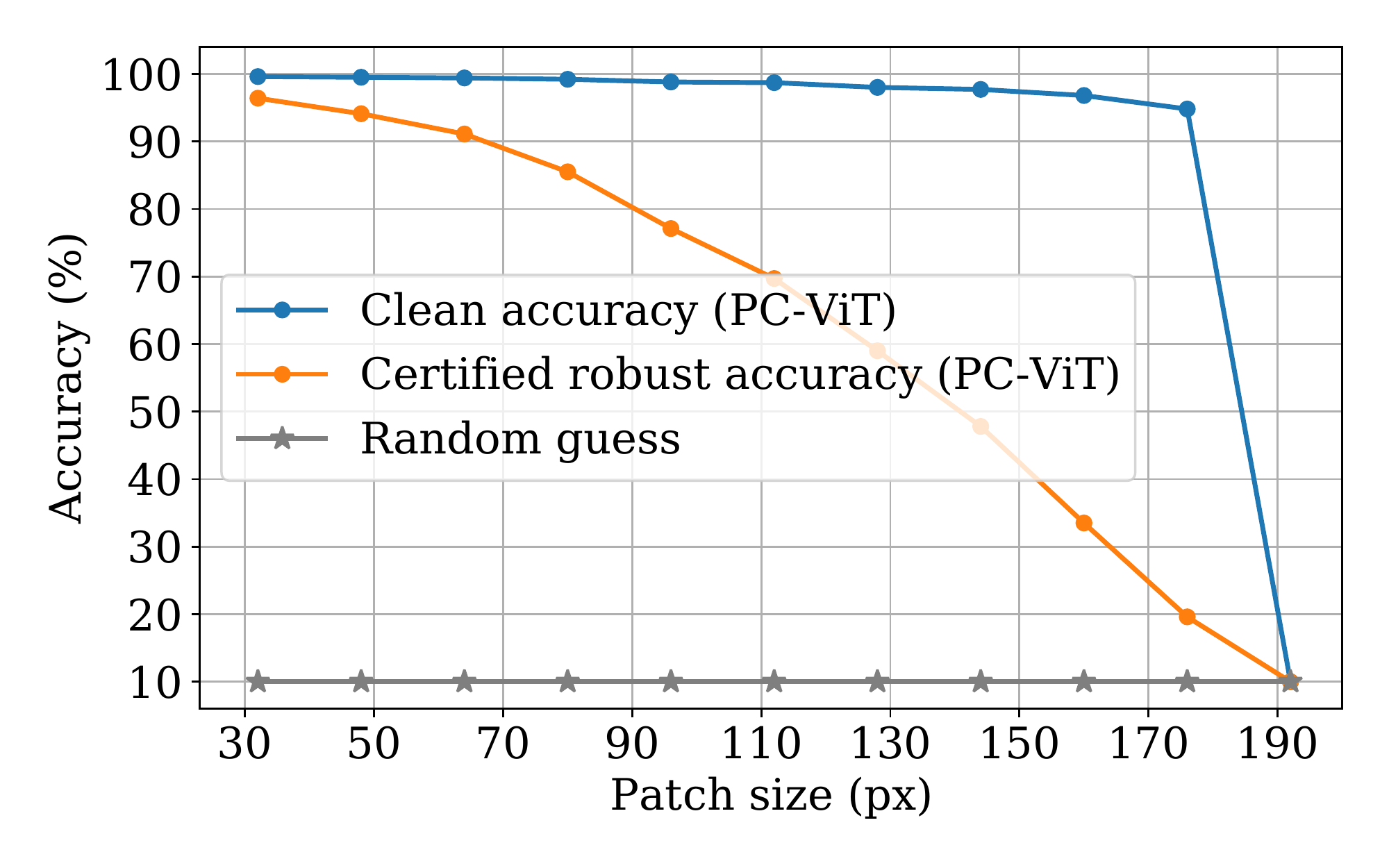}
\vspace{-2em}
\caption{Defense performance against different patch sizes on ImageNette} \label{fig-patch-size-nette}
\end{minipage}%
\end{figure*}
\begin{figure*}
\centering
\begin{minipage}[b]{0.32\linewidth}
\includegraphics[width=\linewidth]{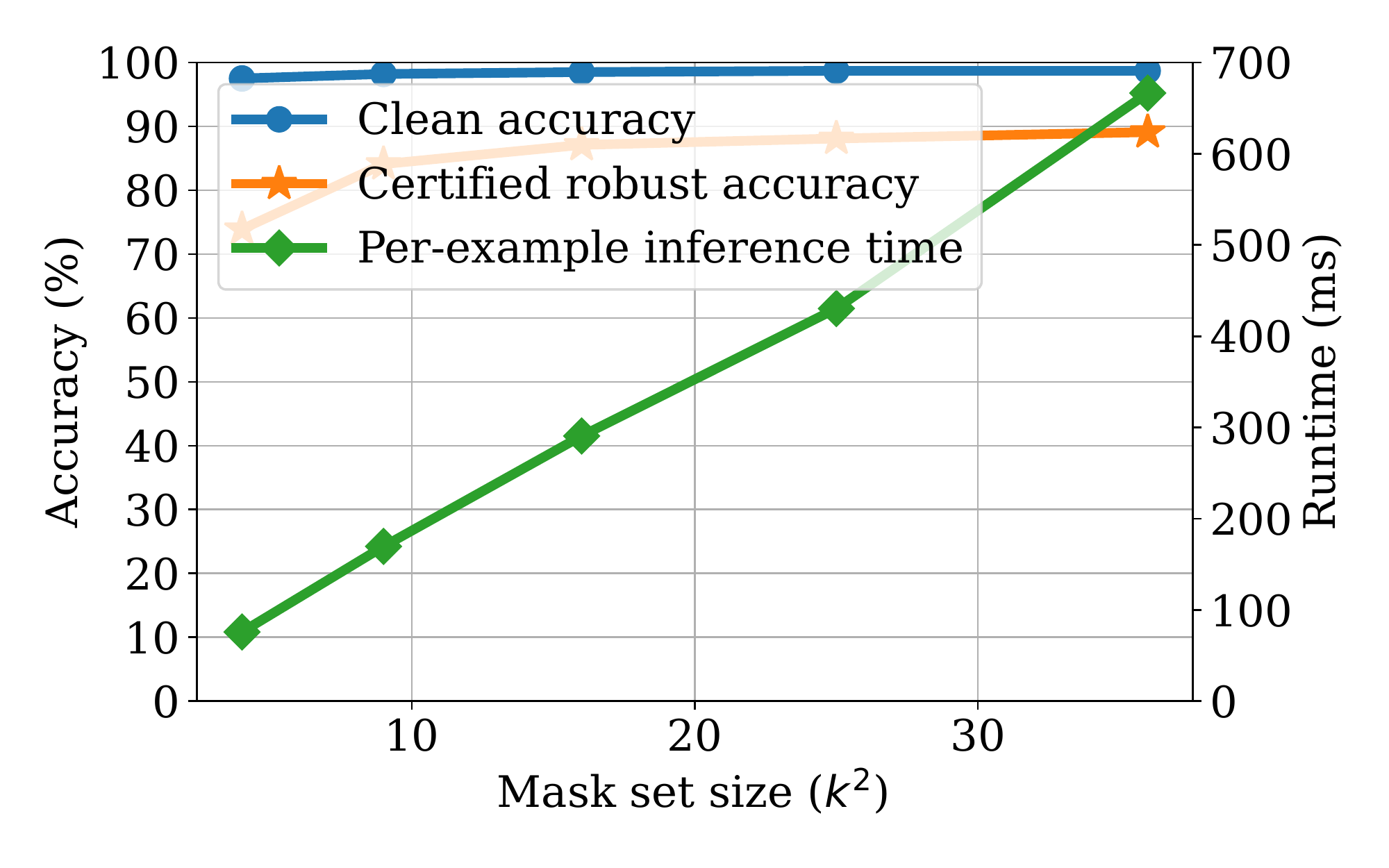}
\vspace{-2em}
\caption{The effect of mask set size on defense performance (CIFAR-10)}\label{fig-stride-cifar}
\end{minipage}%
\quad
\begin{minipage}[b]{0.32\linewidth}
\includegraphics[width=\linewidth]{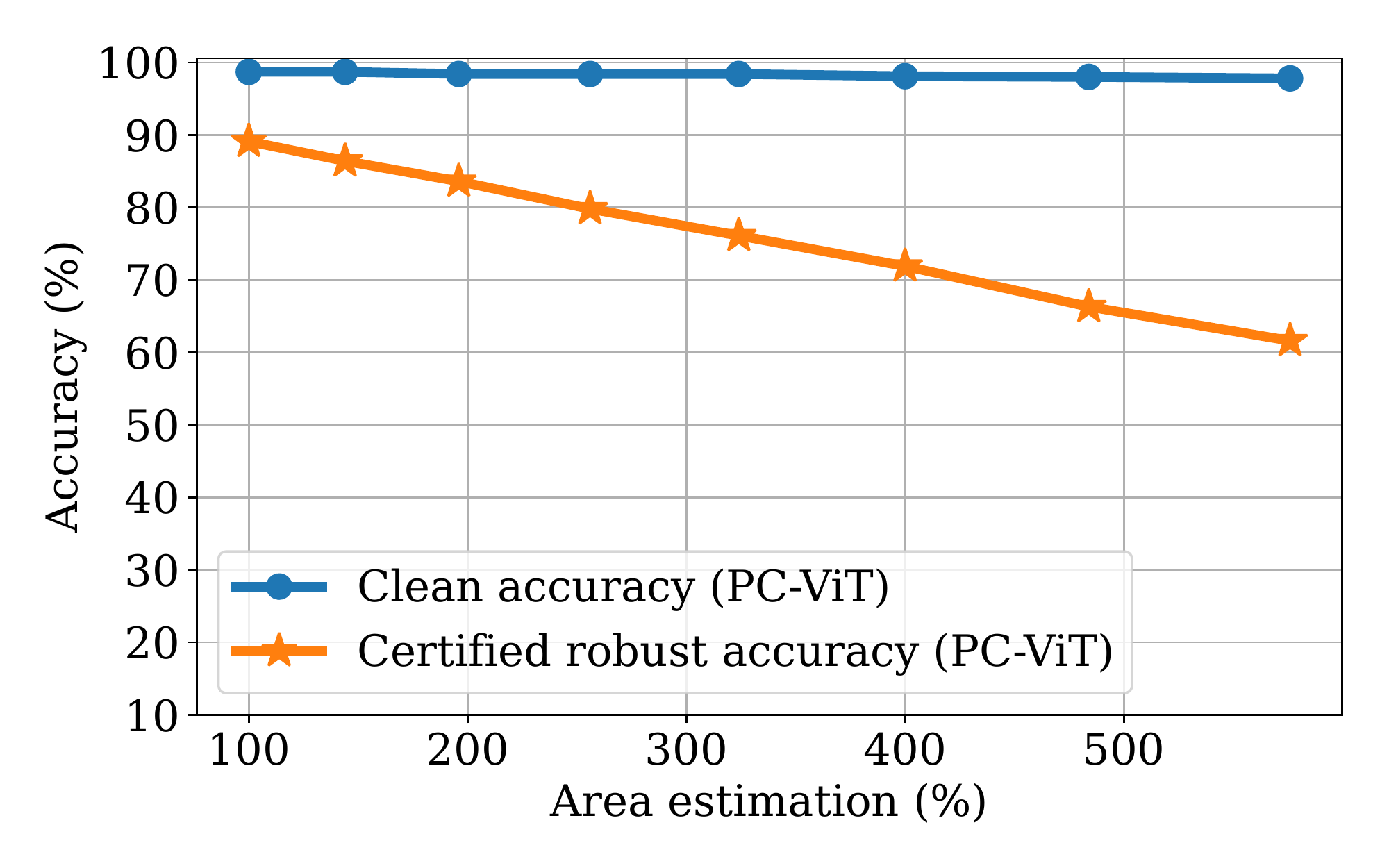}
\vspace{-2em}
\caption{Effect of over-estimated patch size for a 32$\times$32 patch (CIFAR-10)}\label{fig-over-mask-cifar}
\end{minipage}%
\quad
\begin{minipage}[b]{0.32\linewidth}
\includegraphics[width=\linewidth]{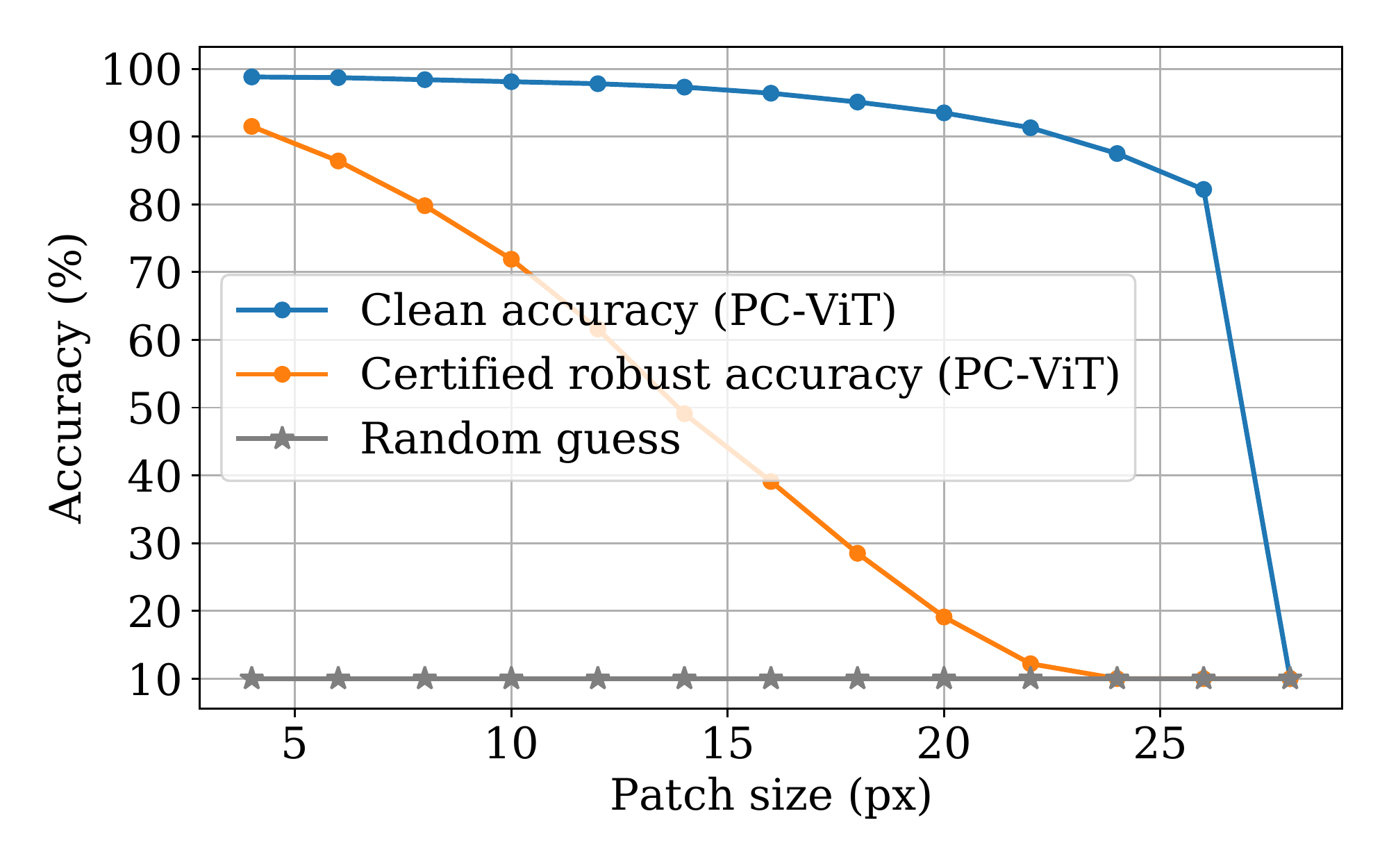}
\vspace{-2em}
\caption{Defense performance against different patch sizes on CIFAR-10} \label{fig-patch-size-cifar}
\end{minipage}%
\end{figure*}

\begin{figure}[t]
\begin{minipage}[b]{0.5\linewidth}
\includegraphics[width=\linewidth]{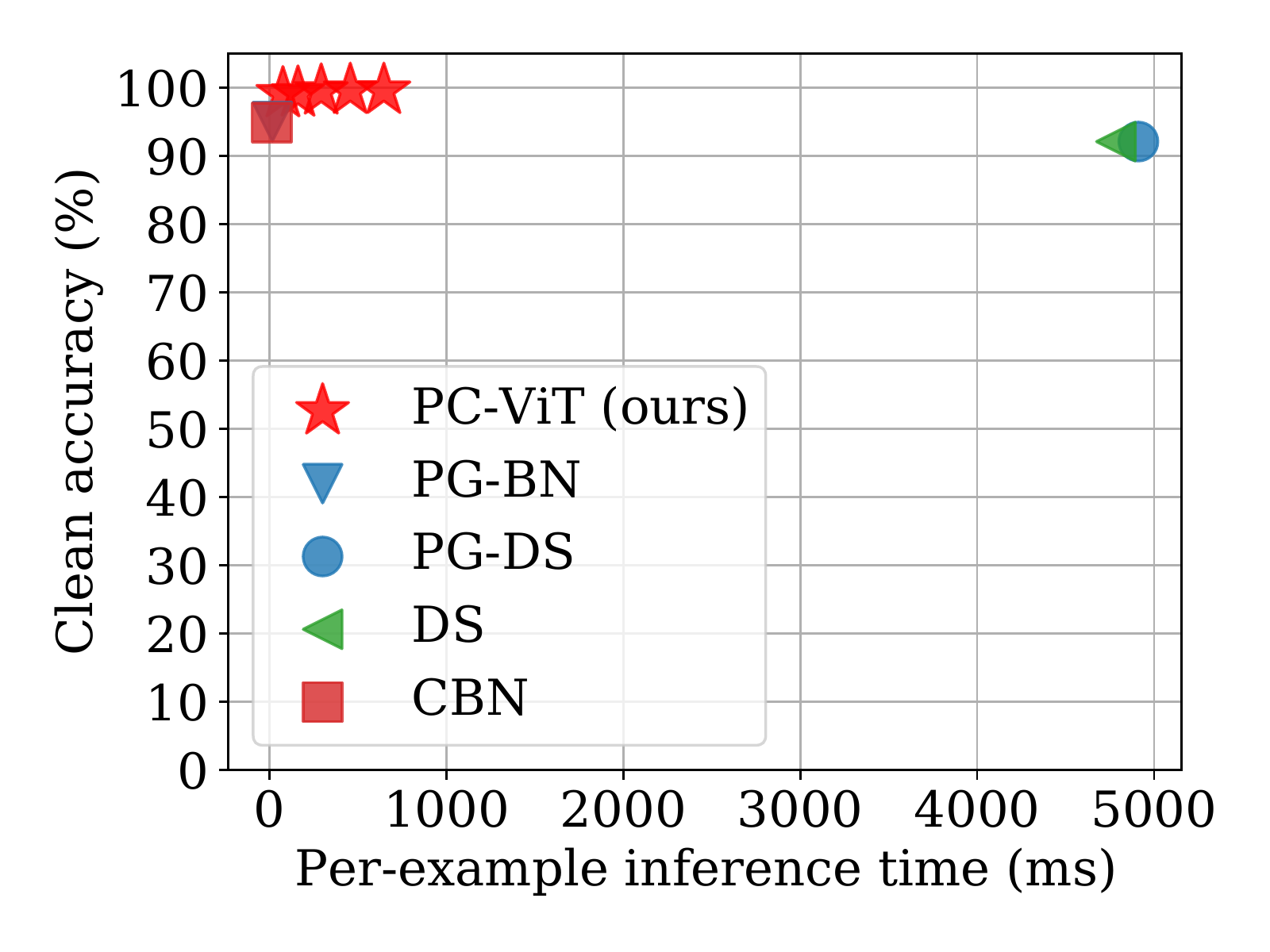}\\
\end{minipage}%
\begin{minipage}[b]{0.5\linewidth}
\includegraphics[width=\linewidth]{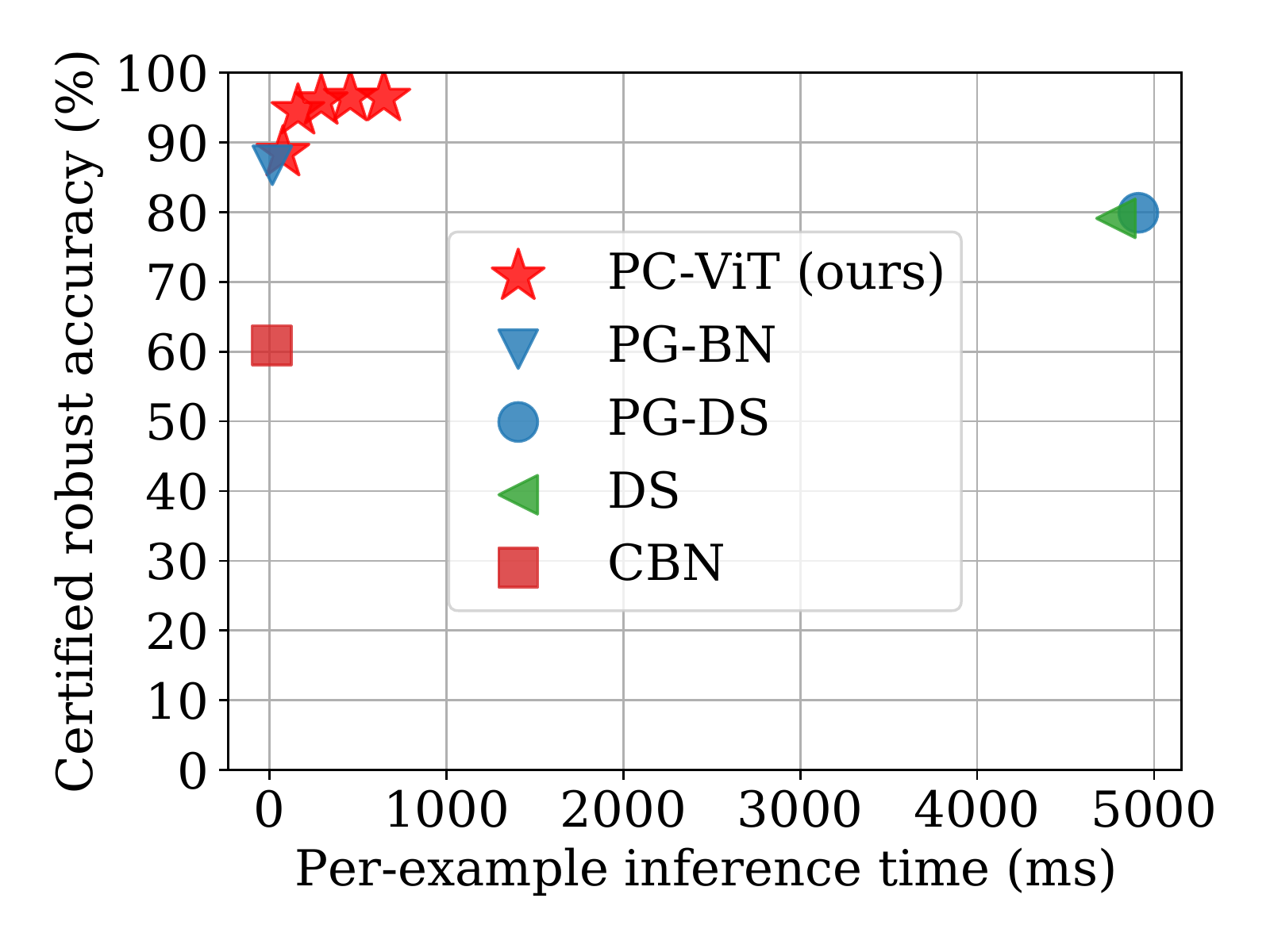}\\
\end{minipage}%
\vspace{-1em}
\caption{Trade-off between defense overhead and defense accuracy on ImageNette (left: clean accuracy; right: certified robust accuracy) }\label{fig-efficiency-nette}
\end{figure}

\begin{figure}[t]
\begin{minipage}[b]{0.5\linewidth}
\includegraphics[width=\linewidth]{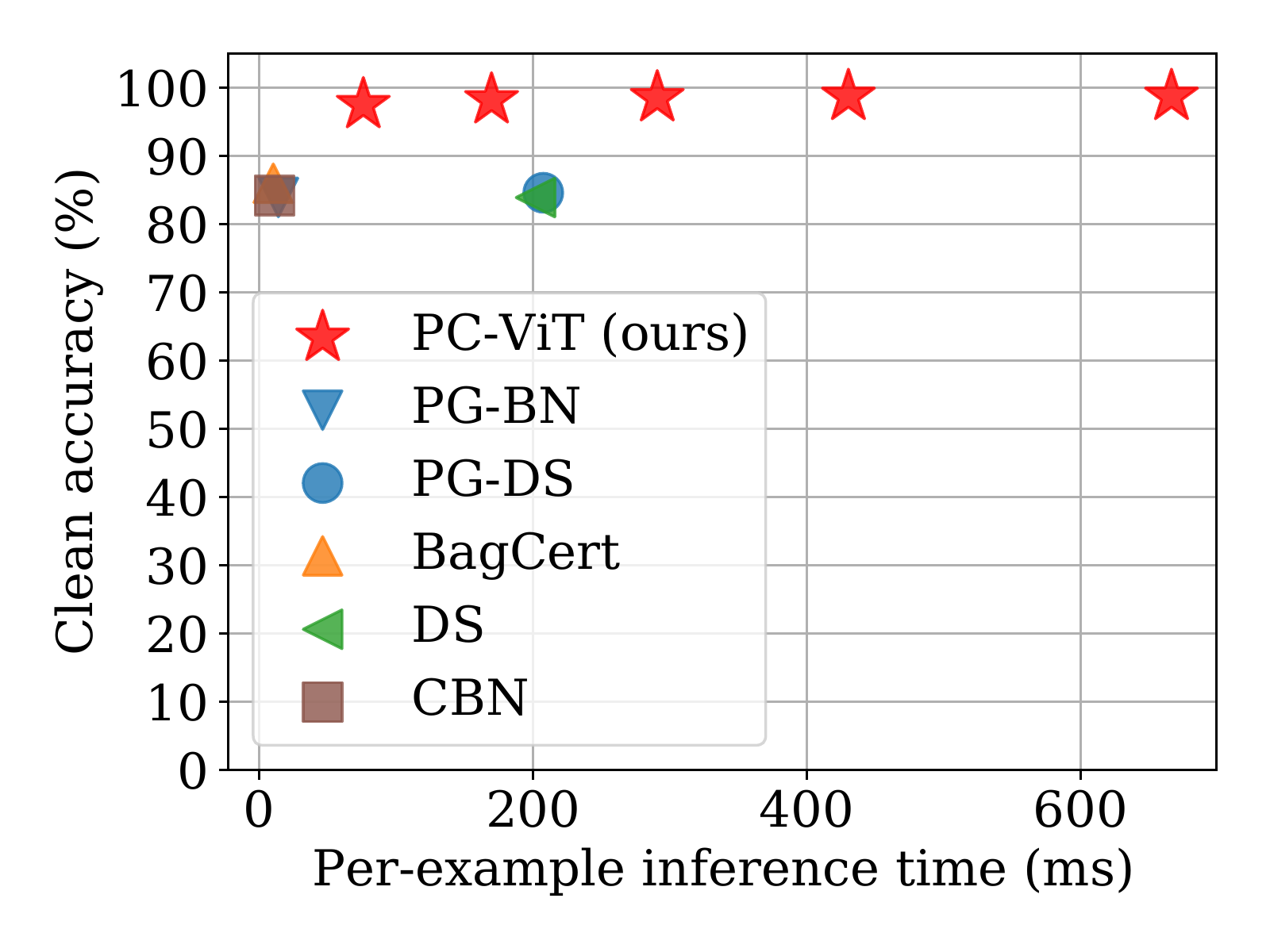}\\
\end{minipage}%
\begin{minipage}[b]{0.5\linewidth}
\includegraphics[width=\linewidth]{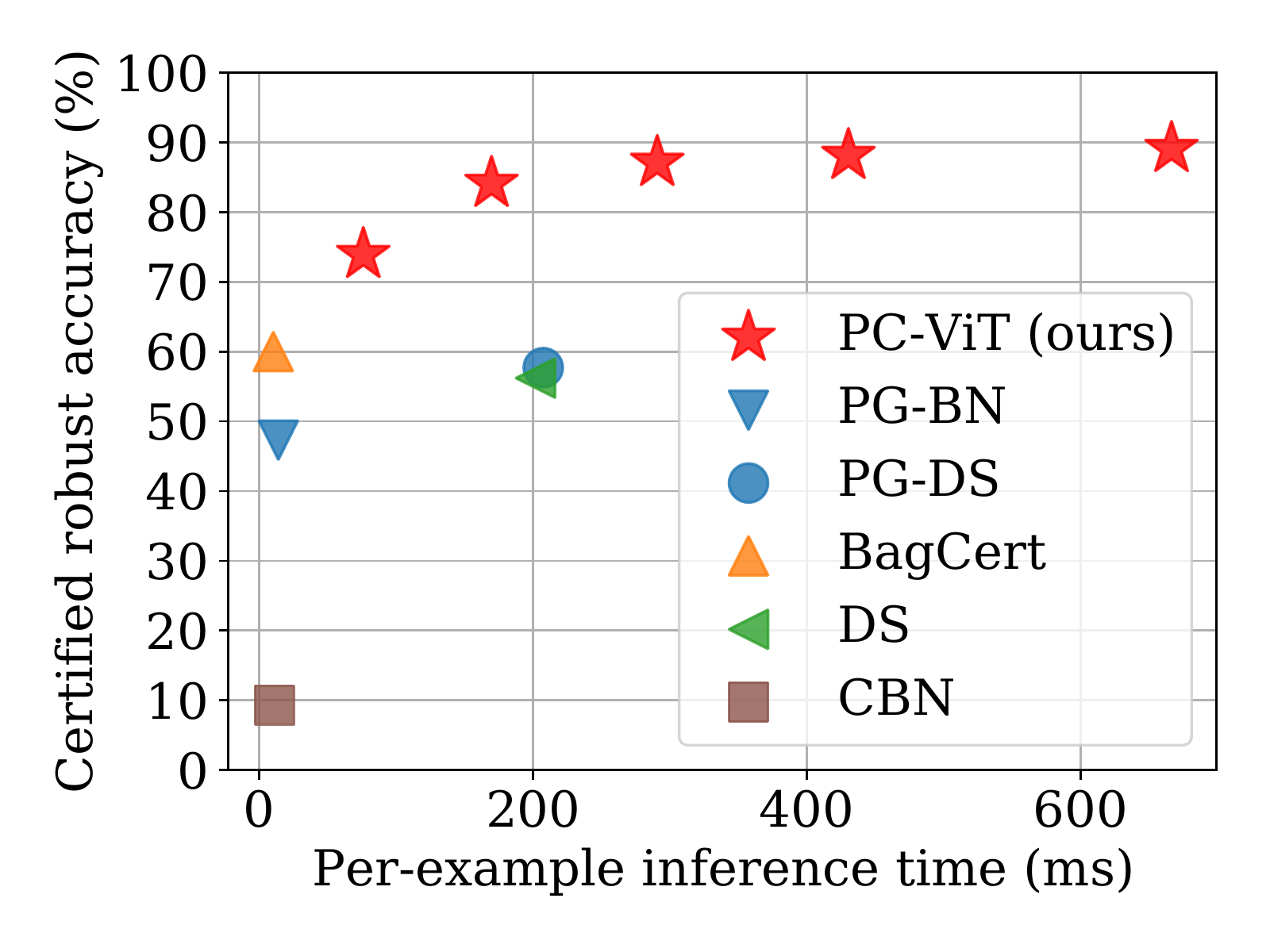}\\
\end{minipage}%
\vspace{-1em}
\caption{Trade-off between defense overhead and defense accuracy on CIFAR-10 (left: clean accuracy; right: certified robust accuracy) }\label{fig-efficiency-cifar}
\end{figure}

\textbf{Robustness certification.} The robustness certification condition for this challenger game is the same as our double-masking defense: two-mask correctness. This is because if a model has two-mask correctness, the first-round mask that removes the patch will never lose the challenger game. 

\textbf{Remark: defense complexity.} The first-round masking needs $O(|\cM|)$ masking operations evaluation. In the challenger game, every first-round mask will be used as a challenger for at most one time. Therefore, the complexity for the challenger game is also $O(|\cM|)$. In summary, the algorithm has a complexity of $O(|\cM|)$, in contrast to $O(|\cM|^2)$ of our double-masking algorithm (Algorithm~\ref{alg-prediction}). 

\textbf{Performance evaluation.} We note that challenger masking (Algorithm~\ref{alg-prediction-new}) and double-masking (Algorithm~\ref{alg-prediction}) have the same certified robust accuracy (certified via two-mask correctness), but different clean accuracy and inference efficiency. In Figure~\ref{fig-pc2}, we plot the per-example runtime (on clean images) and clean accuracy of two algorithms on the ImageNet dataset. As shown in the figure, the challenger masking algorithm has better defense efficiency but lower clean accuracy. We note that the double-masking algorithm has higher clean accuracy because it is more conservative in trusting a one-mask disagreer: double-masking requires all two-mask predictions in the second-round masking to give the same prediction label while challenger masking does not require this. As a result, the double-masking algorithm is less likely to return an incorrect disagreer prediction for clean images whose robustness cannot be certified. We prioritize the defense accuracy and choose the double-masking algorithm in the main body of the paper.

\section{Additional Evaluation Results}\label{apx-eval}
In this section, we analyze the effect of model architectures and the Cutout training~\cite{cutout}; we also provide evaluation results for three additional datasets, as well as evaluation results that were omitted in the main body (Section~\ref{sec-evaluation}). 

\begin{table}[t]
    \centering
        \caption{Effect of different models and masked model training}
    \resizebox{\linewidth}{!}
    {
    \begin{tabular}{c|c|c|c|c|c|c}
    \toprule
          Dataset &  \multicolumn{2}{c|}{ImageNette} &\multicolumn{2}{c|}{ImageNet} & \multicolumn{2}{c}{CIFAR-10}  \\
         \midrule
     Accuracy (\%)   & clean &robust& clean &robust& clean &robust\\
         \midrule
     PC-ResNet-vanilla & 99.4&88.1&81.1&41.6&94.3&39.4 \\
     PC-ResNet-cutout & 99.6&94.4&81.6&53.0&97.8&78.8\\
     \midrule
     PC-ViT-vanilla &  99.3&94.2&83.6&59.4&97.9&77.5\\
     PC-ViT-cutout & \textbf{99.6}&\textbf{96.4}&\textbf{83.9}&\textbf{62.1}&\textbf{98.7}&\textbf{89.1}\\
     \midrule
     PC-MLP-vanilla & 98.6&91.3&79.6&53.1&95.6&62.0\\
    PC-MLP-cutout & 99.3&95.3&79.4&53.8&97.0&78.0\\

     \bottomrule
    \end{tabular}}
    \label{tab-cutout}
\end{table}

\begin{table}[t]
    \centering
        \caption{Defense performance on additional datasets}
    \resizebox{\linewidth}{!}
    {
    \begin{tabular}{c|c|c|c|c|c|c}
    \toprule
          Dataset &   \multicolumn{2}{c|}{Flowers-102~\cite{flowers}} &\multicolumn{2}{c|}{CIFAR-100~\cite{cifar}} & \multicolumn{2}{c}{SVHN~\cite{svhn}}  \\
         \midrule
    Patch size    &\multicolumn{2}{c|}{2\% pixels} & \multicolumn{2}{c|}{2.4\% pixels} & \multicolumn{2}{c}{2.4\% pixels} \\
         \midrule  
     Accuracy (\%)    & clean &robust&  clean &robust&  clean &robust\\
         \midrule
     PC-ResNet-vanilla & 93.9&61.0& 85.1&18.8&96.6&0.9\\
     PC-ResNet-cutout & 96.9&84.8&88.2&53.6&96.9&65.3  \\
     \midrule
     PC-ViT-vanilla & 99.6&96.3&91.5&54.2&96.3&18.4\\
     PC-ViT-cutout & 99.6&98.0& 92.7&69.2&97.8&71.7\\
    \midrule
    PC-MLP-vanilla &92.4&60.3& 83.0&32.6&94.6&8.8\\
     PC-MLP-cutout & 95.1 &79.2&85.6&50.6& 95.5&57.0\\

     \bottomrule
    \end{tabular}}
    \label{tab-more-ds}
\end{table}

\textbf{The Vision Transformer (ViT) based defense has the best performance and the Cutout training augmentation can significantly improve model accuracy.} In this analysis, we study the defense performance of different models (ResNet~\cite{resnet}, ViT~\cite{vit}, and MLP~\cite{resmlp}), and analyze the effect of Cutout training augmentation (discussed in Appendix~\ref{apx-setup}). We report the clean accuracy and certified robust accuracy of different models for different datasets in Table~\ref{tab-cutout}. For ImageNette/ImageNet, we report certified robust accuracy for a 2\%-pixel patch; for CIFAR-10, we report for a 2.4\%-pixel patch. First, we can see that despite the similar clean accuracy of vanilla models (recall Table~\ref{tab-vanilla}), PC-ViT has significantly higher clean accuracy and certified robust accuracy than PC-ResNet and PC-MLP. We believe that this is because ViT has been found more robust to pixel masking~\cite{paul2021vision}.  

Second, we can see that the Cutout augmentation~\cite{cutout} can have a large impact on the defense performance. On ImageNette and CIFAR-10 datasets, the Cutout training (denoted by the suffix ``cutout") significantly improves both clean accuracy and certified robust accuracy. For example, the clean accuracy of PC-ResNet on CIFAR-10 improves from 94.3\% to 97.8\% and the robust accuracy improves from 39.4\% to 78.8\%. However, we find that the Cutout training exhibits a different behavior on ImageNet. For ResNet, the Cutout training significantly improves the certified robust accuracy. However, for PC-ViT and PC-MLP, the Cutout training on ImageNet only has a small improvement. We partially attribute this to the challenging nature of the ImageNet model training.

\textbf{Evaluation results on additional datasets.} In addition to the three datasets used in the Section~\ref{sec-evaluation}, we provide evaluation results for another three datasets: the 102-class Oxford Flowers-102~\cite{flowers} dataset, the 100-class CIFAR-100~\cite{cifar} and the 10-class SVHN~\cite{svhn}. In Table~\ref{tab-more-ds}, we report defense performance against a 2\%-pixel patch for Flowers-102 (i.e., a 32$\times$32 patch on the 224$\times$224 image) and 2.4\%-pixel patch for CIFAR-100 and SVHN (i.e., a 5$\times$5 patch on the 32$\times$32 image). As shown in the table, our defense works well on these three additional datasets and this further demonstrates the general applicability of \framework.

\begin{table}[t]
    \centering
        \caption{Per-example inference time of different models}   
        \label{tab-time}
\resizebox{\linewidth}{!}
{
    \begin{tabular}{c|c|c|c|c|c}
    \toprule
 Vanilla Models & ResNet-50~\cite{bit} &  ViT~\cite{vit}& ResMLP~\cite{resmlp} &\multicolumn{2}{c}{--}\\
  \midrule
Time (ms)  &12.9 &16.7&21.7&\multicolumn{2}{c}{--}\\
\midrule
   Prior defenses &  CBN~\cite{zhang2020clipped}& DS~\cite{levine2020randomized} &PG-BN~\cite{xiang2021patchguard} & PG-DS~\cite{xiang2021patchguard} &BagCert~\cite{metzen2021efficient} \\
   \midrule
 Time (ms)&12.0&4740.0&44.2&4918.0&$\sim$14.0\\
 \midrule
  Our defenses &PC-ResNet&PC-ViT & PC-MLP&\multicolumn{2}{c}{--}\\
  \midrule
Time (ms)& 60.2 - 569.5&78.8 - 672.4&91.8 - 834.8 &\multicolumn{2}{c}{--}\\
     \bottomrule
    \end{tabular}}
\end{table}
\begin{figure*}
\centering
\begin{minipage}[b]{0.32\linewidth}
\includegraphics[width=\linewidth]{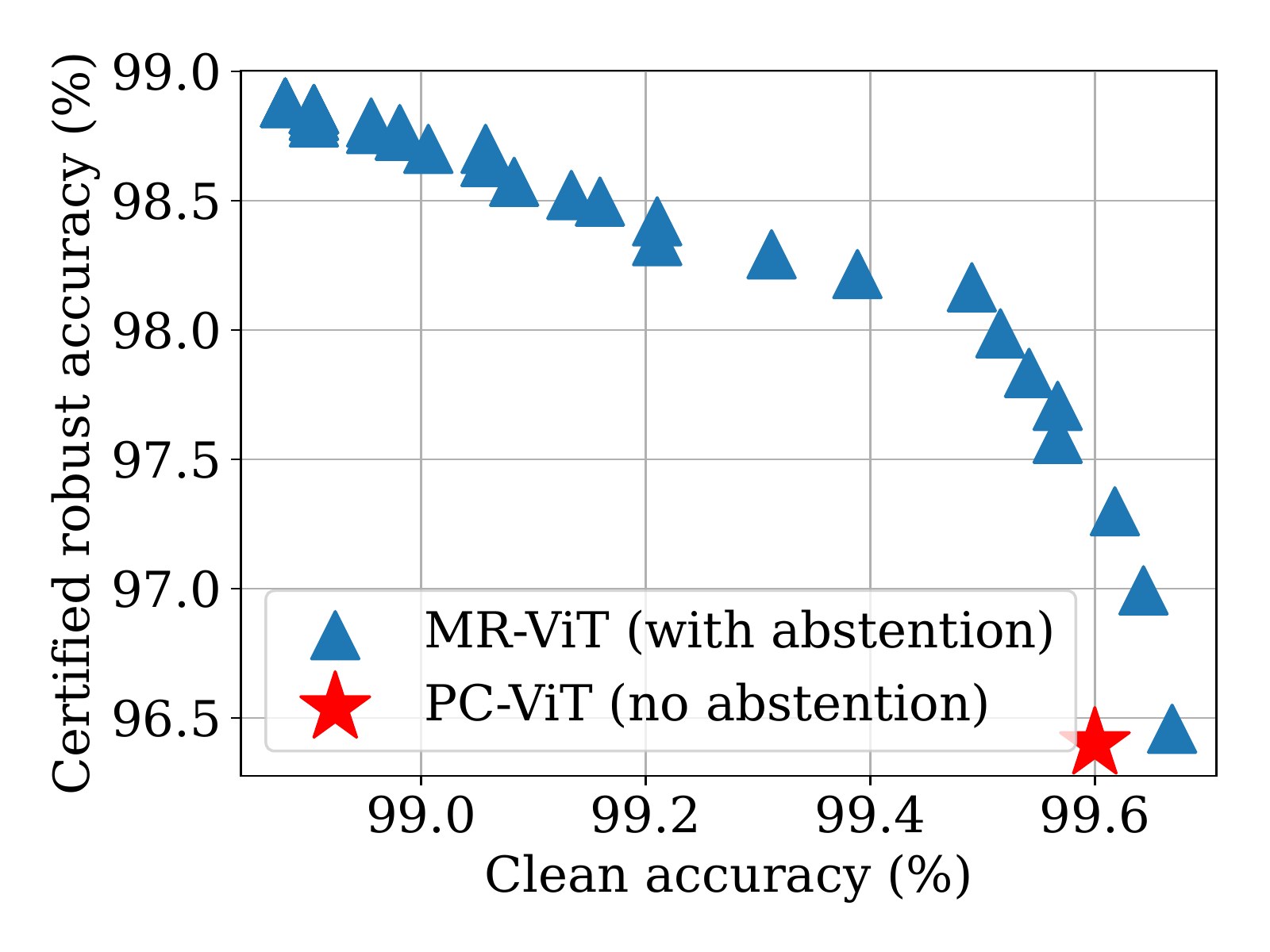}
\vspace{-2em}
\end{minipage}%
\quad
\begin{minipage}[b]{0.32\linewidth}
\includegraphics[width=\linewidth]{img/mr_net.pdf}
\vspace{-2em}
\end{minipage}%
\quad
\begin{minipage}[b]{0.32\linewidth}
\includegraphics[width=\linewidth]{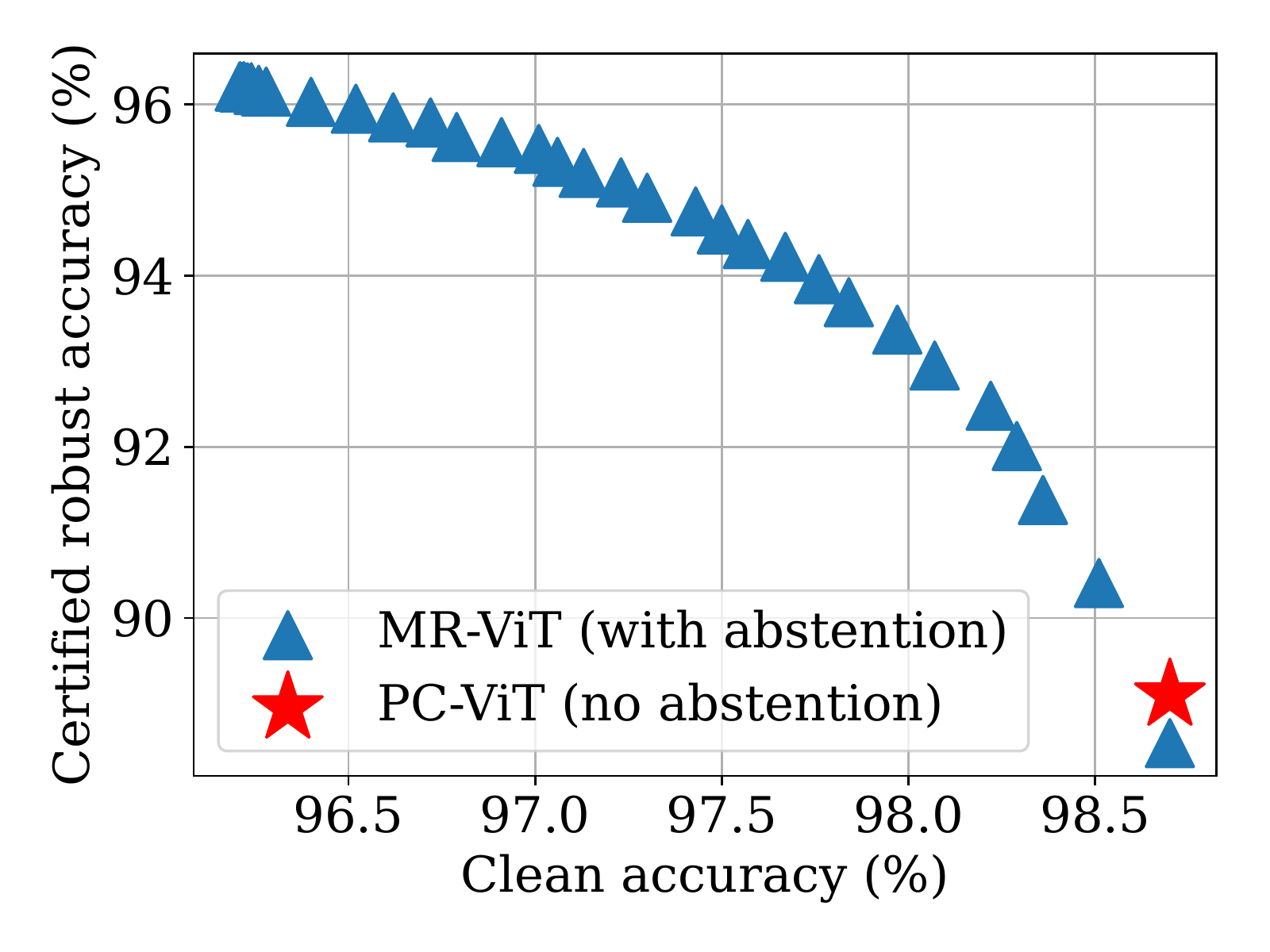}
\vspace{-2em}
\end{minipage}%
\caption{Defense performance for PC-ViT and PC-MR-ViT on ImageNette (left), ImageNet (middle), and CIFAR-10 (right)} \label{fig-mr}
\end{figure*}

\textbf{Trade-off between efficiency and performance (ImageNette and CIFAR-10).} In Figure~\ref{fig-stride-nette} and Figure~\ref{fig-stride-cifar}, we plot the clean accuracy, certified robust accuracy, and per-example inference time with different mask set sizes for ImageNette and CIFAR-10. We use ViT as the underlying model for this evaluation. The accuracy is evaluated against a 2\%-pixel patch for ImageNette, and a 2.4\%-pixel patch for CIFAR-10. These two figures exhibit a similar observation as that from Figure~\ref{fig-stride-net} in Section~\ref{sec-detailed-eval}: a larger mask set improves clean accuracy and certified robust accuracy but incurs a moderate rise in per-example inference time. In Figure~\ref{fig-efficiency-nette} and Figure~\ref{fig-efficiency-cifar}, we analyze the trade-off between defense accuracy and defense efficiency for ImageNette and CIFAR-10. The observations are similar to the ones we made in Figure~\ref{fig-efficiency}. As a side note, different defenses use different input image resolutions for CIFAR-10 images (\framework uses 224$\times$224; CBN and PG-BN use 192$\times$192; DS, PG-DS, and BagCert use 32$\times$32), which has a huge impact on the per-example inference time. 

\textbf{Effect of patch size over-estimation (ImageNette and CIFAR-10).} In Figure~\ref{fig-over-mask-nette} and Figure~\ref{fig-over-mask-cifar}, we report the defense performance when we over-estimate the patch size for ImageNette and CIFAR-10. The actual patch size is 32$\times$32 for 224$\times$224 ImageNette images and 5$\times$5 for 32$\times$32 CIFAR images. The observation is similar to the one in Section~\ref{sec-detailed-eval}.

\textbf{Performance against large patch sizes (ImageNette and CIFAR-10).} In Figure~\ref{fig-patch-size-nette} and Figure~\ref{fig-patch-size-cifar}, we report the defense performance against different patch sizes for ImageNette and CIFAR-10. The observation is similar to that in Section~\ref{sec-detailed-eval}.

\textbf{Additional runtime results for ImageNet.} In Table~\ref{tab-time}, we report per-example runtime of different defenses on ImageNet. For \framework, we report a range of per-example runtime when we use different mask set sizes ($k^2$ from $2^2$ to $6^2$). The same results are plotted in Figure~\ref{fig-efficiency} in Section~\ref{sec-detailed-eval}. 

\textbf{Additional results for Minority Reports~\cite{mccoyd2020minority}.} In Section~\ref{sec-discussion-mr}, we analyze the performance of Minority Reports (MR)~\cite{mccoyd2020minority}. In Figure~\ref{fig-mr}, we additionally report results for ImageNette (2\%-pixel patch) and CIFAR-10 (2.4\%-pixel patch). The observation is similar to what we have in Section~\ref{sec-discussion-mr}.

\begin{table*}[t]
    \centering
        \caption{Empirical and certified robust accuracy of undefended vanilla models and \framework models}
    \label{tab-empirical}
    \resizebox{\linewidth}{!}
{\scriptsize
    \begin{tabular}{c|c|c|c|c|c|c|c|c|c}
    \toprule
      Dataset   & \multicolumn{3}{c|}{ImageNette} & \multicolumn{3}{c|}{ImageNet} & \multicolumn{3}{c}{CIFAR-10}\\
      \midrule
      \multirow{2}{*}{Robust accuracy}  & empirical & empirical & \multirow{2}{*}{certified}&empirical & empirical &certified&empirical & empirical & \multirow{2}{*}{certified}\\
      & (100 iters; 1 loc)& (200 iters; 10 locs)& & (100 iters; 1 loc)& (200 iters; 10 locs)& & (100 iters; 1 loc)& (200 iters; 10 locs)&\\
         \midrule
Undefended &7.0\%&0\%&0\%&1.7\%&0\%&0\%&4.2\%&0\%&0\%\\
\framework&98.5\%&97.3\%&95.4\%&77.2\%&69.9\%&63.6\%&96.4\%&93.7\%&90.2\%\\
\bottomrule
\end{tabular}
    }

\end{table*}

\section{Case Study: Certifiable Robustness vs. Empirical Robustness}

In this paper, we focus on certifiable robustness that accounts for all possible attackers within the threat model; the certified robust accuracy is a lower bound on model's accuracy in presence of any adaptive attackers (recall Section~\ref{sec-provable}). Since we are able obtain such lower bounds, it is not necessary to empirically implement any concrete adaptive attack for robustness evaluation in Section~\ref{sec-evaluation}. In this section, we provide a case study to better understand the difference between empirical robustness and certifiable robustness.

\textbf{A stronger attacker can downgrade the empirical robust accuracy but not certified robust accuracy.} In Table~\ref{tab-empirical}, we report empirical and certified robust accuracy of undefended models and \framework models (using ViT) for 1000 randomly selected test images from different datasets. We perform a PatchAutoPGD attack~\cite{Mu2021defending,croce2020reliable} on vanilla ViT models to generate adversarial images. We report two types of empirical robust accuracy: the first one is evaluated using 100 attack iterations and 1 random patch location; the second one considers a stronger attacker using 200 attack iterations and 10 random patch locations. First, we can see that the empirical robust accuracy is always higher than the certified robust accuracy, as expected. Second, we can see that the empirical robust accuracy of both undefended models and \framework models drops significantly when we consider a stronger attacker (200 iters and 10 locs). This demonstrates the biggest limitation of empirical robustness evaluation: there might exist a stronger attacker (e.g., who has more computational resources and more sophisticated attack strategies) that can further downgrade model accuracy. In contrast, the certified robust accuracy studied in this paper is a provable lower bound on model's accuracy against any adaptive attacker within the threat model. In other words, the attacker can never push the empirical robust accuracy below the certified robust accuracy.

\section{\framework Robustness for Images with Different Object Sizes and Object Classes}\label{apx-dissect}
In this subsection, we study how the certified robustness of PC-ViT is affected by object sizes and object classes on the ImageNet~\cite{deng2009imagenet} dataset. 

\textbf{Object size.} We take the annotations of object bounding boxes from the ImageNet~\cite{deng2009imagenet} dataset to study how the object size affects the certified robustness. For each image, we count the number of pixels of the union of all bounding boxes as our measure of the salient object size. We plot the certified robust accuracy (against a 2\%-pixel patch) of images with different salient object sizes (in the percentage of image pixels) in Figure~\ref{fig-area-analysis}. As shown in the figure, we can see that \framework generally has higher certified robust accuracy for larger objects. This is an expected result since small objects might be completely occluded by the adversarial patch (see Figure~\ref{fig-small-obj} for visual examples).

\textbf{Object class.} In Figure~\ref{fig-class-analysis}, we plot the distribution of certified robust accuracy for different image object classes. We can see that most classes have high certified robustness, but the certified robust accuracy can vary greatly across different object classes. For example, we can achieve 100\% certified robust accuracy for some classes (e.g., classes ``n02116738: African hunting dog", ``n02342885: hamster", ``n11879895: rapeseed", and ``n12057211: yellow lady's slipper") while we only have 4\% certified robust accuracy for some classes (e.g., classes ``n02107908: Appenzeller" and ``n04152593: screen"). Moreover, we plot the average object size (in the percentage of the image pixels) for each class. The result further demonstrates that the certified robust accuracy is affected by not only the object size (Figure~\ref{fig-area-analysis}), but also the object class. Further analysis of the classes with poor performance is one of our future work directions.

\begin{figure}[t]
    \centering
    \includegraphics[width=0.8\linewidth]{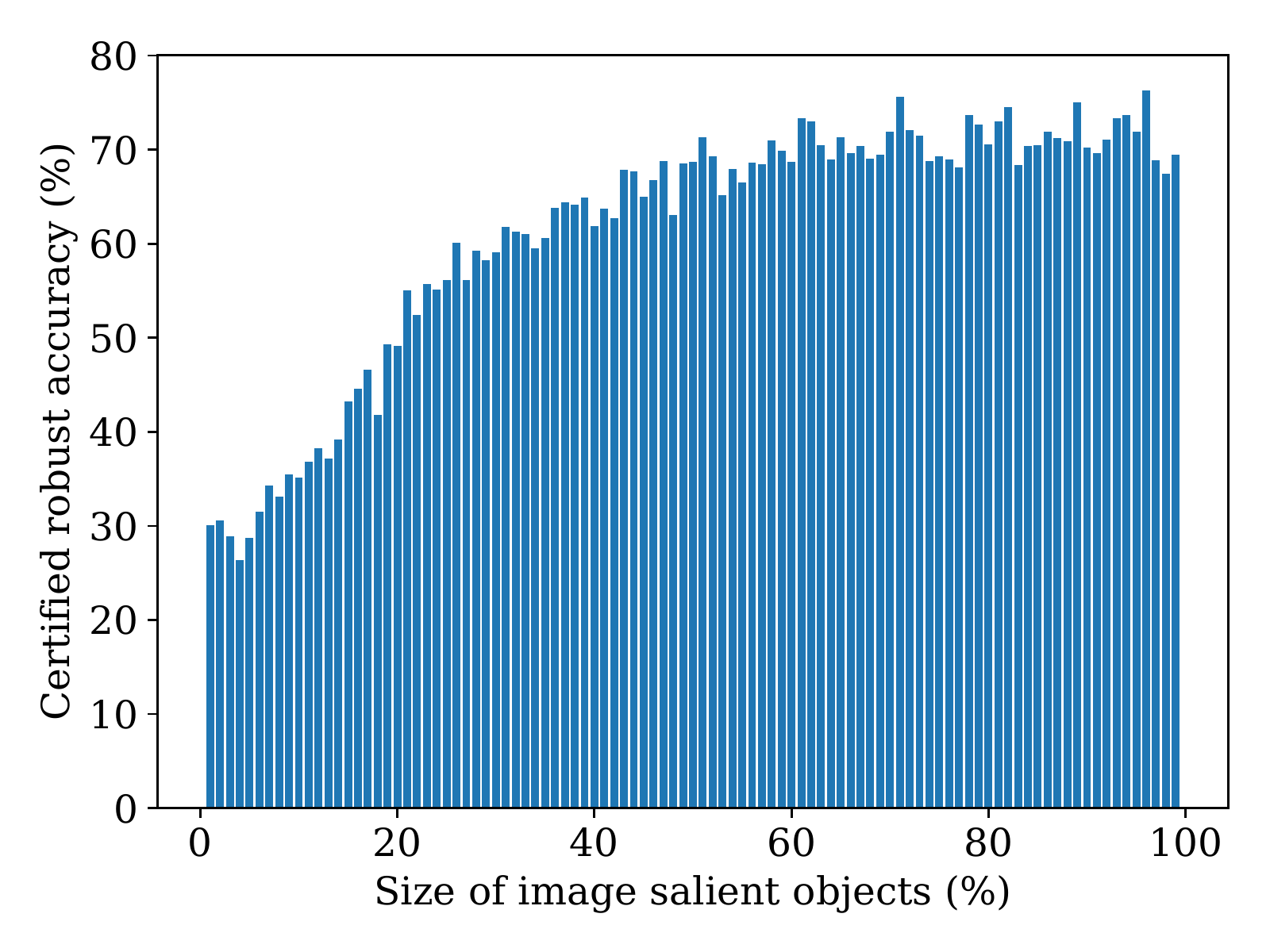}
       \vspace{-1em}
    \caption{Certified robust accuracy for images with different salient object sizes (PC-ViT on ImageNet)}
    \label{fig-area-analysis}
\end{figure}
\begin{figure}[t]
    \centering
    \includegraphics[width=0.9\linewidth]{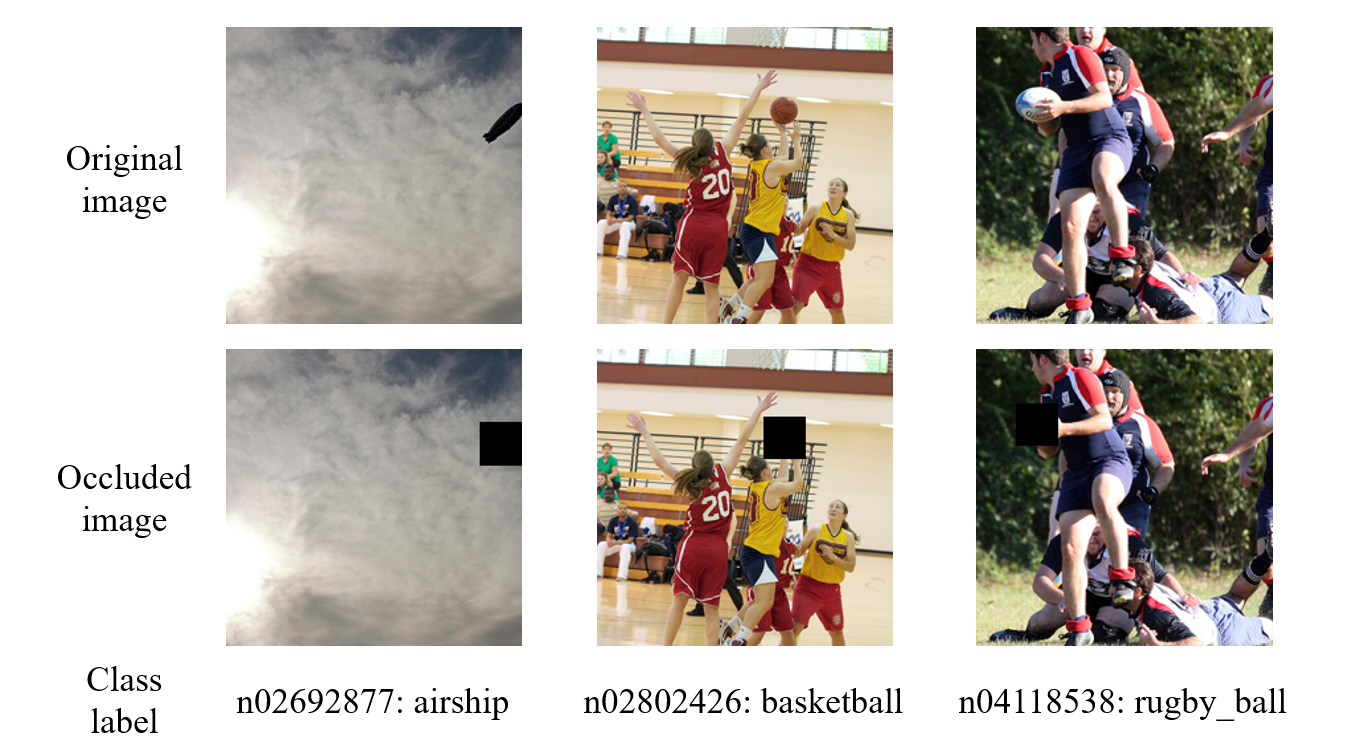}
       \vspace{-1em}
    \caption{Visualization of 32$\times$32 occlusion on 224$\times$224 ImageNet images with small objects}
    \label{fig-small-obj}
\end{figure}
\begin{figure}[b]
    \centering
    \includegraphics[width=0.8\linewidth]{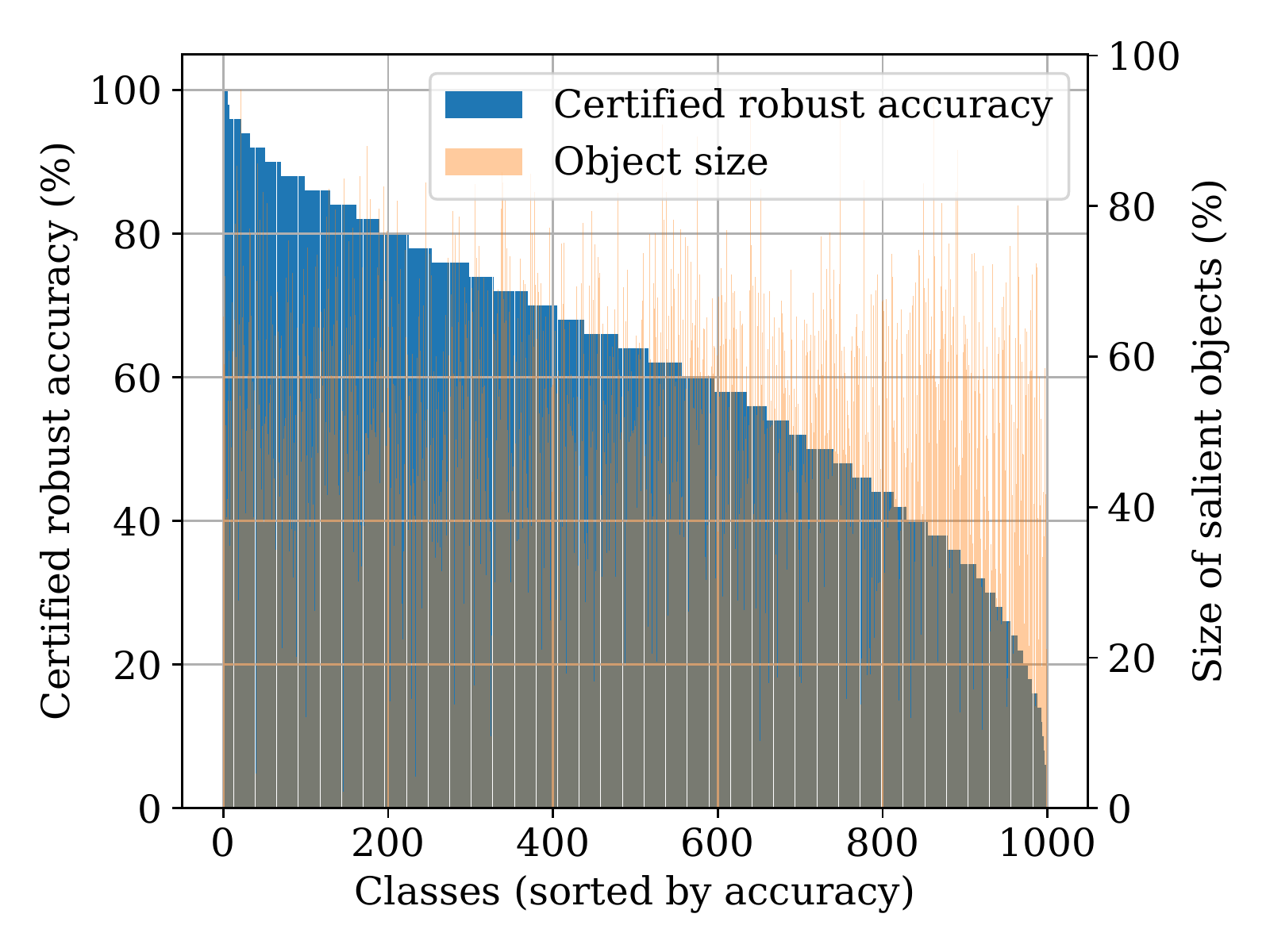}
       \vspace{-1em}
    \caption{Certified robust accuracy and average object size across different classes (PC-ViT on ImageNet; the class indices are sorted based on the certified robust accuracy)}
    \label{fig-class-analysis}
\end{figure}






\end{document}